\documentclass{article} 
\usepackage{iclr2027_conference,times}

\usepackage{amsmath,amsfonts,bm}

\def\eqref#1{equation~\ref{#1}}

\def\plaineqref#1{\ref{#1}}

\def\1{\bm{1}}

\DeclareMathAlphabet{\mathsfit}{\encodingdefault}{\sfdefault}{m}{sl}
\SetMathAlphabet{\mathsfit}{bold}{\encodingdefault}{\sfdefault}{bx}{n}

\usepackage{hyperref}
\usepackage{url}

\usepackage{amsmath}
\usepackage{amssymb}
\usepackage{amsfonts}
\usepackage{amsthm}
\newtheorem{proposition}{Proposition}

\def\eqref#1{Eq.~\ref{#1}}

\usepackage{textcomp,needspace,booktabs,graphicx}
\usepackage{tcolorbox}
\tcbuselibrary{breakable,listings}

\usepackage{subcaption}
\usepackage{wrapfig}
\usepackage{multirow}

\usepackage{enumitem}

\usepackage{wrapfig}

\usepackage{titletoc}

\usepackage[table]{xcolor}
\definecolor{accent}{RGB}{166,25,85}
\hypersetup{
  colorlinks = true,
  citecolor  = accent,
  linkcolor  = accent,
  urlcolor   = accent,
  pdfauthor  = {},
  pdftitle   = {},
}

\usepackage{array}
\newcolumntype{Z}{>{\centering\arraybackslash}m{4em}}

\newcommand{\aff}[1]{\textsuperscript{\normalfont#1}}

\title{Retrospective Distillation Attribution \\ via Normalized Response Similarity}

\author{Minwoo Jang\aff{1}, Jaechang Kim\aff{2}, Minhyeon Oh\aff{3}, Jeongyeon Hwang\aff{1}, Jungseul Ok\aff{1,3}\thanks{Corresponding Author} \\
\aff{1}Graduate School of Artificial Intelligence, POSTECH, South Korea \\
\aff{2}POSTECH Institute of Artificial Intelligence, South Korea \\
\aff{3}Department of Computer Science and Engineering, POSTECH, South Korea \\
\texttt{\{minwoo,jungseul\}@postech.ac.kr} \\
}

\iclrfinalcopy 
\begin{document}

\setcounter{footnote}{1}
\maketitle

\begin{abstract}
\label{sec:abstract}
Model distillation transfers capabilities through supervised fine-tuning (SFT) on teacher responses, often collected from commercial APIs, raising questions of model provenance. Existing distillation attribution methods have been largely evaluated on students immediately after the SFT step. However, a distilled model may undergo further SFT, preference optimization, or reinforcement learning before release, while an auditor may lack access to the pre-distillation checkpoint required by reference-based attribution. To close this gap, we propose SCOUT, an output-only method that aggregates recurring \emph{syntactic patterns} into candidate profiles, filters low-contrast patterns, and calibrates student--candidate distances against inter-candidate distances. SCOUT supports attribution and abstention using only current texts, without model weights, token likelihoods, or historical checkpoints. Auditing publicly released descendants of distilled models spanning diverse post-training objectives, SCOUT consistently identifies the distillation source. Furthermore, tracing teacher-associated \emph{syntactic signatures} along training trajectories reveals that they emerge during distillation and persist through subsequent preference optimization and reinforcement learning.

\end{abstract}

\section{Introduction}
\label{sec:intro}
Model distillation~\citep{hsieh-etal-2023-distilling} transfers capabilities through supervised fine-tuning (SFT) of a student on teacher responses. For example, responses from GPT-4o and GPT-4o-mini~\citep{openai2024gpt4ocard} were used for fine-tuning OLMo 2~\citep{walsh2025} and SARDI-Dream-7B~\citep{junger2026selfaugmenting}, respectively. However, providers now prohibit using such outputs to train competing models \citep{anthropic2025commercialterms,openai2026terms} and have reported attempts to collect them for that purpose \citep{anthropic2026distillationattacks,openai2026adversarialdistillation}. Given that released artifacts alone do not reveal whether a student was trained on such data, recent work on distillation attribution~\citep{wadhwa-etal-2025-taught,rawat2026referencebaseddistillationdetectionllms} seeks to identify which \emph{teacher candidate} generated the responses used to train the student. For attribution to serve as credible provenance evidence, an audit must rely only on externally available information and clarify the evidentiary scope of positive and negative verdicts.

Existing methods either compare token likelihoods under the student and a pre-distillation checkpoint~\citep{rawat2026referencebaseddistillationdetectionllms} or train a classifier on Part-of-Speech (PoS) n-grams~\citep{wadhwa-etal-2025-taught}. In this paper, we leverage this \emph{syntactic signal} by aggregating patterns that recur across each candidate's responses and filtering those that no candidate uses distinctively. Even after filtering, a candidate may appear close to a student simply because it sits close to other candidates' text in general, or because its distances vary widely enough that such closeness is unremarkable. To separate unusual proximity from these tendencies, we additionally use a candidate-pool calibration that evaluates each student--candidate distance relative to distances between that candidate and other candidates. We call the resulting method SCOUT (\textbf{\underline{S}}yntactic \textbf{\underline{C}}alibration for \textbf{\underline{O}}utput-only \textbf{\underline{U}}nreferenced \textbf{\underline{T}}argets), which only uses current text outputs from the student and teacher candidates, requiring no classifier, model weights, token likelihoods, or historical checkpoints.

The need for such an output-only method becomes particularly clear in the \emph{retrospective} setting we introduce, where distillation attribution is performed after the audited model has been released and may have undergone further training phases such as SFT, preference optimization, or reinforcement learning, possibly by a third party. In this setting, subsequent training may alter the observable teacher-associated signature. Moreover, a pre-distillation checkpoint may be unavailable, and mistakenly using a post-distillation checkpoint can mask the distillation effect. An external auditor may also observe only current text outputs. Consequently, we ask whether student and candidate outputs retain sufficient evidence for distillation attribution and abstention. Our contributions are as follows:

\begin{figure}[t]
  \centering
  \includegraphics[width=\linewidth]{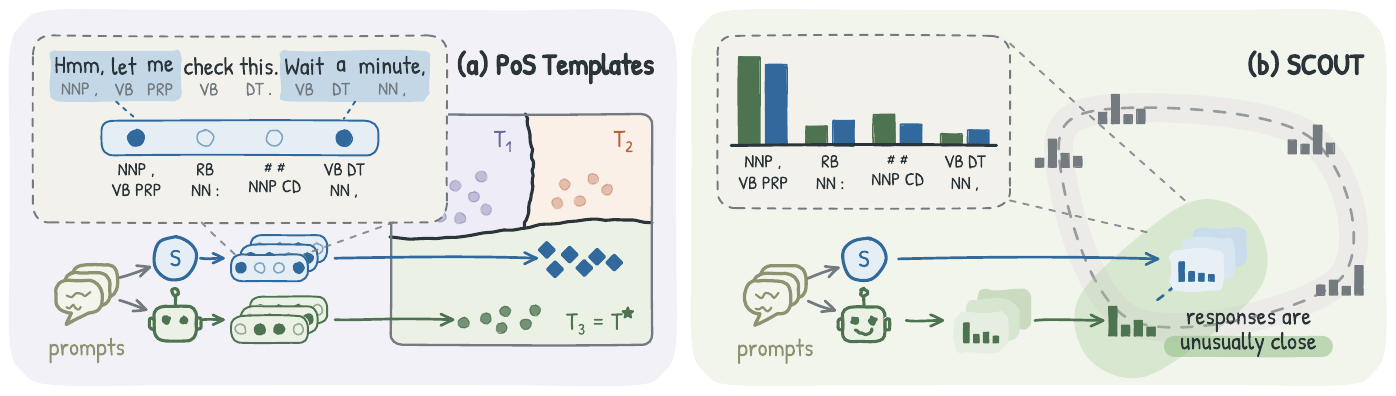}
  \caption{\textbf{How PoS Templates and SCOUT use syntactic evidence for distillation attribution.}
  \textbf{(a)} PoS Templates learns a classifier from syntactic patterns in individual candidate responses and classifies each student's responses.
  \textbf{(b)} SCOUT aggregates recurring syntactic patterns across each teacher candidate's responses into a \emph{profile} and compares each student response with these profiles.}
  \label{fig:overview}
  \vspace{-0.5em}
\end{figure}

\noindent\textbf{Method}:
We propose SCOUT, an output-only distillation attribution method that aggregates patterns across prompts into candidate profiles, filters low-contrast patterns, and calibrates student--candidate distances against the pool to identify an unusually close candidate. Figure~\ref{fig:overview} contrasts SCOUT with PoS Templates~\citep{wadhwa-etal-2025-taught}.

\noindent\textbf{Retrospective Evaluation}:
We evaluate distillation attribution on public descendants of models distilled from DeepSeek-R1~\citep{deepseekai2025r1} and subsequently trained by different groups through additional SFT, preference optimization, or reinforcement learning. SCOUT selects the documented teacher for all evaluated descendants within the candidate pool using current outputs alone. 

\noindent\textbf{Findings}:
Using SCOUT as an analytical instrument, we observe that syntactic teacher signatures emerge during SFT and remain readable through preference optimization and reinforcement learning across public training trajectories and released descendants spanning dense and mixture-of-experts architectures. Moreover, the signature emerges and persists in a diffusion language model lineage \citep{ye2025dream7bdiffusionlarge}, suggesting that the phenomenon is not confined to autoregressive generation.

\vspace{-0.5em}

\section{Related Work}
\label{sec:related}
\vspace{-0.5em}

\paragraph{Distillation Attribution and Model Provenance}
PoS Templates~\citep{wadhwa-etal-2025-taught} trains a classifier on syntactic patterns and uses it to assign each student response to one of the teacher candidates, but has no native rejection rule. DistillDetect~\citep{rawat2026referencebaseddistillationdetectionllms} ranks candidates by likelihood shifts from an earlier student checkpoint. DLI~\citep{zeng2026distillation} verifies a hypothesized teacher, but its text-only setting still requires shadow distillation and auditor training. Separate model-provenance work tests parameter ancestry among fine-tuned or quantized descendants~\citep{foley-etal-2023-matching,yax2025phylolm,nikolic2025model,Shao_2026,hu-etal-2026-fingerprinting}, rather than which teacher generated the responses used for SFT. SCOUT addresses teacher attribution after subsequent training using only current student and candidate texts. It adapts reference-population normalization from speaker verification, stylometry, and membership inference~\citep{rosenberg92_icslp,10.1093/llc/17.3.267,watson2022on} and abstains when the evidence distinguishes no candidate.

\vspace{-0.5em}

\paragraph{Proactive Model Protection}
Proactive defenses monitor query streams~\citep{8806737}, modify models through active fingerprints~\citep{xu-etal-2024-instructional}, watermark generations~\citep{pmlr-v202-kirchenbauer23a,NEURIPS2024_2567c95f,sander2026textseallocalizedllmwatermark}, or alter sampling, interaction behavior, and reasoning traces to impede distillation~\citep{savani2025antidistillation,yang2026askingbackinteractionlayerantidistillation,ma-etal-2026-protecting}. Weight-release defenses transform checkpoints to resist unauthorized merging~\citep{Junhao_2025_ICCV,wang2025modelunmergingmakingmodels,jang2026making}. These mechanisms can require monitoring infrastructure, additional training, controlled decoding, or intervention in served responses. They may add serving overhead or introduce utility trade-offs. Moreover, they must be deployed while the owner controls the model or API. They cannot protect outputs collected without these safeguards or address suspicions that arise only after a third-party model is released. In such cases, post-hoc auditing becomes necessary. SCOUT addresses this complementary need without modifying the audited or candidate models.

\vspace{-0.5em}

\section{Problem Formulation}
\label{sec:problem}
\vspace{-0.5em}

\paragraph{Training Trajectory}
For any model \(M\) and prompt \(x\), let \(M(x)\) be the response produced by \(M\) for \(x\). Let \(T^\star\) denote the \emph{teacher model} whose outputs constitute the data source for model distillation, and let \(S^{(0)}\) denote the student checkpoint immediately before the distillation step. For a set of training prompts \(\mathcal{Q}\), the supervision dataset is
\(
\mathcal{D}^\star
=
\left\{
\bigl(q,T^\star(q)\bigr)
\;\middle|\;
q\in\mathcal{Q}
\right\}.
\)
SFT on \(\mathcal{D}^\star\) produces \(S^{(1)}\), after which the student may undergo further training:
\begin{equation}
S^{(0)}
\;\xrightarrow{\operatorname{SFT} \text{ on } \mathcal{D}^\star}\;
S^{(1)}
\;\xrightarrow{\mathcal{P}_1}\;
\cdots
\;\xrightarrow{\mathcal{P}_{\ell-1}}\;
S^{(\ell)}
\;\xrightarrow{\mathcal{P}_{\ell}}\;
S^{(\ell+1)}
\;\xrightarrow{\mathcal{P}_{\ell+1}}\;
\cdots
\;\xrightarrow{\mathcal{P}_{L-1}}\;
S^{(L)}.
\label{eq:training-trajectory}
\end{equation}
Here, each \(\mathcal{P}_\ell\) denotes the \(\ell\)-th downstream training stage, such as additional SFT, preference optimization, or reinforcement learning. We write \(S:=S^{(L)}\) for the released model available to the auditor. If no training follows distillation, intermediate terms are omitted, i.e., \(L=1\) and \(S=S^{(1)}\). Retrospective distillation attribution targets \(T^\star\), the teacher model for the step from \(S^{(0)}\) to \(S^{(1)}\).

\vspace{-0.5em}

\paragraph{Auditor Observations}
Let \(\mathcal{T}=\{T_k\}_{k=1}^{K}\) be the teacher candidate pool, which may or may not contain \(T^\star\). Given audit prompts \(\mathcal{X}=\{x_i\}_{i=1}^{N}\), the auditor queries \(S\) and every \(T_k\in\mathcal{T}\) on each prompt \(x_i\). The auditor has access only to the text responses and cannot assume access to model weights, token likelihoods, the supervision data \(\mathcal{D}^\star\), or any earlier checkpoint \(S^{(\ell)}\) with \(0\leq\ell<L\).

\vspace{-0.5em}

\paragraph{Attribution and Abstention} The auditor compares responses from \(S\) and each \(T_k\in\mathcal{T}\) on shared audit prompts \(\mathcal{X}\), returning a verdict \(\widehat{T}\in\mathcal{T}\cup\{\bot\}\), where \(\bot\) denotes abstention. Correctness requires \(\widehat{T}=T^\star\) when \(T^\star\in\mathcal{T}\), and \(\widehat{T}=\bot\) when \(T^\star\notin\mathcal{T}\). Exact recovery is the objective, but resolution depends on the pool: successive releases or models connected by distillation may produce outputs too similar for reliable exact attribution. Moreover, a verdict of \(\bot\) means only that current evidence supports no candidate, not that none supplied responses at an earlier training stage.

\vspace{-0.5em}

\section{Proposed Method}
\label{sec:methodology}
\begin{figure}[t]
  \centering
  \includegraphics[width=\linewidth]{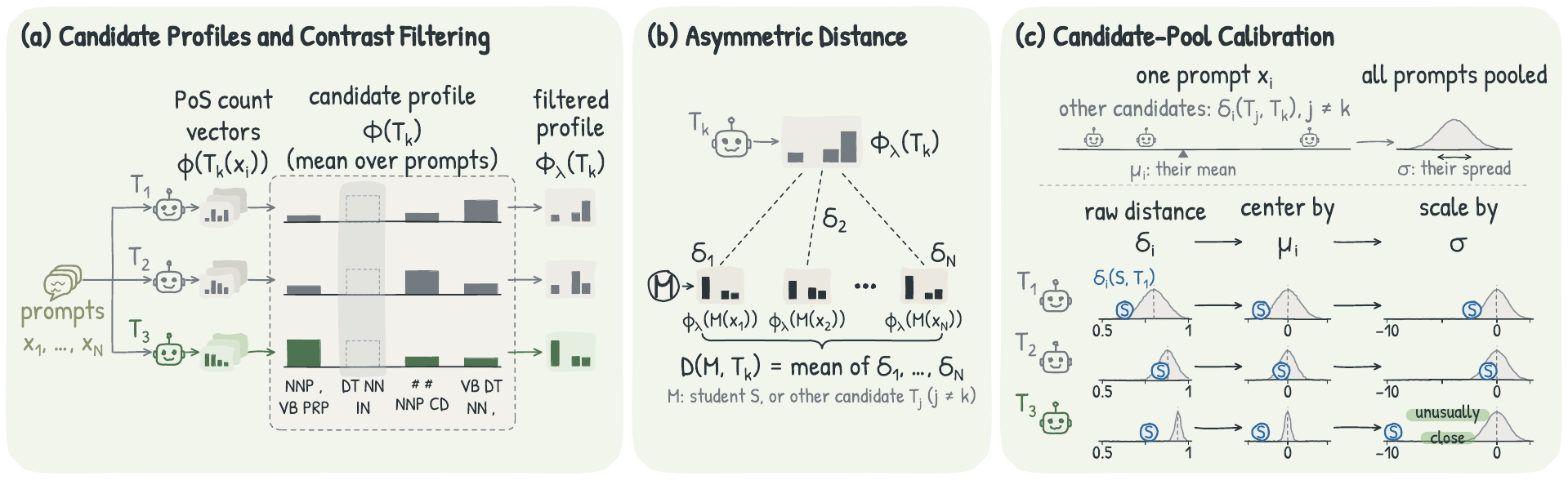}
  \caption{\textbf{Overview of SCOUT.} \textbf{(a)} PoS \(n\)-grams are aggregated into candidate profiles and low-contrast patterns are filtered. \textbf{(b)} Each response is compared with the filtered profiles. \textbf{(c)} Inter-candidate distances provide prompt-specific centers \(\mu_i\) and candidate-level scales \(\sigma\). Filtering isolates candidate-discriminative syntax, while calibration makes scores comparable across candidates. Therefore, low scores indicate proximity not explained by generic similarity to the pool.}
  \label{fig:method-scout}
  \vspace{-0.5em}
\end{figure}

\vspace{-0.5em}

In this section, we introduce SCOUT, an output-only distillation attribution method that combines (a) PoS \(n\)-gram candidate profiles, (b) candidate-contrast filtering, and (c) candidate-pool calibration for attribution with abstention. Figure~\ref{fig:method-scout} summarizes these components.

\vspace{-0.5em}

\subsection{Response Representation and Candidate Profiles}
\label{sec:setup}

\vspace{-0.5em}

To capture linguistic form rather than content or task correctness, we encode a response \(M(x)\) as the count vector \(\phi(M(x))\) of contiguous PoS \(n\)-grams for \(n\in\{3,4,5\}\). Replacing words with syntactic tags reduces lexical sensitivity while retaining local structure~\citep{shaib-etal-2024-detection,wadhwa-etal-2025-taught}. However, a single response realizes the patterns present in one generation only. Consequently, rather than fitting a classifier to individual response vectors, we represent each candidate \(T_k\) by its centroid across \(N\) prompts, termed its \emph{candidate profile}:
\begin{equation}
\Phi(T_k)
=
\frac{1}{N}
\sum_{i=1}^{N}
\phi\bigl(T_k(x_i)\bigr).
\label{eq:candidate-profile}
\end{equation}
\(\Phi(T_k)\) captures recurring patterns across candidate responses, while the audited model remains at the prompt level. Appendix~\ref{app:ngram-examples} provides examples in model responses.

\vspace{-0.5em}

\subsection{Candidate-Contrast Filtering}
\label{sec:method-two-kinds}

\vspace{-0.5em}

Patterns used at similar rates by every candidate can dominate frequency-based similarity while contributing little to candidate discrimination. We identify such patterns from the candidate profiles themselves. First, we normalize each candidate profile by its total count:
\begin{equation}
\widetilde{\Phi}(T_k)
=
\frac{\Phi(T_k)}
{\lVert\Phi(T_k)\rVert_1}.
\label{eq:candidate-normalization}
\end{equation}

Then, over features observed in at least one candidate, we compute the coordinatewise mean \(\bar{\Phi}\), standard deviation \(s_{\Phi}\), and contrast vector \(c\):
\begin{equation}
\bar{\Phi}=\frac{1}{K}\sum_{k=1}^{K}\widetilde{\Phi}(T_k),
\qquad
s_{\Phi}=\left[\frac{1}{K}\sum_{k=1}^{K}
\left(\widetilde{\Phi}(T_k)-\bar{\Phi}\right)^2\right]^\frac{1}{2},
\qquad
c=\frac{s_{\Phi}}{\bar{\Phi}}.
\label{eq:candidate-contrast}
\end{equation}
The square, square root, and division are coordinatewise. Small entries of \(c\) indicate patterns shared at similar rates across candidates, whereas large entries indicate candidate-specific variation.

Finally, let \(\lambda\in[0,1]\) denote the filtering fraction and \(\gamma_\lambda\) the \(\lambda\)-quantile of \(c\) over features observed in the current candidate pool. The mask \(w_\lambda=\mathbf{1}[c\geq\gamma_\lambda]\) retains features at or above this cutoff. Features absent from every candidate are excluded from the quantile and retain unit weight. Thus, \(\lambda=0\) recovers the unfiltered representation, while increasing \(\lambda\) removes more low-contrast features. Using this mask, we define \(\phi_\lambda(y)=w_\lambda\odot\phi(y)\) and \(\Phi_\lambda(T_k)=w_\lambda\odot\Phi(T_k)\), where \(\odot\) denotes the Hadamard product. We fix \(\lambda=0.15\) for all experiments and report sensitivity to \(\lambda\) in Appendix~\ref{app:ngram}.

\subsection{Candidate-Pool Calibration and Decision Rule}
\label{sec:method-scout}

\paragraph{Candidate-Pool Calibration}
For \(M\in\{S\}\cup\mathcal{T}\), we use the filtered representations to define the prompt-level profile distance and its uncalibrated aggregate as
\begin{equation}
\delta_i(M,T_k)=1-\cos\!\left(\phi_\lambda(M(x_i)),\Phi_\lambda(T_k)\right),
\qquad
D(M,T_k)=\frac{1}{N}\sum_{i=1}^{N}\delta_i(M,T_k).
\label{eq:profile-distance}
\end{equation}
Lower values indicate greater syntactic similarity, but raw distances are not comparable across candidates as their baseline proximity and dispersion can differ. SCOUT calibrates them using null locations and scales estimated from inter-candidate distances. Treating every other candidate \(T_j\) as a pseudo-suspect, SCOUT estimates the prompt-specific null location and candidate-level null scale
\begin{equation}
\mu_i(T_k)
=
\frac{1}{K-1}
\sum_{j\neq k}
\delta_i(T_j,T_k),
\qquad
\sigma(T_k)
=
\operatorname{std}\!\left\{
\delta_i(T_j,T_k):
j\neq k,\;
1\leq i\leq N
\right\}.
\label{eq:finite-pool-null}
\end{equation}
Here, \(\operatorname{std}\) is the population standard deviation over the \((K-1)N\) distances. Prompt-specific means capture response structure, while one scale per candidate pools all prompts because each prompt provides only \(K-1\) pseudo-suspects. The centered and standardized scores are
\begin{equation}
C(S,T_k)
=
\frac{1}{N}
\sum_{i=1}^{N}
\left[
\delta_i(S,T_k)-\mu_i(T_k)
\right],
\qquad
Z(S,T_k)
=
\frac{
C(S,T_k)
}{
\sigma(T_k)
}.
\label{eq:calibrated-score}
\end{equation}
Centering and scaling normalize response similarity relative to the candidate pool, so lower \(Z(S,T_k)\) means that the student is closer to \(T_k\) than the pool null predicts. This normalization allows SCOUT to use only current student and candidate responses, without classifier training, model weights, token likelihoods, or historical checkpoints. Appendix~\ref{app:theory} formalizes the correction, derives finite-prompt bounds, and characterizes limits of candidate-only calibration and output-only history recovery.

\paragraph{Attribution and Abstention}
\label{sec:method-attribution}
For closed-set attribution, the auditor selects
\begin{equation}
\widehat{T}
=
\arg\min_{T_k\in\mathcal{T}}
Z(S,T_k).
\label{eq:closed-set-winner}
\end{equation}
For open-set attribution, we define the closest-alternative gap:
\begin{equation}
G(S)
=
\min_{T'\in\mathcal{T}\setminus\{\widehat{T}\}}
Z(S,T')
-
Z(S,\widehat{T}).
\label{eq:margin-G}
\end{equation}
Larger \(G(S)\) indicates clearer separation. The auditor accepts \(\widehat{T}\) when \(G(S)\geq\tau\) and abstains otherwise, with \(\tau\) fitted without the audited fold, following Section~\ref{sec:exp-controlled}.

\vspace{-0.5em}

\subsection{Trajectory Margins}
\label{sec:method-trajectory}

\vspace{-0.5em}

For the trajectory analysis in Section~\ref{sec:findings}, the designated teacher \(T\) and intermediate checkpoints \(S^{(\ell)}\) are known. We summarize each checkpoint using the pool-average teacher gap
\begin{equation}
G_T\bigl(S^{(\ell)}\bigr)
=
\frac{1}{K-1}
\sum_{T'\in\mathcal{T}\setminus\{T\}}
\left[
Z\bigl(S^{(\ell)},T'\bigr)
-
Z\bigl(S^{(\ell)},T\bigr)
\right].
\label{eq:margin}
\end{equation}
Larger values indicate stronger separation toward \(T\) relative to the rest of the pool. Averaging over alternatives prevents shifts in the closest rival from appearing as movement toward or away from \(T\).

For each checkpoint, we report the change in this margin toward \(T\):
\begin{equation}
\Delta G_T\bigl(S^{(\ell)}\bigr)
=
G_T\bigl(S^{(\ell)}\bigr)
-
G_T\bigl(S^{(0)}\bigr).
\label{eq:delta-margin}
\end{equation}
Section~\ref{sec:findings} uses a common candidate pool, with candidate profiles and empirical nulls fixed along each trajectory. Positive values indicate stronger teacher-relative separation than at \(S^{(0)}\). Both \(G_T\) and \(\Delta G_T\) are teacher-aware descriptive quantities and do not determine attribution verdicts.

\vspace{-0.5em}

\section{Experiments}
\label{sec:experiments}
\vspace{-0.5em}

In this section, we evaluate SCOUT in two regimes. In Section~\ref{sec:exp-descendants}, we evaluate it retrospectively on released descendants of DeepSeek-R1~\citep{deepseekai2025r1} distillations built on Qwen2.5-Math bases~\citep{yang2024qwen25mathtechnicalreportmathematical} after further post-training, which may have altered the teacher's syntactic signature. Section~\ref{sec:exp-controlled} evaluates it in the controlled setting, where students are audited immediately after SFT and the known teacher can be included in or removed from the candidate pool.

\vspace{-0.5em}

\subsection{Compared Methods and Probe Sets}
\label{sec:exp-setup}

\vspace{-0.5em}

We compare SCOUT with DistillDetect~\citep{rawat2026referencebaseddistillationdetectionllms} and PoS Templates~\citep{wadhwa-etal-2025-taught}. Because PoS Templates is a closed-set classifier without native abstention, we retrain it on the three remaining candidates in teacher-removed controlled trials and use its top-two class-probability gap as the detection score. Following \citet{3666122.3668142}, we also evaluate two training-free LLM judges, Mistral Medium 3.1~\citep{mistral_medium_3} and Solar Pro 4~\citep{upstage2026solarpro4}. Score-based attribution follows the DistillDetect evaluator's mean-score rule~\citep{rawat2026referencebaseddistillationdetectionllms}, while controlled detection also uses its threshold and significance procedures. LLM judges use the most frequent prompt-level label, with ties yielding no verdict. Appendix~\ref{app:tables} reports component comparisons and analyzes distillation attribution under a sibling-enriched candidate pool.

We use four probe sets throughout the evaluation. OpenMathInstruct (OMI)~\citep{toshniwal2024openmathinstruct} and s1~\citep{muennighoff2025s1simpletesttimescaling} serve as math probes, while OASST1~\citep{3666122.3668186} and Databricks Dolly 15K (Dolly)~\citep{DatabricksBlog2023DollyV2} serve as conversation probes. Appendix~\ref{app:template} documents the public checkpoints, their designated teachers, and training sequences used in the retrospective and trajectory analyses, while Appendix~\ref{app:details} provides additional experimental details, including the prompts provided to the LLM judges.

\vspace{-0.5em}

\subsection{Distillation Attribution After Subsequent Training}
\label{sec:exp-descendants}

\vspace{-0.5em}

\paragraph{Evaluation Panel} We study DeepSeek-R1-Distill-Qwen-1.5B and 7B distilled from DeepSeek-R1, whose pre-distillation bases are Qwen2.5-Math-1.5B and 7B, respectively. We audit these two bases, two distilled parents, and their twelve public descendants. As detailed in Appendix~\ref{app:details}, ten descendants underwent reward-based or self-play post-training without additional R1 responses, whereas the remaining two were re-exposed to R1 responses during further training. Since R1 is the teacher for every distilled parent and descendant, this panel tests whether its signature remains detectable after subsequent training and absent from the pre-distillation bases.

\vspace{-0.5em}

\begin{table*}[t]
  \centering
  \caption{\textbf{Distillation attribution after subsequent training.} \emph{Desc.} and \emph{Par.} count correct R1 assignments among twelve descendants and two distilled parents, while \emph{Base} counts pre-distillation bases correctly identified as non-R1. Using the same twelve descendants and prompt bank, \emph{Dist.} compares independently sorted candidate scores at equal ranks, whereas \emph{Same} compares candidates on each shared prompt. Both report probe-level R1 selection accuracy averaged across descendants, distinguishing distributional separation from same-prompt candidate selection. For DistillDetect, \(R\) denotes the reference model. \emph{Dist.} is unavailable for LLM judges because they return labels only. Higher is better for all metrics. Bold marks the best output-only value in each column.}
  \vspace{-0.25em}
  \label{tab:r1-retrospective}
  \begingroup
  \renewcommand{\arraystretch}{0.90}
  \resizebox{0.90\textwidth}{!}{%
  \begin{tabular}{lcccccccccc}
    \toprule
    \textbf{Math probes}
      & \multicolumn{5}{c}{OMI}
      & \multicolumn{5}{c}{s1} \\
    \cmidrule(lr){2-6}\cmidrule(r){7-11}
      & \multicolumn{3}{c}{Model Level}
      & \multicolumn{2}{c}{\makebox[0pt][c]{Probe Level (\%)}}
      & \multicolumn{3}{c}{Model Level}
      & \multicolumn{2}{c}{\makebox[0pt][c]{Probe Level (\%)}} \\
    \cmidrule(lr){2-4}\cmidrule(lr){5-6}
    \cmidrule(lr){7-9}\cmidrule(lr){10-11}
    Method
      & Desc. & Par. & Base
      & \makebox[2.7em][c]{Dist.}
      & \makebox[2.7em][c]{Same}
      & Desc. & Par. & Base
      & \makebox[2.7em][c]{Dist.}
      & \makebox[2.7em][c]{Same} \\
    \midrule
    \rowcolor{black!7}
    \multicolumn{11}{l}{\textit{Reference-based (A pre-distillation reference checkpoint is required.)}} \\
    \rowcolor{black!7}
    DistillDetect, \(R=S^{(0)}\)
      & 12/12 & 2/2 & 2/2 & 98.0 & 94.1
      & 12/12 & 2/2 & 2/2 & 97.9 & 96.2 \\
    \rowcolor{black!7}
    DistillDetect, \(R=S^{(1)}\)
      & 3/12 & --- & --- & --- & ---
      & 3/12 & --- & --- & --- & --- \\
    \midrule
    \multicolumn{11}{l}{\textit{Output-only (Only current model outputs are required.)}} \\
    PoS Templates
      & 9/12 & 1/2 & \textbf{2/2} & 66.6 & 42.1
      & 9/12 & \textbf{2/2} & \textbf{2/2} & 79.4 & 46.9 \\
    LLM judge (Mistral)
      & \textbf{12/12} & \textbf{2/2} & \textbf{2/2} & --- & 45.5
      & 11/12 & \textbf{2/2} & \textbf{2/2} & --- & 42.0 \\
    LLM judge (Solar)
      & \textbf{12/12} & \textbf{2/2} & \textbf{2/2} & --- & 53.3
      & \textbf{12/12} & \textbf{2/2} & 1/2 & --- & 54.2 \\
    SCOUT (Ours)
      & \textbf{12/12} & \textbf{2/2} & \textbf{2/2} & \textbf{87.9} & \textbf{70.8}
      & \textbf{12/12} & \textbf{2/2} & \textbf{2/2} & \textbf{99.0} & \textbf{80.1} \\
    \midrule
    \textbf{Conversation probes}
      & \multicolumn{5}{c}{OASST1}
      & \multicolumn{5}{c}{Dolly} \\
    \cmidrule(lr){2-6}\cmidrule(r){7-11}
      & \multicolumn{3}{c}{Model Level}
      & \multicolumn{2}{c}{\makebox[0pt][c]{Probe Level (\%)}}
      & \multicolumn{3}{c}{Model Level}
      & \multicolumn{2}{c}{\makebox[0pt][c]{Probe Level (\%)}} \\
    \cmidrule(lr){2-4}\cmidrule(lr){5-6}
    \cmidrule(lr){7-9}\cmidrule(lr){10-11}
    Method
      & Desc. & Par. & Base
      & \makebox[2.7em][c]{Dist.}
      & \makebox[2.7em][c]{Same}
      & Desc. & Par. & Base
      & \makebox[2.7em][c]{Dist.}
      & \makebox[2.7em][c]{Same} \\
    \midrule
    \rowcolor{black!7}
    \multicolumn{11}{l}{\textit{Reference-based (A pre-distillation reference checkpoint is required.)}} \\
    \rowcolor{black!7}
    DistillDetect, \(R=S^{(0)}\)
      & 0/12 & 0/2 & 2/2 & 6.3 & 36.0
      & 4/12 & 1/2 & 2/2 & 28.3 & 43.5 \\
    \rowcolor{black!7}
    DistillDetect, \(R=S^{(1)}\)
      & 2/12 & --- & --- & --- & ---
      & 2/12 & --- & --- & --- & --- \\
    \midrule
    \multicolumn{11}{l}{\textit{Output-only (Only current model outputs are required.)}} \\
    PoS Templates
      & 1/12 & 0/2 & \textbf{2/2} & 7.8 & 25.3
      & 1/12 & 0/2 & \textbf{2/2} & 7.7 & 22.1 \\
    LLM judge (Mistral)
      & 0/12 & 0/2 & \textbf{2/2} & --- & 19.4
      & 1/12 & 0/2 & \textbf{2/2} & --- & 27.7 \\
    LLM judge (Solar)
      & 1/12 & 0/2 & 1/2 & --- & 29.8
      & 1/12 & 0/2 & 0/2 & --- & 34.9 \\
    SCOUT (Ours)
      & \textbf{12/12} & \textbf{2/2} & \textbf{2/2} & \textbf{92.9} & \textbf{75.1}
      & \textbf{12/12} & \textbf{2/2} & \textbf{2/2} & \textbf{95.5} & \textbf{84.5} \\
    \bottomrule
  \end{tabular}%
  }
  \endgroup
  \vspace{-0.5em}
\end{table*}

\paragraph{Attribution Performance} We report model-level accuracy together with two probe-level accuracies that measure distributional separation and same-prompt candidate selection, respectively. Table~\ref{tab:r1-retrospective} shows that SCOUT correctly assigns all descendants and parents to R1 and identifies both pre-distillation bases as non-R1 on all four probes. Under both probe-level views, it attains the highest output-only accuracy on every probe, indicating that its R1 signal appears both in the score distribution and in same-prompt comparisons. With \(S^{(0)}\), DistillDetect matches SCOUT on all math model-level counts. Using the already-distilled parent \(S^{(1)}\) instead yields only \(3/12\) descendants on OMI and \(3/12\) on s1 because this comparison reflects subsequent post-training rather than the R1 distillation step. Model-level attribution uses distinct ten-candidate pools for math and conversation.

\vspace{-0.5em}

\begin{figure*}[t]
  \centering
  \begin{subfigure}[t]{0.475\textwidth}
    \centering
    \includegraphics[width=\linewidth]{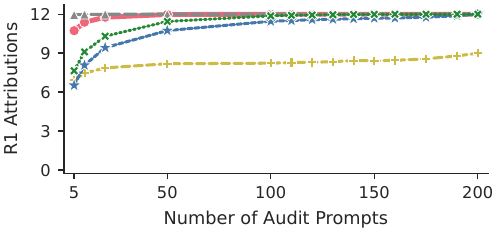}
    \caption{OMI probe set}
  \end{subfigure}\hfill
  \begin{subfigure}[t]{0.475\textwidth}
    \centering
    \includegraphics[width=\linewidth]{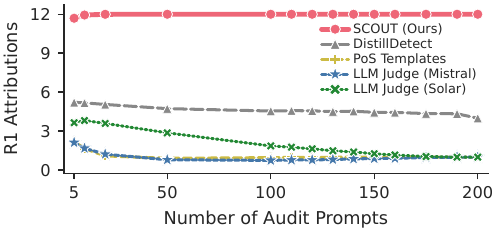}
    \caption{Dolly probe set}
  \end{subfigure}
  \vspace{-0.25em}
  \caption{\textbf{Prompt efficiency after subsequent training.} For each audit budget \(n<200\), we randomly select \(n\) distinct prompts \(2{,}000\) times and report the mean number of the 12 descendants correctly attributed to R1. At \(n=200\), we evaluate all prompts once. Higher is better.}
  \label{fig:retrospective-budget}
  \vspace{-0.5em}
\end{figure*}

\paragraph{Prompt Efficiency} Figure~\ref{fig:retrospective-budget} evaluates prompt efficiency on OMI and Dolly. With five prompts, SCOUT attributes \(10.7/12\) descendants on OMI and \(11.7/12\) on Dolly, reaching \(12/12\) with 100 and 50 prompts, respectively. On OMI, it approaches DistillDetect with \(S^{(0)}\) and is more efficient than the other output-only methods. On Dolly, SCOUT remains near-perfect as the budget grows, whereas DistillDetect declines despite access to \(S^{(0)}\) and the other output-only methods fall to \(1/12\).

\vspace{-0.5em}

\subsection{Distillation Attribution Before Subsequent Training}
\label{sec:exp-controlled}

\paragraph{Controlled Cohorts and Protocol} Each cohort contains nineteen 1.5--4B students based on Gemma-3-4B-PT~\citep{gemmateam2025gemma3technicalreport}, Llama-3.2-3B-Instruct~\citep{grattafiori2024llama3herdmodels}, and Qwen-2.5 at 1.5B and 3B~\citep{qwen2.5}. Across matched configurations, the teacher varies among GPT-OSS-120B~\citep{openai2025gptoss120bgptoss20bmodel}, Llama-3.3-70B-Instruct, and Qwen-3-8B~\citep{qwen3}. Gemma-3-27B-it completes the four-candidate pool as an unexposed distractor. Identification uses the probe set not used for SFT, whereas detection and abstention use both probe sets.

For score-based methods, we fit the acceptance threshold on the non-held-out cells of each fold and apply it to the held-out group. SCOUT applies this threshold to the winning margin \(G(S)\) defined in Section~\ref{sec:method-scout}. Student-backbone holdout excludes all cells sharing one backbone, while teacher holdout excludes all cells associated with one teacher and tests transfer to a teacher unseen during threshold fitting. We denote these schemes by \(S\) and \(T\), respectively, following~\citet{rawat2026referencebaseddistillationdetectionllms}.

\vspace{-0.5em}

\begin{table*}[t]
  \centering
  \caption{\textbf{Detection and threshold transfer.} Both 1.5--4B cohorts use disjoint SFT and audit examples. Math uses OMI and s1, and conversation uses OASST1 and Dolly. \(\mathrm{Det.}\) counts students correctly attributed on both probes with the teacher present using student-backbone holdout. \(\mathrm{Acc}(S)\) and \(\mathrm{Acc}(T)\) count correct accept-or-abstain decisions under student-backbone and teacher holdout, respectively, across 4 cells per student and 76 per cohort. \(\mathrm{FP}(S)\) and \(\mathrm{FP}(T)\) count accepted students after teacher removal under the same holdouts. Higher is better for \(\mathrm{Det.}\) and \(\mathrm{Acc}\), and lower for \(\mathrm{FP}\).}
  \label{tab:controlled-detection}
  \resizebox{\textwidth}{!}{%
  \begin{tabular}{l ccccc ccccc}
    \toprule
    & \multicolumn{5}{c}{Math (\(19\) Students)}
    & \multicolumn{5}{c}{Conversation (\(19\) Students)} \\
    \cmidrule(lr){2-6}\cmidrule(r){7-11}
    Method
      & \(\mathrm{Det.}\) & \(\mathrm{Acc}(S)\) & \(\mathrm{Acc}(T)\) & \(\mathrm{FP}(S)\) & \(\mathrm{FP}(T)\)
      & \(\mathrm{Det.}\) & \(\mathrm{Acc}(S)\) & \(\mathrm{Acc}(T)\) & \(\mathrm{FP}(S)\) & \(\mathrm{FP}(T)\) \\
    \midrule
    \rowcolor{black!7}
    \multicolumn{11}{l}{\textit{Reference-based (A pre-distillation reference checkpoint is required.)}} \\
    \rowcolor{black!7}
    DistillDetect
      & 11/19 & 57/76 & 43/76 & 3/19 & 10/19
      & 13/19 & 54/76 & 38/76 & 5/19 & 12/19 \\
    \midrule
    \multicolumn{11}{l}{\textit{Output-only (Only current model outputs are required.)}} \\
    PoS Templates
      & 14/19 & 48/76 & 32/76 & 11/19 & 11/19
      & 11/19 & 64/76 & 53/76 & \textbf{0/19} & \textbf{0/19} \\
    LLM judge (Mistral)
      & 15/19 & 56/76 & 56/76 & 6/19 & 6/19
      & 12/19 & 56/76 & 56/76 & 3/19 & 3/19 \\
    LLM judge (Solar)
      & 15/19 & 54/76 & 54/76 & 4/19 & 4/19
      & \textbf{15/19} & 58/76 & 58/76 & 5/19 & 5/19 \\
    \cmidrule{1-11}
    SCOUT (Ours)
      & \textbf{18/19} & \textbf{74/76} & 58/76 & \textbf{0/19} & \textbf{0/19}
      & \textbf{15/19} & \textbf{68/76} & \textbf{69/76} & \textbf{0/19} & \textbf{0/19} \\
    \quad \( - \) Scaling (\(C\))
      & 16/19 & 68/76 & 57/76 & \textbf{0/19} & \textbf{0/19}
      & 14/19 & 64/76 & 64/76 & 2/19 & 2/19 \\
    \qquad \( - \) Centering (\(D\))
      & 17/19 & 73/76 & \textbf{68/76} & \textbf{0/19} & \textbf{0/19}
      & 12/19 & 61/76 & 65/76 & 2/19 & 2/19 \\
    \bottomrule
  \end{tabular}%
  }
  \vspace{-0.5em}
\end{table*}

\begin{figure*}[t]
  \centering
  \begin{subfigure}[t]{0.475\textwidth}
    \centering
    \includegraphics[width=\linewidth]{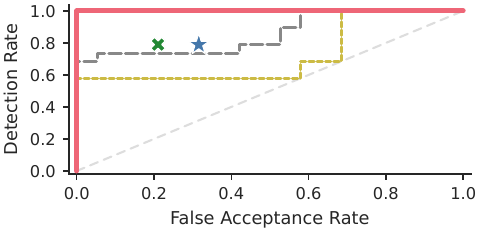}
    \caption{Math cohort}
    \label{fig:detection-roc-a}
  \end{subfigure}\hfill
  \begin{subfigure}[t]{0.475\textwidth}
    \centering
    \includegraphics[width=\linewidth]{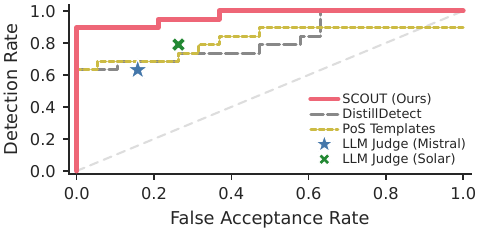}
    \caption{Conversation cohort}
    \label{fig:detection-roc-b}
  \end{subfigure}
  \caption{\textbf{Detection performance across false-acceptance rates.} Axes use the student-level definitions in Table~\ref{tab:controlled-detection}. LLM judges appear as single operating points, and the diagonal denotes chance. SCOUT matches the best math performance and is best in conversation at low false-acceptance rates.}
  \label{fig:detection-roc}
  \vspace{-1.0em}
\end{figure*}

\begin{figure*}[t]
  \centering
  \begin{subfigure}[t]{0.475\textwidth}
    \centering
    \includegraphics[width=\linewidth]{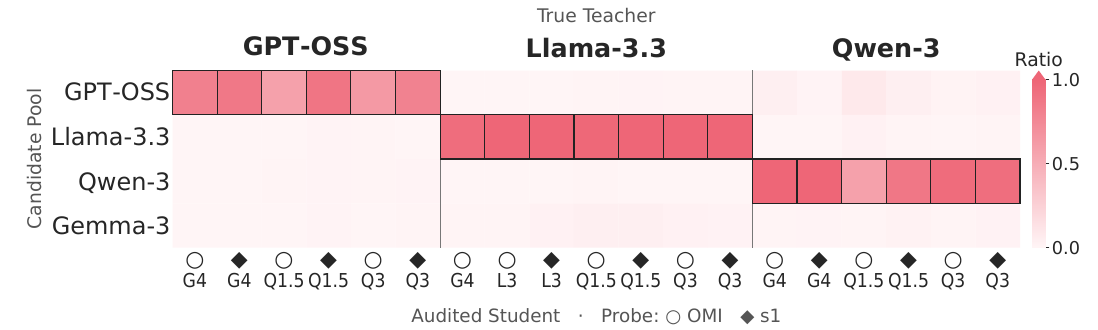}
    \caption{Math cohort}
    \label{fig:controlled-signatures-a}
  \end{subfigure}\hfill
  \begin{subfigure}[t]{0.475\textwidth}
    \centering
    \includegraphics[width=\linewidth]{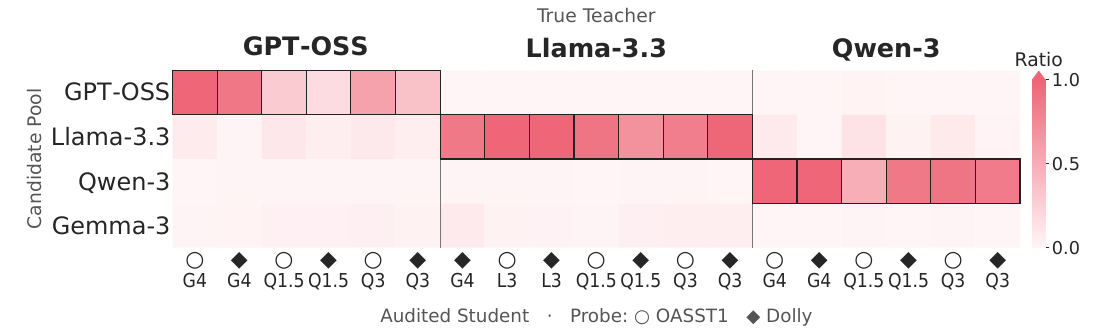}
    \caption{Conversation cohort}
    \label{fig:controlled-signatures-b}
  \end{subfigure}
  \caption{\textbf{Controlled students reproduce teacher-associated syntactic patterns.} Rows are candidate models, and columns are 1.5--4B students grouped by their true teacher. Each cell reports the share of student PoS \(n\)-gram tokens matching the candidate's ten most characteristic patterns, normalized by the share in responses from that candidate. Column labels use G4 for Gemma-3-4B-PT, L3 for Llama-3.2-3B-Instruct, Q1.5 for Qwen-2.5-1.5B, and Q3 for Qwen-2.5-3B.}
  \label{fig:controlled-signatures}
  \vspace{-1.0em}
\end{figure*}

\paragraph{Attribution, Abstention, and Calibration} Table~\ref{tab:controlled-detection} shows that, under the fold-fitted thresholds, SCOUT detects \(18/19\) math students and \(15/19\) conversation students while accepting none after teacher removal. Moving from \(D\) to \(C\) reduces detection by one student in math but raises it by two in conversation. Adding scaling improves detection in both and eliminates false positives in conversation. Thus, gains are not monotonic across stages, but full location-and-scale calibration gives the strongest attribution-abstention balance across domains.

Figure~\ref{fig:detection-roc} shows that this separation is not specific to the fitted thresholds. SCOUT attains student-level AUCs of \(1.000\) and \(0.970\), respectively. At zero false positives, it detects \(19/19\) and \(17/19\) students, compared with \(13/19\) and \(12/19\) for DistillDetect, which requires a pre-distillation reference checkpoint. Appendix~\ref{app:details} describes the controlled cohorts, while Appendix~\ref{app:tables} reports the full calibration ablation, including a 7--8B cohort testing whether the results extend to larger students.

\vspace{-0.5em}

\paragraph{Inheritance of Syntactic Signatures} Figure~\ref{fig:controlled-signatures} visualizes the syntactic evidence captured by profile distance. For each candidate, we identify its ten most characteristic PoS \(n\)-grams relative to the pool and measure their frequency in held-out student responses, normalized by candidate responses. The block-diagonal structure across backbones and domains shows consistent inheritance of teacher-associated patterns. Appendix~\ref{app:ngram-examples} shows response examples, while Appendix~\ref{app:ngram} provides full-size heatmaps and further inheritance analyses.

\vspace{-0.5em}

\begin{wrapfigure}{r}{0.55\textwidth}
  \vspace{-\intextsep}
  \vspace{-0.25em}
  \centering
  \includegraphics[width=\linewidth]{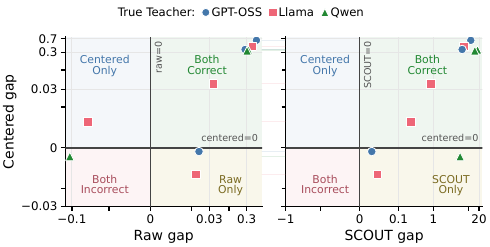}
  \caption{\textbf{Scaling corrects errors after centering.} For ten students on one held-out Dolly prompt, panels plot raw against centered gaps and centered against scaled gaps. Positive gaps indicate correct attribution.}
  \label{fig:margin-stages}
  \vspace{-\intextsep}
  \vspace{-1.5em}
\end{wrapfigure}

\paragraph{Effect of Candidate-Specific Scaling} Figure~\ref{fig:margin-stages} compares raw, centered, and scaled decision gaps for ten students on one held-out Dolly prompt. Raw gaps correctly attribute \(8/10\) students, while centered gaps attribute \(7/10\). The three centered errors remain within \(0.007\) of zero because centering removes baseline proximity but not differences in candidate scale. Dividing by the candidate-specific spread \(\sigma\) moves all ten scaled gaps above zero. This single-prompt example illustrates the correction. The full procedure aggregates \(200\) prompts per probe and requires agreement across probes.

\vspace{0.25em}

\definecolor{scout}{HTML}{EE6677}

\begin{tcolorbox}[
  colback=scout!5!white,
  colframe=scout,
  boxrule=0pt,
  leftrule=2.5pt,
  arc=0.8mm,
  left=3mm, right=3mm, top=2.2mm, bottom=2.2mm,
  before skip=8pt, after skip=5pt
]
{\bfseries\color{scout!85!black} Takeaway 1: SCOUT Enables Output-Only Attribution Across Training Stages.}\par
\vspace{2pt}
{\color{scout!30!white}\hrule height 0.4pt}
\vspace{3pt}
\noindent In controlled cohorts audited immediately after SFT, SCOUT achieves high detection and accepts no candidate after the teacher is removed from the pool. After subsequent training, it attributes every public descendant to the documented teacher and identifies both pre-distillation bases as non-R1 using only current outputs, without a pre-distillation checkpoint.
\end{tcolorbox}

\vspace{-0.5em}

\section{Reading the Signature Along a Training Trajectory}
\label{sec:findings}
\vspace{-0.5em}

The preceding section evaluates SCOUT before and after subsequent training. In this section, we use the trajectory margin change \(\Delta G_T\) defined in \eqref{eq:delta-margin} to examine when teacher-associated evidence emerges and whether it persists through subsequent training.

\subsection{Teacher Signatures Emerge and Persist Across Training Trajectories}
\label{sec:traj-ladders}

First, we examine thirteen public trajectories spanning five teachers, dense and mixture-of-experts architectures, and roughly 1B active to 32B dense parameters. Figure~\ref{fig:trajectories-a} shows that, on OMI, \(\Delta G_T\) becomes positive immediately after distillation and remains positive at every later checkpoint in all thirteen trajectories. The same pattern holds on s1. Appendix~\ref{app:template} reports the s1 results and documents each trajectory, its designated teacher, and its training data.

Furthermore, we examine the two R1-distilled parents and twelve public descendants from Section~\ref{sec:exp-descendants}. For each lineage, the Qwen2.5-Math base is \(S^{(0)}\), the R1-distilled parent is \(S^{(1)}\), and each descendant is a subsequent training outcome \(S^{(2)}\). Figure~\ref{fig:trajectories-b} shows \(\Delta G_T>0\) on OMI for both parents and all twelve descendants. The same holds on s1, so every descendant retains an R1 signature relative to \(S^{(0)}\). Compared with its corresponding \(S^{(1)}\) parent, the margin is higher for \(7/12\) descendants on OMI and \(3/12\) on s1. Thus, the signature persists across the panel, while further strengthening varies among descendants.

\begin{figure*}[t]
\centering
\begin{subfigure}[b]{0.44\textwidth}
\centering
\includegraphics[width=\linewidth]{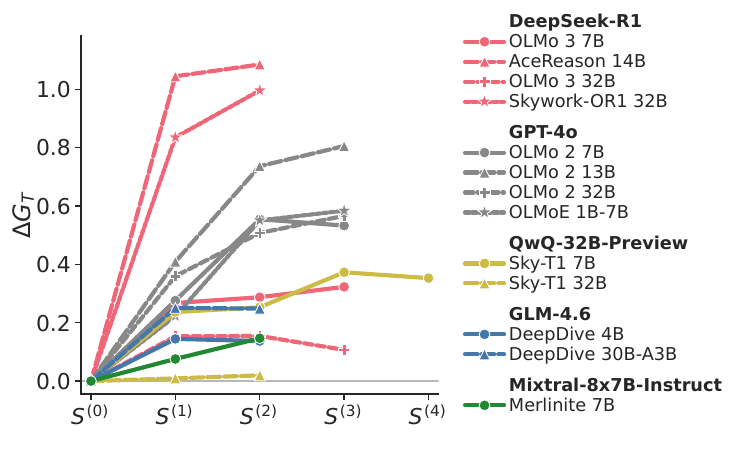}
\caption{Thirteen public ladders}
\label{fig:trajectories-a}
\end{subfigure}\hfill
\begin{subfigure}[b]{0.27\textwidth}
\centering
\includegraphics[width=\linewidth]{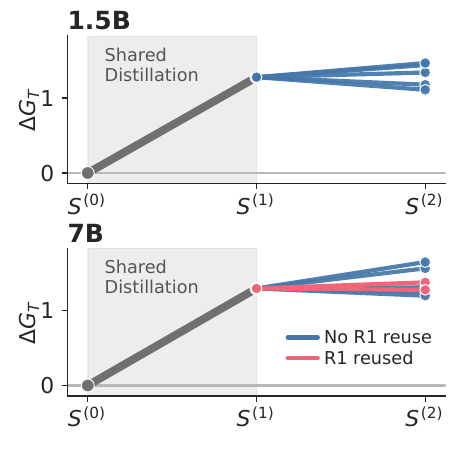}
\caption{Qwen2.5-Math lineages}
\label{fig:trajectories-b}
\end{subfigure}\hfill
\begin{subfigure}[b]{0.27\textwidth}
\centering
\includegraphics[width=\linewidth]{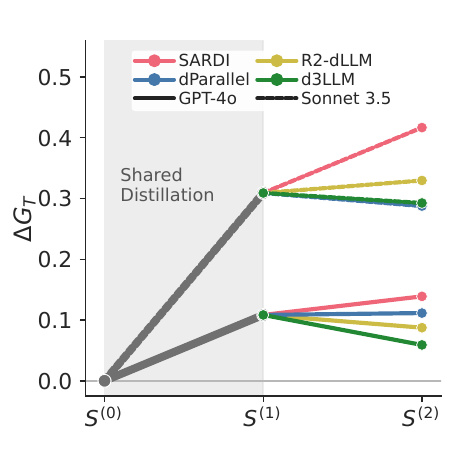}
\caption{Dream-7B lineages}
\label{fig:trajectories-c}
\end{subfigure}
\caption{\textbf{Source-margin shifts from the pre-distillation state on OMI.} All panels use the same 13-candidate pool. Positive values indicate a larger margin toward the indicated teacher than at \(S^{(0)}\), and every post-distillation value is positive. Shading marks distillation in \textbf{(b)} and \textbf{(c)}.}
\label{fig:trajectories}
\vspace{-1.0em}
\end{figure*}

\vspace{-0.5em}

\subsection{The Signature Extends Beyond Autoregressive Models}
\label{sec:traj-diffusion}

\vspace{-0.5em}

Because SCOUT compares contiguous PoS \(n\)-grams, the signal could depend on autoregressive token generation. Therefore, we apply the same readout to Dream-7B~\citep{ye2025dream7bdiffusionlarge}, a diffusion language model, evaluating its pre-distillation base, instruction-tuned successor, and four public post-training descendants. Since its instruction data include responses from both GPT-4o~\citep{openai2024gpt4ocard} and Sonnet 3.5~\citep{claude35}, Figure~\ref{fig:trajectories-c} reports separate margins toward the two teacher models. Both margins increase after instruction tuning and remain positive across all four branches on both probes, showing that signatures associated with both documented teachers can emerge and persist in a diffusion language model.

\begin{tcolorbox}[
  colback=scout!5!white,
  colframe=scout,
  boxrule=0pt,
  leftrule=2.5pt,
  arc=0.8mm,
  left=3mm, right=3mm, top=2.2mm, bottom=2.2mm,
  before skip=8pt, after skip=5pt
]
{\bfseries\color{scout!85!black} Takeaway 2: Teacher Signatures Remain Readable Through Diverse Post-Training.}\par
\vspace{2pt}
{\color{scout!30!white}\hrule height 0.4pt}
\vspace{3pt}
\noindent SCOUT traces teacher-associated syntactic signatures that emerge during distillation and remain readable after preference optimization and reinforcement learning across dense and mixture-of-experts autoregressive models. In the diffusion model, margins toward both documented contributors remain positive under mixed-teacher supervision.
\end{tcolorbox}

\section{Conclusion}
\label{sec:conclusion}
\vspace{-0.5em}

In this paper, we introduce SCOUT, an output-only method for distillation attribution. It aggregates recurring syntactic patterns into candidate profiles, filters low-contrast patterns, and calibrates student--candidate distances against inter-candidate distances. Across controlled students and public descendants, SCOUT supports open-set attribution in controlled cohorts and closed-set retrospective attribution after subsequent training. Across training trajectories, teacher-associated signatures emerge during distillation and persist through preference optimization and reinforcement learning. The same signal extends to a diffusion language model under mixed-teacher supervision.

\vspace{-0.5em}

\paragraph{Limitations}
SCOUT depends on the supplied candidate pool, and principled pool construction for third-party auditors remains open. Where applicable, we follow the pool construction used by DistillDetect~\citep{rawat2026referencebaseddistillationdetectionllms}, but this protocol does not determine which models to include in new settings. Pool composition also bounds attribution resolution. A sibling-enriched pool sharply reduces exact attribution from current outputs, as detailed in Appendix~\ref{app:tables}.


\vspace{-0.5em}

\paragraph{Future Work}
Promising directions for future work include testing whether the observed persistence extends to watermarks and fingerprints under matched post-training and response-rewriting interventions, and evaluating adversaries optimized to evade attribution.




\section*{AI Use Statement}

Generative AI tools were used during manuscript preparation to provide feedback, assist literature search, and improve clarity. LLMs were also used experimentally to generate training and audit responses and serve as judge baselines, as described in Section~\ref{sec:experiments} and Appendix~\ref{app:details}. These experimental roles followed the stated procedures, while interpretation and scientific conclusions remained with the authors. All AI-assisted manuscript content was reviewed and verified by the authors, who take full responsibility for the final content and claims.

\section*{Ethics Statement}

Output-based attribution can inform provenance audits, but should not be treated as conclusive evidence of a model's training history. For example, close source relatives can reduce exact-teacher attribution. Appendix~\ref{app:theory} shows that subsequent SFT can make an earlier source unrecoverable by the evaluated methods under the tested training budgets and probe sets. This result has dual-use implications because it may inform attempts to evade attribution. It also shows that non-detection does not establish non-use. Therefore, attribution findings should be assessed alongside independent provenance evidence.

\section*{Reproducibility Statement}

The response representation is specified in Section~\ref{sec:methodology}, with candidate-pool calibration in Eqs.~\plaineqref{eq:finite-pool-null}--\plaineqref{eq:calibrated-score}. Appendix~\ref{app:ngram} reports preprocessing, filtering, and robustness analyses, Appendix~\ref{app:theory} gives the finite-prompt analysis, and Appendix~\ref{app:details} documents generation, training, and evaluation procedures.

\clearpage

\bibliography{iclr2027_conference}

@inproceedings{shaib-etal-2024-detection,
    title = "Detection and Measurement of Syntactic Templates in Generated Text",
    author = "Shaib, Chantal  and
      Elazar, Yanai  and
      Li, Junyi Jessy  and
      Wallace, Byron C",
    editor = "Al-Onaizan, Yaser  and
      Bansal, Mohit  and
      Chen, Yun-Nung",
    booktitle = "Proceedings of the 2024 Conference on Empirical Methods in Natural Language Processing",
    month = nov,
    year = "2024",
    address = "Miami, Florida, USA",
    publisher = "Association for Computational Linguistics",
    url = "https://aclanthology.org/2024.emnlp-main.368/",
    doi = "10.18653/v1/2024.emnlp-main.368",
    pages = "6416--6431"
}

@inproceedings{hsieh-etal-2023-distilling,
    title = "Distilling Step-by-Step! Outperforming Larger Language Models with Less Training Data and Smaller Model Sizes",
    author = "Hsieh, Cheng-Yu  and
      Li, Chun-Liang  and
      Yeh, Chih-kuan  and
      Nakhost, Hootan  and
      Fujii, Yasuhisa  and
      Ratner, Alex  and
      Krishna, Ranjay  and
      Lee, Chen-Yu  and
      Pfister, Tomas",
    editor = "Rogers, Anna  and
      Boyd-Graber, Jordan  and
      Okazaki, Naoaki",
    booktitle = "Findings of the Association for Computational Linguistics: ACL 2023",
    month = jul,
    year = "2023",
    address = "Toronto, Canada",
    publisher = "Association for Computational Linguistics",
    url = "https://aclanthology.org/2023.findings-acl.507/",
    doi = "10.18653/v1/2023.findings-acl.507",
    pages = "8003--8017"
}

@inproceedings{xu-etal-2024-instructional,
    title = "Instructional Fingerprinting of Large Language Models",
    author = "Xu, Jiashu  and
      Wang, Fei  and
      Ma, Mingyu  and
      Koh, Pang Wei  and
      Xiao, Chaowei  and
      Chen, Muhao",
    editor = "Duh, Kevin  and
      Gomez, Helena  and
      Bethard, Steven",
    booktitle = "Proceedings of the 2024 Conference of the North American Chapter of the Association for Computational Linguistics: Human Language Technologies (Volume 1: Long Papers)",
    month = jun,
    year = "2024",
    address = "Mexico City, Mexico",
    publisher = "Association for Computational Linguistics",
    url = "https://aclanthology.org/2024.naacl-long.180/",
    doi = "10.18653/v1/2024.naacl-long.180",
    pages = "3277--3306"
}

@inproceedings{wadhwa-etal-2025-taught,
    title = "Who Taught You That? Tracing Teachers in Model Distillation",
    author = "Wadhwa, Somin  and
      Shaib, Chantal  and
      Amir, Silvio  and
      Wallace, Byron C",
    editor = "Che, Wanxiang  and
      Nabende, Joyce  and
      Shutova, Ekaterina  and
      Pilehvar, Mohammad Taher",
    booktitle = "Findings of the Association for Computational Linguistics: ACL 2025",
    month = jul,
    year = "2025",
    address = "Vienna, Austria",
    publisher = "Association for Computational Linguistics",
    url = "https://aclanthology.org/2025.findings-acl.173/",
    doi = "10.18653/v1/2025.findings-acl.173",
    pages = "3307--3315",
    ISBN = "979-8-89176-256-5"
}

@misc{rawat2026referencebaseddistillationdetectionllms,
      title={Reference-Based Distillation Detection in LLMs}, 
      author={Rajat Rawat and Sizhe Chen and Akshay Anand and Michael Duan and Bob Rotsted and Sewon Min},
      year={2026},
      eprint={2607.09692},
      archivePrefix={arXiv},
      primaryClass={cs.LG},
      url={https://arxiv.org/abs/2607.09692}, 
}

@misc{
    zeng2026distillation,
    title={Distillation Lineage Inspector: Black-Box Auditing of Model Distillation in {LLM}s},
    author={Zhirui Zeng and Jiamou Liu and Meng-Fen Chiang and Jialing He and Shangwei Guo},
    year={2026},
    url={https://openreview.net/forum?id=mcUWhTcqTx}
}

@inproceedings{savani2025antidistillation,
    title={Antidistillation Sampling},
    author={Yash Savani and Asher Trockman and Zhili Feng and Yixuan Even Xu and Avi Schwarzschild and Alexander Robey and Marc Anton Finzi and J Zico Kolter},
    booktitle={The Thirty-ninth Annual Conference on Neural Information Processing Systems},
    year={2025},
    url={https://openreview.net/forum?id=Vo2UHqMu8t}
}

@InProceedings{pmlr-v202-kirchenbauer23a,
  title = 	 {A Watermark for Large Language Models},
  author =       {Kirchenbauer, John and Geiping, Jonas and Wen, Yuxin and Katz, Jonathan and Miers, Ian and Goldstein, Tom},
  booktitle = 	 {Proceedings of the 40th International Conference on Machine Learning},
  pages = 	 {17061--17084},
  year = 	 {2023},
  editor = 	 {Krause, Andreas and Brunskill, Emma and Cho, Kyunghyun and Engelhardt, Barbara and Sabato, Sivan and Scarlett, Jonathan},
  volume = 	 {202},
  series = 	 {Proceedings of Machine Learning Research},
  month = 	 {23--29 Jul},
  publisher =    {PMLR},
  url = 	 {https://proceedings.mlr.press/v202/kirchenbauer23a.html}
}

@inproceedings{NEURIPS2024_2567c95f,
	author = {Sander, Tom and Fernandez, Pierre and Durmus, Alain and Douze, Matthijs and Furon, Teddy},
	booktitle = {Advances in Neural Information Processing Systems},
	editor = {A. Globerson and L. Mackey and D. Belgrave and A. Fan and U. Paquet and J. Tomczak and C. Zhang},
	pages = {21079--21113},
	publisher = {Curran Associates, Inc.},
	title = {Watermarking Makes Language Models Radioactive},
	volume = {37},
	year = {2024}}

@misc{yang2026askingbackinteractionlayerantidistillation,
      title={Asking Back: Interaction-Layer Antidistillation Watermarks}, 
      author={Guang Yang and Amir Ghasemian and Fengchen Liu and Zhong Wang and Ninareh Mehrabi and Homa Hosseinmardi},
      year={2026},
      eprint={2605.16462},
      archivePrefix={arXiv},
      primaryClass={cs.CR},
      url={https://arxiv.org/abs/2605.16462}, 
}

@inproceedings{ma-etal-2026-protecting,
    title = "Protecting Language Models Against Unauthorized Distillation through Trace Rewriting",
    author = "Ma, Xinhang  and
      Yeoh, William  and
      Zhang, Ning  and
      Vorobeychik, Yevgeniy",
    editor = "Liakata, Maria  and
      Moreira, Viviane P.  and
      Zhang, Jiajun  and
      Jurgens, David",
    booktitle = "Proceedings of the 64th Annual Meeting of the {A}ssociation for {C}omputational {L}inguistics (Volume 1: Long Papers)",
    month = jul,
    year = "2026",
    address = "San Diego, California, United States",
    publisher = "Association for Computational Linguistics",
    url = "https://aclanthology.org/2026.acl-long.519/",
    doi = "10.18653/v1/2026.acl-long.519",
    pages = "11307--11324",
    ISBN = "979-8-89176-390-6"
}

@misc{sander2026textseallocalizedllmwatermark,
      title={TextSeal: A Localized LLM Watermark for Provenance \& Distillation Protection}, 
      author={Tom Sander and Hongyan Chang and Tomáš Souček and Tuan Tran and Valeriu Lacatusu and Sylvestre-Alvise Rebuffi and Alexandre Mourachko and Surya Parimi and Christophe Ropers and Rashel Moritz and Vanessa Stark and Hady Elsahar and Pierre Fernandez},
      year={2026},
      eprint={2605.12456},
      archivePrefix={arXiv},
      primaryClass={cs.CR},
      url={https://arxiv.org/abs/2605.12456}, 
}

@InProceedings{Junhao_2025_ICCV,
    author    = {Junhao, Wei and Zhe, Yu and Sakuma, Jun},
    title     = {Disrupting Model Merging: A Parameter-Level Defense Without Sacrificing Accuracy},
    booktitle = {Proceedings of the IEEE/CVF International Conference on Computer Vision (ICCV)},
    month     = {October},
    year      = {2025},
    pages     = {17698-17707}
}

@misc{wang2025modelunmergingmakingmodels,
      title={Model Unmerging: Making Your Models Unmergeable for Secure Model Sharing}, 
      author={Zihao Wang and Enneng Yang and Lu Yin and Shiwei Liu and Li Shen},
      year={2025},
      eprint={2509.01548},
      archivePrefix={arXiv},
      primaryClass={cs.LG},
      url={https://arxiv.org/abs/2509.01548}, 
}

@inproceedings{jang2026making,
    title={Making Models Unmergeable via Scaling-Sensitive Loss Landscape},
    author={Minwoo Jang and Hoyoung Kim and Jabin Koo and Jungseul Ok},
    booktitle={Forty-third International Conference on Machine Learning},
    year={2026},
    url={https://openreview.net/forum?id=o9NThz6auq}
}

@inproceedings{yax2025phylolm,
    title={Phylo{LM}: Inferring the Phylogeny of Large Language Models and Predicting their Performances in Benchmarks},
    author={Nicolas Yax and Pierre-Yves Oudeyer and Stefano Palminteri},
    booktitle={The Thirteenth International Conference on Learning Representations},
    year={2025},
    url={https://openreview.net/forum?id=rTQNGQxm4K}
}

@inproceedings{foley-etal-2023-matching,
    title = "Matching Pairs: Attributing Fine-Tuned Models to their Pre-Trained Large Language Models",
    author = "Foley, Myles  and
      Rawat, Ambrish  and
      Lee, Taesung  and
      Hou, Yufang  and
      Picco, Gabriele  and
      Zizzo, Giulio",
    editor = "Rogers, Anna  and
      Boyd-Graber, Jordan  and
      Okazaki, Naoaki",
    booktitle = "Proceedings of the 61st Annual Meeting of the Association for Computational Linguistics (Volume 1: Long Papers)",
    month = jul,
    year = "2023",
    address = "Toronto, Canada",
    publisher = "Association for Computational Linguistics",
    url = "https://aclanthology.org/2023.acl-long.410/",
    doi = "10.18653/v1/2023.acl-long.410",
    pages = "7423--7442"
}

@inproceedings{Shao_2026,
   title={Reading Between the Lines: Towards Reliable Black-box LLM Fingerprinting via Zeroth-order Gradient Estimation},
   url={http://dx.doi.org/10.1145/3774904.3792196},
   DOI={10.1145/3774904.3792196},
   booktitle={Proceedings of the ACM Web Conference 2026},
   publisher={ACM},
   author={Shao, Shuo and Li, Yiming and Yao, Hongwei and Chen, Yifei and Yang, Yuchen and Qin, Zhan},
   year={2026},
   month=Apr, pages={2637–2648}
}

@inproceedings{nikolic2025model,
    title={Model Provenance Testing for Large Language Models},
    author={Ivica Nikolic and Teodora Baluta and Prateek Saxena},
    booktitle={The Thirty-ninth Annual Conference on Neural Information Processing Systems},
    year={2025},
    url={https://openreview.net/forum?id=Iy4cAXotrf}
}

@inproceedings{hu-etal-2026-fingerprinting,
    title = "Fingerprinting {LLM}s via Prompt Injection",
    author = "Hu, Yuepeng  and
      Jiang, Zhengyuan  and
      Li, Mengyuan  and
      Ahmed, Osama  and
      Huang, Zhicong  and
      Hong, Cheng  and
      Gong, Neil Zhenqiang",
    editor = "Liakata, Maria  and
      Moreira, Viviane P.  and
      Zhang, Jiajun  and
      Jurgens, David",
    booktitle = "Proceedings of the 64th Annual Meeting of the {A}ssociation for {C}omputational {L}inguistics (Volume 1: Long Papers)",
    month = jul,
    year = "2026",
    address = "San Diego, California, United States",
    publisher = "Association for Computational Linguistics",
    url = "https://aclanthology.org/2026.acl-long.541/",
    doi = "10.18653/v1/2026.acl-long.541",
    pages = "11795--11810",
    ISBN = "979-8-89176-390-6"
}

@INPROCEEDINGS{8806737,
  author={Juuti, Mika and Szyller, Sebastian and Marchal, Samuel and Asokan, N.},
  booktitle={2019 IEEE European Symposium on Security and Privacy (EuroS\&P)}, 
  title={PRADA: Protecting Against DNN Model Stealing Attacks}, 
  year={2019},
  volume={},
  number={},
  pages={512-527},
  doi={10.1109/EuroSP.2019.00044}}

@misc{muennighoff2025s1simpletesttimescaling,
      title={s1: Simple test-time scaling}, 
      author={Niklas Muennighoff and Zitong Yang and Weijia Shi and Xiang Lisa Li and Li Fei-Fei and Hannaneh Hajishirzi and Luke Zettlemoyer and Percy Liang and Emmanuel Candès and Tatsunori Hashimoto},
      year={2025},
      eprint={2501.19393},
      archivePrefix={arXiv},
      primaryClass={cs.CL},
      url={https://arxiv.org/abs/2501.19393}, 
}

@inproceedings{toshniwal2024openmathinstruct,
    title={OpenMathInstruct-2: Accelerating {AI} for Math with Massive Open-Source Instruction Data},
    author={Shubham Toshniwal and Wei Du and Ivan Moshkov and Branislav Kisacanin and Alexan Ayrapetyan and Igor Gitman},
    booktitle={The 4th Workshop on Mathematical Reasoning and AI at NeurIPS'24},
    year={2024},
    url={https://openreview.net/forum?id=l5FDMofecw}
}

@inproceedings{3666122.3668186,
    author = {K{\"o}pf, Andreas and Kilcher, Yannic and von R{\"u}tte, Dimitri and Anagnostidis, Sotiris and Tam, Zhi-Rui and Stevens, Keith and Barhoum, Abdullah and Duc, Nguyen Minh and Stanley, Oliver and Nagyfi, Rich{\'a}rd and ES, Shahul and Suri, Sameer and Glushkov, David and Dantuluri, Arnav and Maguire, Andrew and Schuhmann, Christoph and Nguyen, Huu and Mattick, Alexander},
    title = {OpenAssistant conversations - democratizing large language model alignment},
    year = {2023},
    publisher = {Curran Associates Inc.},
    address = {Red Hook, NY, USA},
    booktitle = {Proceedings of the 37th International Conference on Neural Information Processing Systems},
    articleno = {2064},
    numpages = {13},
    location = {New Orleans, LA, USA},
    series = {NIPS '23}
}

@misc{DatabricksBlog2023DollyV2,
    author    = {Mike Conover and Matt Hayes and Ankit Mathur and Jianwei Xie and Jun Wan and Sam Shah and Ali Ghodsi and Patrick Wendell and Matei Zaharia and Reynold Xin},
    title     = {Free Dolly: Introducing the World's First Truly Open Instruction-Tuned LLM},
    year      = {2023},
    url       = {https://www.databricks.com/blog/2023/04/12/dolly-first-open-commercially-viable-instruction-tuned-llm},
}

@article{10.1093/llc/17.3.267,
	author = {Burrows, John},
	journal = {Literary and Linguistic Computing},
	month = {09},
	number = {3},
	pages = {267-287},
	title = {`Delta': a Measure of Stylistic Difference and a Guide to Likely Authorship},
	volume = {17},
	year = {2002}}

@inproceedings{watson2022on,
    title={On the Importance of Difficulty Calibration in Membership Inference Attacks},
    author={Lauren Watson and Chuan Guo and Graham Cormode and Alexandre Sablayrolles},
    booktitle={International Conference on Learning Representations},
    year={2022},
    url={https://openreview.net/forum?id=3eIrli0TwQ}
}

@inproceedings{rosenberg92_icslp,
  title     = {{The use of cohort normalized scores for speaker verification}},
  author    = {Aaron E. Rosenberg and Joel DeLong and Chin-Hui Lee and Biing-Hwang Juang and Frank K. Soong},
  year      = {1992},
  booktitle = {{2nd International Conference on Spoken Language Processing (ICSLP 1992)}},
  pages     = {599--602},
  doi       = {10.21437/ICSLP.1992-176},
  issn      = {2958-1796},
}

@misc{olmo2026olmo3,
      title={Olmo 3}, 
      author={{OLMo Team}},
      year={2026},
      eprint={2512.13961},
      archivePrefix={arXiv},
      primaryClass={cs.CL},
      url={https://arxiv.org/abs/2512.13961}, 
}

@misc{he2025skyworkopenreasoner1,
      title={Skywork Open Reasoner 1 Technical Report}, 
      author={Jujie He and Jiacai Liu and Chris Yuhao Liu and Rui Yan and Chaojie Wang and Peng Cheng and Xiaoyu Zhang and Fuxiang Zhang and Jiacheng Xu and Wei Shen and Siyuan Li and Liang Zeng and Tianwen Wei and Cheng Cheng and Bo An and Yang Liu and Yahui Zhou},
      year={2025},
      eprint={2505.22312},
      archivePrefix={arXiv},
      primaryClass={cs.LG},
      url={https://arxiv.org/abs/2505.22312}, 
}

@misc{chen2025acereasonnemotronadvancingmathcode,
      title={AceReason-Nemotron: Advancing Math and Code Reasoning through Reinforcement Learning}, 
      author={Yang Chen and Zhuolin Yang and Zihan Liu and Chankyu Lee and Peng Xu and Mohammad Shoeybi and Bryan Catanzaro and Wei Ping},
      year={2025},
      eprint={2505.16400},
      archivePrefix={arXiv},
      primaryClass={cs.LG},
      url={https://arxiv.org/abs/2505.16400}, 
}

@inproceedings{
    walsh2025,
    title={2 {OLM}o 2 Furious ({COLM}{\textquoteright}s Version)},
    author={Evan Pete Walsh and others},
    booktitle={Second Conference on Language Modeling},
    year={2025},
    url={https://openreview.net/forum?id=2ezugTT9kU}
}

@misc{sky_t1_2025,
  author       = {{NovaSky Team}},
  title        = {Sky-T1: Train your own O1 preview model within \$450},
  howpublished = {https://novasky-ai.github.io/posts/sky-t1},
  year         = {2025}
}

@misc{sudalairaj2024lablargescalealignmentchatbots,
      title={LAB: Large-Scale Alignment for ChatBots}, 
      author={Shivchander Sudalairaj and Abhishek Bhandwaldar and Aldo Pareja and Kai Xu and David D. Cox and Akash Srivastava},
      year={2024},
      eprint={2403.01081},
      archivePrefix={arXiv},
      primaryClass={cs.CL},
      url={https://arxiv.org/abs/2403.01081}, 
}

@misc{zhang2026chainingevidencerobustreinforcement,
      title={Chaining the Evidence: Robust Reinforcement Learning for Deep Search Agents with Citation-Aware Rubric Rewards}, 
      author={Jiajie Zhang and Xin Lv and Ling Feng and Lei Hou and Juanzi Li},
      year={2026},
      eprint={2601.06021},
      archivePrefix={arXiv},
      primaryClass={cs.CL},
      url={https://arxiv.org/abs/2601.06021}, 
}

@misc{ye2025dream7bdiffusionlarge,
      title={Dream 7B: Diffusion Large Language Models}, 
      author={Jiacheng Ye and Zhihui Xie and Lin Zheng and Jiahui Gao and Zirui Wu and Xin Jiang and Zhenguo Li and Lingpeng Kong},
      year={2025},
      eprint={2508.15487},
      archivePrefix={arXiv},
      primaryClass={cs.CL},
      url={https://arxiv.org/abs/2508.15487}, 
}

@inproceedings{
    chen2026dparallel,
    title={dParallel: Learnable Parallel Decoding for d{LLM}s},
    author={Zigeng Chen and Gongfan Fang and Xinyin Ma and Ruonan Yu and Xinchao Wang},
    booktitle={The Fourteenth International Conference on Learning Representations},
    year={2026},
    url={https://openreview.net/forum?id=hVOcstAURb}
}

@inproceedings{
    junger2026selfaugmenting,
    title={Self-Augmenting Retrieval for Diffusion Language Models},
    author={Paul J{\"u}nger and Justin Lovelace and Linxi Zhao and Dongyoung Go and Kilian Q Weinberger},
    booktitle={Forty-third International Conference on Machine Learning},
    year={2026},
    url={https://openreview.net/forum?id=uiE6BIx2vC}
}

@inproceedings{
    qian2026dllm,
    title={d3{LLM}: Ultra-Fast Diffusion {LLM} using Pseudo-Trajectory Distillation},
    author={Yu-Yang Qian and Junda Su and Lanxiang Hu and Peiyuan Zhang and Zhijie Deng and Peng Zhao and Hao Zhang},
    booktitle={Forty-third International Conference on Machine Learning},
    year={2026},
    url={https://openreview.net/forum?id=rzBAQT2Fkg}
}

@misc{du2026r2dllmacceleratingdiffusionlarge,
      title={$R^2$-dLLM: Accelerating Diffusion Large Language Models via Spatio-Temporal Redundancy Reduction}, 
      author={Zhenbang Du and Kejing Xia and Xinrui Zhong and Yonggan Fu and Nicolai Oswald and Binfei Ji and Brucek Khailany and Pavlo Molchanov and Yingyan Lin},
      year={2026},
      eprint={2604.18995},
      archivePrefix={arXiv},
      primaryClass={cs.CL},
      url={https://arxiv.org/abs/2604.18995}, 
}

@article{qwen3,
    title={Qwen3 Technical Report}, 
    author={An Yang and others},
    journal = {arXiv preprint arXiv:2505.09388},
    year={2025}
}

@article{qwen2.5,
    title   = {Qwen2.5 Technical Report}, 
    author={An Yang and others},
    journal = {arXiv preprint arXiv:2412.15115},
    year    = {2024}
}

@misc{yang2024qwen25mathtechnicalreportmathematical,
      title={Qwen2.5-Math Technical Report: Toward Mathematical Expert Model via Self-Improvement}, 
      author={An Yang and Beichen Zhang and Binyuan Hui and Bofei Gao and Bowen Yu and Chengpeng Li and Dayiheng Liu and Jianhong Tu and Jingren Zhou and Junyang Lin and Keming Lu and Mingfeng Xue and Runji Lin and Tianyu Liu and Xingzhang Ren and Zhenru Zhang},
      year={2024},
      eprint={2409.12122},
      archivePrefix={arXiv},
      primaryClass={cs.CL},
      url={https://arxiv.org/abs/2409.12122}, 
}

@misc{grattafiori2024llama3herdmodels,
      title={The Llama 3 Herd of Models}, 
      author={Aaron Grattafiori and others},
      year={2024},
      eprint={2407.21783},
      archivePrefix={arXiv},
      primaryClass={cs.AI},
      url={https://arxiv.org/abs/2407.21783}, 
}

@misc{jiang2023mistral7b,
      title={Mistral 7B}, 
      author={Albert Q. Jiang and Alexandre Sablayrolles and Arthur Mensch and Chris Bamford and Devendra Singh Chaplot and Diego de las Casas and Florian Bressand and Gianna Lengyel and Guillaume Lample and Lucile Saulnier and Lélio Renard Lavaud and Marie-Anne Lachaux and Pierre Stock and Teven Le Scao and Thibaut Lavril and Thomas Wang and Timothée Lacroix and William El Sayed},
      year={2023},
      eprint={2310.06825},
      archivePrefix={arXiv},
      primaryClass={cs.CL},
      url={https://arxiv.org/abs/2310.06825}, 
}

@article{deepseekai2025r1,
	author = {Guo, Daya and others},
	journal = {Nature},
	number = {8081},
	pages = {633--638},
	title = {DeepSeek-R1 incentivizes reasoning in LLMs through reinforcement learning},
	volume = {645},
	year = {2025}}

@misc{jiang2024mixtralexperts,
      title={Mixtral of Experts}, 
      author={Albert Q. Jiang and others},
      year={2024},
      eprint={2401.04088},
      archivePrefix={arXiv},
      primaryClass={cs.LG},
      url={https://arxiv.org/abs/2401.04088}, 
}

@misc{gemmateam2025gemma3technicalreport,
      title={Gemma 3 Technical Report}, 
      author={{Gemma Team} and others},
      year={2025},
      eprint={2503.19786},
      archivePrefix={arXiv},
      primaryClass={cs.CL},
      url={https://arxiv.org/abs/2503.19786}, 
}

@misc{openai2024gpt4ocard,
      title={GPT-4o System Card},
      author = {{OpenAI}},
      year={2024},
      eprint={2410.21276},
      archivePrefix={arXiv},
      primaryClass={cs.CL},
      url={https://arxiv.org/abs/2410.21276}, 
}

@misc{openai2025gptoss120bgptoss20bmodel,
      title={gpt-oss-120b \& gpt-oss-20b Model Card}, 
      author = {{OpenAI}},
      year={2025},
      eprint={2508.10925},
      archivePrefix={arXiv},
      primaryClass={cs.CL},
      url={https://arxiv.org/abs/2508.10925}, 
}

@misc{team2025glm45agenticreasoningcoding,
      title={GLM-4.5: Agentic, Reasoning, and Coding (ARC) Foundation Models}, 
      author={{GLM-4.5 Team} and others},
      year={2025},
      eprint={2508.06471},
      archivePrefix={arXiv},
      primaryClass={cs.CL},
      url={https://arxiv.org/abs/2508.06471}, 
}

@inproceedings{3666122.3668142,
    author = {Zheng, Lianmin and Chiang, Wei-Lin and Sheng, Ying and Zhuang, Siyuan and Wu, Zhanghao and Zhuang, Yonghao and Lin, Zi and Li, Zhuohan and Li, Dacheng and Xing, Eric P. and Zhang, Hao and Gonzalez, Joseph E. and Stoica, Ion},
    title = {Judging LLM-as-a-judge with MT-bench and Chatbot Arena},
    year = {2023},
    publisher = {Curran Associates Inc.},
    address = {Red Hook, NY, USA},
    booktitle = {Proceedings of the 37th International Conference on Neural Information Processing Systems},
    articleno = {2020},
    numpages = {29},
    location = {New Orleans, LA, USA},
    series = {NIPS '23}
}

@misc{qwq-32b-preview,
    title = {QwQ: Reflect Deeply on the Boundaries of the Unknown},
    url = {https://qwenlm.github.io/blog/qwq-32b-preview/},
    author = {{Qwen Team}},
    month = {November},
    year = {2024}
}

@misc{openai2024openaio1card,
      title={OpenAI o1 System Card}, 
      author = {{OpenAI}},
      year={2024},
      eprint={2412.16720},
      archivePrefix={arXiv},
      primaryClass={cs.AI},
      url={https://arxiv.org/abs/2412.16720}, 
}

@misc{Falcon3,
    title = {The Falcon 3 Family of Open Models},
    url = {https://huggingface.co/blog/falcon3},
    author = {{Falcon-LLM Team}},
    month = {December},
    year = {2024}
}

@misc{brown2020languagemodelsfewshotlearners,
      title={Language Models are Few-Shot Learners}, 
      author={Tom B. Brown and Benjamin Mann and Nick Ryder and Melanie Subbiah and Jared Kaplan and Prafulla Dhariwal and Arvind Neelakantan and Pranav Shyam and Girish Sastry and Amanda Askell and Sandhini Agarwal and Ariel Herbert-Voss and Gretchen Krueger and Tom Henighan and Rewon Child and Aditya Ramesh and Daniel M. Ziegler and Jeffrey Wu and Clemens Winter and Christopher Hesse and Mark Chen and Eric Sigler and Mateusz Litwin and Scott Gray and Benjamin Chess and Jack Clark and Christopher Berner and Sam McCandlish and Alec Radford and Ilya Sutskever and Dario Amodei},
      year={2020},
      eprint={2005.14165},
      archivePrefix={arXiv},
      primaryClass={cs.CL},
      url={https://arxiv.org/abs/2005.14165}, 
}

@misc{muennighoff2024olmoeopenmixtureofexpertslanguage,
      title={OLMoE: Open Mixture-of-Experts Language Models}, 
      author={Niklas Muennighoff and Luca Soldaini and Dirk Groeneveld and Kyle Lo and Jacob Morrison and Sewon Min and Weijia Shi and Pete Walsh and Oyvind Tafjord and Nathan Lambert and Yuling Gu and Shane Arora and Akshita Bhagia and Dustin Schwenk and David Wadden and Alexander Wettig and Binyuan Hui and Tim Dettmers and Douwe Kiela and Ali Farhadi and Noah A. Smith and Pang Wei Koh and Amanpreet Singh and Hannaneh Hajishirzi},
      year={2025},
      eprint={2409.02060},
      archivePrefix={arXiv},
      primaryClass={cs.CL},
      url={https://arxiv.org/abs/2409.02060}, 
}

@misc{sky-t1-7b,
  author       = {{NovaSky Team}},
  title        = {Unlocking the Potential of Reinforcement Learning in Improving Reasoning Models},
  howpublished = {https://novasky-ai.github.io/posts/sky-t1-7b},
  year         = {2025}
}

@misc{deepscaler2025,
  title={DeepScaleR: Surpassing O1-Preview with a 1.5B Model by Scaling RL},
  author={Michael Luo and Sijun Tan and Justin Wong and Xiaoxiang Shi and William Y. Tang and Manan Roongta and Colin Cai and Jeffrey Luo and Li Erran Li and Raluca Ada Popa and Ion Stoica},
  year={2025}
}

@misc{zyphra2025ZR1,
  title     = {ZR1-1.5B: A small but powerful reasoning model for math and code},
  author    = {Zyphra},
  year      = {2025},
}

@article{liu2025prorl,
  title   = {{ProRL}: Prolonged Reinforcement Learning Expands Reasoning Boundaries in Large Language Models},
  author  = {Liu, Mingjie and Diao, Shizhe and Lu, Ximing and Hu, Jian and Dong, Xin and Choi, Yejin and Kautz, Jan and Dong, Yi},
  journal = {arXiv preprint arXiv:2505.24864},
  year    = {2025},
  url     = {https://arxiv.org/abs/2505.24864}
}

@article{li2025drpo,
  title={DRPO: Efficient Reasoning via Decoupled Reward Policy Optimization},
  author={Li, Gang and Chen, Yan and Lin, Ming and Yang, Tianbao},
  journal={arXiv preprint arXiv:2510.04474},
  year={2025}
}

@article{li2025disco,
  title={DisCO: Reinforcing Large Reasoning Models with Discriminative Constrained Optimization},
  author={Li, Gang and Lin, Ming and Galanti, Tomer and Tu, Zhengzhong and Yang, Tianbao},
  journal={arXiv preprint arXiv:2505.12366},
  year={2025}
}

@article{lyu2025exploring,
  title={Exploring the Limit of Outcome Reward for Learning Mathematical Reasoning},
  author={Lyu, Chengqi and Gao, Songyang and Gu, Yuzhe and Zhang, Wenwei and Gao, Jianfei and Liu, Kuikun and Wang, Ziyi and Li, Shuaibin and Zhao, Qian and Huang, Haian and others},
  journal={arXiv preprint arXiv:2502.06781},
  year={2025}
}

@article{liu2025spiral,
  title={SPIRAL: Self-Play on Zero-Sum Games Incentivizes Reasoning via Multi-Agent Multi-Turn Reinforcement Learning},
  author={Liu, Bo and Guertler, Leon and Yu, Simon and Liu, Zichen and Qi, Penghui and Balcells, Daniel and Liu, Mickel and Tan, Cheston and Shi, Weiyan and Lin, Min and Lee, Wee Sun and Jaques, Natasha},
  journal={arXiv preprint arXiv:2506.24119},
  year={2025},
  url={https://arxiv.org/abs/2506.24119}
}

@misc{lightr1proj,
  title  = {{Light-R1}: Curriculum {SFT}, {DPO}, and {RL} for Long {CoT} from Scratch and Beyond},
  author = {Wen, Liang and Cai, Yunke and Xiao, Fenrui and He, Xin and An, Qi and Duan, Zhenyu and Du, Yimin and Liu, Junchen and Tang, Lifu and Lv, Xiaowei and Zou, Haosheng and Deng, Yongchao and Jia, Shousheng and Zhang, Xiangzheng},
  year   = {2025},
  url    = {https://github.com/Qihoo360/Light-R1}
}

@misc{yang2025thinkingpreferenceoptimization,
      title={Thinking Preference Optimization}, 
      author={Wang Yang and Hongye Jin and Jingfeng Yang and Vipin Chaudhary and Xiaotian Han},
      year={2025},
      eprint={2502.13173},
      archivePrefix={arXiv},
      primaryClass={cs.LG},
      url={https://arxiv.org/abs/2502.13173}, 
}

@article{marcus-etal-1993-building,
    title = "Building a Large Annotated Corpus of {E}nglish: The {P}enn {T}reebank",
    author = "Marcus, Mitchell P.  and
      Santorini, Beatrice  and
      Marcinkiewicz, Mary Ann",
    editor = "Hirschberg, Julia",
    journal = "Computational Linguistics",
    volume = "19",
    number = "2",
    year = "1993",
    address = "Cambridge, MA",
    publisher = "MIT Press",
    url = "https://aclanthology.org/J93-2004/",
    pages = "313--330"
}

@inproceedings{bird-loper-2004-nltk,
    title = "{NLTK}: The Natural Language Toolkit",
    author = "Bird, Steven  and
      Loper, Edward",
    booktitle = "Proceedings of the {ACL} Interactive Poster and Demonstration Sessions",
    month = jul,
    year = "2004",
    address = "Barcelona, Spain",
    publisher = "Association for Computational Linguistics",
    url = "https://aclanthology.org/P04-3031/",
    pages = "214--217"
}

@article{zhou2023instruction,
  title={Instruction-Following Evaluation for Large Language Models},
  author={Zhou, Jeffrey and Lu, Tianjian and Mishra, Swaroop and Brahma, Siddhartha and Basu, Sujoy and Luan, Yi and Zhou, Denny and Hou, Le},
  journal={arXiv preprint arXiv:2311.07911},
  year={2023}
}

@article{cobbe2021gsm8k,
  title={Training Verifiers to Solve Math Word Problems},
  author={Cobbe, Karl and Kosaraju, Vineet and Bavarian, Mohammad and Chen, Mark and Jun, Heewoo and Kaiser, Lukasz and Plappert, Matthias and Tworek, Jerry and Hilton, Jacob and Nakano, Reiichiro and Hesse, Christopher and Schulman, John},
  journal={arXiv preprint arXiv:2110.14168},
  year={2021}
}

@article{hendrycksmath2021,
  title={Measuring Mathematical Problem Solving With the MATH Dataset},
  author={Dan Hendrycks and Collin Burns and Saurav Kadavath and Akul Arora and Steven Basart and Eric Tang and Dawn Song and Jacob Steinhardt},
  journal={NeurIPS},
  year={2021}
}

@article{hoeffding1963probability,
 ISSN = {01621459, 1537274X},
 URL = {http://www.jstor.org/stable/2282952},
 author = {Wassily Hoeffding},
 journal = {Journal of the American Statistical Association},
 number = {301},
 pages = {13--30},
 publisher = {[American Statistical Association, Taylor & Francis, Ltd.]},
 title = {Probability Inequalities for Sums of Bounded Random Variables},
 volume = {58},
 year = {1963}
}

@misc{openai2026adversarialdistillation,
  author = {{OpenAI}},
  title = {Update to the {U.S.} House Select Committee on Strategic Competition between the {United States} and the {Chinese Communist Party}},
  year = {2026}, month = feb, howpublished = {Letter to the U.S. House Select Committee},
  url = {https://cdn.openai.com/pdf/045aa967-ee96-4a09-94ee-3098ddf6db2c/OpenAI-US-House-Select-Cmte-Update-%5B021226%5D.pdf},
  note = {Dated 12 February 2026. Accessed 14 August 2026}}

@misc{openai2026terms,
  author = {{OpenAI}}, title = {Terms of Use},
  year = {2026}, howpublished = {OpenAI Policies},
  url = {https://openai.com/policies/row-terms-of-use/},
  note = {Accessed 14 August 2026}}

@misc{anthropic2026distillationattacks,
  author = {{Anthropic}}, title = {Detecting and Preventing Distillation Attacks},
  year = {2026}, month = feb, howpublished = {Anthropic News},
  url = {https://www.anthropic.com/news/detecting-and-preventing-distillation-attacks},
  note = {Published 23 February 2026. Accessed 14 August 2026}}

@misc{anthropic2025commercialterms,
  author       = {{Anthropic}},
  title        = {Commercial Terms of Service},
  year         = {2025},
  month        = jun,
  howpublished = {Anthropic Legal},
  url          = {https://www.anthropic.com/legal/commercial-terms},
  note         = {Effective 17 June 2025. Accessed 14 August 2026}
}

@misc{mistral_medium_3,
  author       = {{Mistral AI}},
  title        = {Medium is the new large: Introducing Mistral Medium 3},
  howpublished = {\url{https://mistral.ai/news/mistral-medium-3/}},
  year         = {2025},
}

@misc{upstage2026solarpro4,
  author       = {{Upstage}},
  title        = {Solar Pro 4: The Agentic Model That Finishes the Job},
  url          = {https://www.upstage.ai/blog/en/solar-pro-4},
  year         = {2026},
}

@misc{claude35,
    title={Claude 3.5 Sonnet Model Card Addendum},
    author={{Anthropic}},
    year={2024},
    url={https://www-cdn.anthropic.com/fed9cc193a14b84131812372d8d5857f8f304c52/Model_Card_Claude_3_Addendum.pdf}
}

@misc{claude4,
    title={System Card: Claude Opus 4 \& Claude Sonnet 4},
    author={{Anthropic}},
    year={2025},
    url={https://www-cdn.anthropic.com/4263b940cabb546aa0e3283f35b686f4f3b2ff47.pdf}
}

@misc{opus45,
  author       = {{Anthropic}},
  title        = {Claude Opus 4.5},
  year         = {2025},
  url          = {https://www.anthropic.com/news/claude-opus-4-5},
}

@misc{opus46,
  author       = {{Anthropic}},
  title        = {Claude Opus 4.6},
  year         = {2026},
  url          = {https://www.anthropic.com/news/claude-opus-4-6},
}

@misc{openai2025openaio3card,
      title={OpenAI o3 and o4-mini System Card}, 
      author = {{OpenAI}},
      year={2025},
      url={https://cdn.openai.com/pdf/2221c875-02dc-4789-800b-e7758f3722c1/o3-and-o4-mini-system-card.pdf}, 
}
\bibliographystyle{iclr2027_conference}

\clearpage

\appendix

\startcontents[appendices]

\section*{Appendix Contents}
\printcontents[appendices]{}{1}{%
  \setcounter{tocdepth}{2}%
}

\clearpage

\section{Qualitative Examples of the \texorpdfstring{PoS \(n\)-gram}{PoS n-gram} Signal}
\label{app:ngram-examples}
\begingroup

\begin{table}[t]
\centering
\setlength{\tabcolsep}{4pt}
\caption{\textbf{Candidate signatures (OMI).} The two highest-ranked non-overlapping PoS $n$-grams of each candidate against the mean of the other eighteen, each with one matching span.}
\label{tab:qual-teachers}
\begin{tabular}{@{}p{0.32\linewidth}p{0.29\linewidth}p{0.35\linewidth}@{}}
\toprule
Candidate & Signature & Example \\
\midrule
GPT-OSS-120B & . RB NN : & ``. Thus answer :'' \\
\citep{openai2025gptoss120bgptoss20bmodel} & . NNP JJ & ``. Need third'' \\
\addlinespace
Gemma-3-27B-it & \$ . RB \$ & ``\$ . Then \$'' \\
\citep{gemmateam2025gemma3technicalreport} & \$ . PRP VBP & ``\$ . We want'' \\
\addlinespace
Llama-3.3-70B-Instruct & . \# \# NNP CD & ``. \# \# Step 2'' \\
\citep{grattafiori2024llama3herdmodels} & CD : VB DT NN & ``2 : Apply the condition'' \\
\addlinespace
Qwen-3-8B & . NNP , VB PRP & ``. Hmm , let me'' \\
\citep{qwen3} & VB DT . NNP , & ``approach this . First ,'' \\
\addlinespace
Claude-Sonnet-3.5 & . CD ) & ``. 2 )'' \\
\citep{claude35} & CD \$ JJ \$ & ``8 \$ * \$'' \\
\addlinespace
Claude-Opus-4.5 & \$ VB PRP VB & ``\$ Let me compute'' \\
\citep{opus45} & \$ . VB PRP & ``\$ . Let me'' \\
\addlinespace
Claude-Opus-4.6 & \# \# VBG & ``\# \# Setting'' \\
\citep{opus46} & \$ \$ \# & ``\$ \$ \#'' \\
\addlinespace
DeepSeek-R1 & NNP , NN . & ``x=0 , y=1 .'' \\
\citep{deepseekai2025r1} & . NNP . & ``. Hmmm .'' \\
\addlinespace
QwQ-32B-Preview & . VB DT NN , & ``. Wait a minute ,'' \\
\citep{qwq-32b-preview} & RB PRP VBP VBN DT & ``Maybe I have misidentified the'' \\
\addlinespace
o1 & NNP PRP VBP VBG & ``* I 'm exploring'' \\
\citep{openai2024openaio1card} & VBP JJ VBG IN & ``' m working with'' \\
\addlinespace
o3 & NNS VBP NN PRP VBP & ``conditions * * I '{}'' \\
\citep{openai2025openaio3card} & PRP VBP JJ VBG & ``I ' m considering'' \\
\addlinespace
Qwen3-235B-A22B-Thinking & JJ CC VBG NN & ``complex or challenging question'' \\
\citep{qwen3} & VB IN PRP . NNP & ``think about it . Well'' \\
\addlinespace
GPT-4o & VBZ : JJ NNP NNP & ``is : \textbackslash{} [ \textbackslash{}frac'' \\
\citep{openai2024gpt4ocard} & CD NNP NNP NNP , & ``32c \textbackslash{} ] So ,'' \\
\addlinespace
GLM-4.6 & NN : \textasciigrave{}\textasciigrave{} & ``Wait : \textasciigrave{}\textasciigrave{}'' \\
\citep{team2025glm45agenticreasoningcoding} & ) . RB JJ & ``) . Consequently *'' \\
\addlinespace
Mixtral-8x7B-Instruct & . CD . RB , & ``. 4 . Similarly ,'' \\
\citep{jiang2024mixtralexperts} & . RB VBZ DT NN & ``. Here 's the reasoning'' \\
\addlinespace
Llama-3-70B-Instruct & VBP PRP VBZ JJ . & ``hope it is correct .'' \\
\citep{grattafiori2024llama3herdmodels} & . PRP VBP PRP VBZ & ``. I hope it is'' \\
\addlinespace
GPT-3.5-turbo & PRP VBP IN \$ & ``we note that \$'' \\
\citep{brown2020languagemodelsfewshotlearners} & IN \$ CD \$ & ``Since \$ 1111111100 \$'' \\
\addlinespace
Qwen3-235B-A22B & \# \# \# NNP CD & ``\# \# \# Step 1'' \\
\citep{qwen3} & . : : \# \# & ``. -{}- - \# \#'' \\
\addlinespace
Qwen3-32B & \$ \$ : & ``\$ \$ -{}-'' \\
\citep{qwen3} & : \# \# \# NNP & ``- \# \# \# Key'' \\
\bottomrule
\end{tabular}
\end{table}

\begin{table}[t]
\centering
\setlength{\tabcolsep}{4pt}
\caption{\textbf{Candidate signatures (s1).} The two highest-ranked non-overlapping PoS $n$-grams of each candidate against the mean of the other eighteen, each with one matching span.}
\label{tab:qual-teachers-s1}
\begin{tabular}{@{}p{0.32\linewidth}p{0.29\linewidth}p{0.35\linewidth}@{}}
\toprule
Candidate & Signature & Example \\
\midrule
GPT-OSS-120B & . RB JJ NN & ``. So -a cos'' \\
\citep{openai2025gptoss120bgptoss20bmodel} & ) . NNP NN & ``\} . Define c'' \\
\addlinespace
Gemma-3-27B-it & \$ . VB \$ & ``\$ . Let \$'' \\
\citep{gemmateam2025gemma3technicalreport} & \$ . RB \$ & ``\$ . Then \$'' \\
\addlinespace
Llama-3.3-70B-Instruct & . \# \# NNP CD & ``. \# \# Step 2'' \\
\citep{grattafiori2024llama3herdmodels} & CD : VB DT NN & ``8 : Use the fact'' \\
\addlinespace
Qwen-3-8B & . NNP , VB PRP & ``. First , let me'' \\
\citep{qwen3} & . VB PRP VB . & ``? Let me think .'' \\
\addlinespace
Claude-Sonnet-3.5 & . CD ) & ``. 1 )'' \\
\citep{claude35} & ) \$ JJ \$ & ``\} \$ * \$'' \\
\addlinespace
Claude-Opus-4.5 & CD \$ : \$ CD & ``1 \$ : \$ 10\textasciicircum{}1'' \\
\citep{opus45} & \$ VB PRP & ``\$ Let me'' \\
\addlinespace
Claude-Opus-4.6 & . RB PRP VBP VBG & ``. Now I 'm looking'' \\
\citep{opus46} & . \# \# VBG & ``. \# \# Setting'' \\
\addlinespace
DeepSeek-R1 & NNP , RB PRP VBP & ``Alright , so I have'' \\
\citep{deepseekai2025r1} & , VB PRP VB TO & ``, let me try to'' \\
\addlinespace
QwQ-32B-Preview & RB PRP VBP VBN DT & ``So I 've got this'' \\
\citep{qwq-32b-preview} & VBP VBN DT NN RB & ``'ve got this problem here'' \\
\addlinespace
o1 & PRP VBP JJ VBG & ``I ' m considering'' \\
\citep{openai2024openaio1card} & NNP PRP VBP VBG & ``* I 'm thinking'' \\
\addlinespace
o3 & NNP NNP PRP VBP TO & ``* * I need to'' \\
\citep{openai2025openaio3card} & VBP NN PRP VBP & ``* * I need'' \\
\addlinespace
Qwen3-235B-A22B-Thinking & DT VBZ DT JJ CC & ``This is a complex or'' \\
\citep{qwen3} & CC VBG NN , CC & ``or challenging question , and'' \\
\addlinespace
GPT-4o & : JJ NNP NNP ( & ``: \textbackslash{} [ \textbackslash{}cos ('' \\
\citep{openai2024gpt4ocard} & , CD ) NN ) & ``, 0 ) \textbackslash{} )'' \\
\addlinespace
GLM-4.6 & . \# \# \# & ``. \# \# \#'' \\
\citep{team2025glm45agenticreasoningcoding} & . IN PRP VBP TO & ``. So we need to'' \\
\addlinespace
Mixtral-8x7B-Instruct & PRP VBP \$ \$ & ``we have \$ \$'' \\
\citep{jiang2024mixtralexperts} & VB ( NN NN ) & ``\textbackslash{}end \{ align * \}'' \\
\addlinespace
Llama-3-70B-Instruct & VBP PRP VBZ JJ . & ``hope it is correct .'' \\
\citep{grattafiori2024llama3herdmodels} & . PRP VBP PRP VBZ & ``. I hope it is'' \\
\addlinespace
GPT-3.5-turbo & \$ . NNP , DT & ``\$ . Thus , the'' \\
\citep{brown2020languagemodelsfewshotlearners} & JJ ( CD ) \$ & ``\textbackslash{}boxed \{ 10800 \} \$'' \\
\addlinespace
Qwen3-235B-A22B & : \# \# \# NNP & ``- \# \# \# Step'' \\
\citep{qwen3} & . : : \# \# & ``. -{}- - \# \#'' \\
\addlinespace
Qwen3-32B & VB PRP VB NN & ``Let me place point'' \\
\citep{qwen3} & CD . VB PRP & ``2x . Let me'' \\
\bottomrule
\end{tabular}
\end{table}

\definecolor{sigGptA}{HTML}{77AADD}
\definecolor{sigGptB}{HTML}{99DDFF}
\definecolor{sigGptC}{HTML}{DDDDDD}
\definecolor{sigLlaA}{HTML}{44BB99}
\definecolor{sigQwnA}{HTML}{FFAABB}
\definecolor{sigGemA}{HTML}{EEDD88}
\definecolor{sigGemC}{HTML}{BBCC33}
\definecolor{sigGemD}{HTML}{AAAA00}

\providecommand{\sighl}[2]{}
\renewcommand{\sighl}[2]{%
  {\setlength{\fboxsep}{0.5pt}%
   \colorbox{#1}{\vphantom{Ag}#2}}%
}

\providecommand{\sigcut}{}
\renewcommand{\sigcut}{\textdagger}

\providecommand{\sigkey}[4]{}
\renewcommand{\sigkey}[4]{%
  \par\noindent
  \sighl{#1}{#2}\enspace
  e.g.\ ``#3''%
  \if\relax\detokenize{#4}\relax
  \else
    \par Related patterns: #4%
  \fi
  \par\smallskip
}

\tcbset{
  fonttitle={},
  fontupper={},
  fontlower={}
}

\paragraph{Signature Construction}

A model's signature comprises PoS $n$-grams it uses more often than its peers. We process the first 2,000 characters using NLTK Penn Treebank tags~\citep{bird-loper-2004-nltk,marcus-etal-1993-building}, retain markup, and count $3$--$5$-grams. Distinctiveness is $\ln\big((p_{kj}+\epsilon)/(\bar p_{-k,j}+\epsilon)\big)$, where $\epsilon=10^{-5}$, $p_{kj}$ is $n$-gram $j$'s share in model $k$'s profile, and $\bar p_{-k,j}$ is its mean share among the other models. We retain patterns appearing in at least 20 of 200 responses and report the top ten per probe set. These unfiltered lists are descriptive and do not enter SCOUT's score.

\paragraph{Candidate Signatures}

Tables~\ref{tab:qual-teachers} and~\ref{tab:qual-teachers-s1} show two non-overlapping patterns from each list. The shared-prompt examples below use the four controlled candidates. Table~\ref{tab:qual-teachers} captures recurring forms: step headings (Llama-3.3-70B-Instruct, Qwen3-235B-A22B), first-person reasoning (Qwen-3-8B, DeepSeek-R1, QwQ-32B-Preview), omitted subjects and articles (GPT-OSS-120B), inline-math definitions (Gemma-3-27B-it), and descriptions of ongoing reasoning (o1, o3). Repeated phrases also matter: Qwen3-235B-A22B-Thinking repeats one opening sentence in 126/200 OMI and 145/200 s1 responses; Llama-3-70B-Instruct uses ``I hope it is correct'' in 140/200 and 117/200. The candidate responses contain no GPT-OSS channel labels or Qwen \texttt{<think>} tags within the analyzed prefix. All ten leading patterns for Llama-3.3 and Qwen occur in variants of ``\#\# Step \(k\):'' and ``let me,'' respectively, so each model uses a single highlight color below.

\paragraph{Same-Prompt Response Examples}

For each prompt \(q_i\), we compute the mean pairwise Jensen--Shannon divergence among the four candidates' normalized signature vectors, restricted to patterns appearing at least five times across their responses. We show high-contrast examples from OMI and s1 that require no specialized background. All four candidate-versus-rest contrasts are positive on \(161/200\) OMI prompts and \(140/200\) s1 prompts, so the examples illustrate a common direction rather than typical contrast magnitude.

Line breaks are adjusted for readability. Wording and markup are otherwise preserved, with quotation marks typeset in the manuscript font. Leading or trailing period and comma tokens are left unhighlighted. \sigcut{} marks the end of the first \(2{,}000\) characters of the original response.

\Needspace{26\baselineskip}
\begin{tcolorbox}[title={OMI Example (prompt 140)}]
Ava has 18 vases with nine flowers in each vase. If 1/3 of them have 3 wilted flowers each, how many flowers are still fresh?
\end{tcolorbox}

\begin{tcolorbox}[breakable, title={OMI / GPT-OSS-120B},
  title after break={OMI / GPT-OSS-120B (continued)}, lines before break=16,
  before upper={\raggedright\setlength{\parindent}{0pt}},
  before lower={\raggedright\setlength{\parindent}{0pt}\setlength{\parskip}{0pt}}]
\sigkey{sigGptA}{. NNP JJ}{. Need third}{NN . NNP NN, . NNP JJ NN, ) NN . NNP}
\sigkey{sigGptB}{RB NN :}{Thus answer :}{. RB NN, . RB NN :, . RB JJ NN, NN . RB JJ, . NN :}
\tcblower
We need to parse: Ava has 18 vases with nine flowers in each vase. \sighl{sigGptA}{So total} flowers = 18 * 9 = 162. If 1/3 of them have 3 wilted flowers each, ``them'' likely refers to vases? It says ``1/3 of them have 3 wilted flowers each''. Could be 1/3 of the vases have 3 wilted flowers each. So number of vases with wilted flowers = 1/3 * 18 = 6 vases. Each of those 6 vases has 3 wilted flowers. So total wilted flowers = 6 * 3 = 18 wilted flowers. All other flowers are fresh. Total flowers = 162, wilted = 18, fresh = 144. \sighl{sigGptB}{Thus answer:} 144 fresh flowers. \sighl{sigGptA}{Check alternative interpretation}: maybe 1/3 of the flowers have wilted - but phrase ``them have 3 wilted flowers each'' suggests vases. So answer 144. \sighl{sigGptA}{Thus final answer}: 144 fresh flowers.First, find the total number of flowers:

\textbackslash{}[

18\textbackslash{}text\{ vases\}\textbackslash{}times 9\textbackslash{}text\{ flowers per vase\}=162\textbackslash{}text\{ flowers\}

\textbackslash{}]

Next, determine how many vases contain wilted flowers. One-third of the vases are affected:

\textbackslash{}[

\textbackslash{}frac13 \textbackslash{}times 18 = 6\textbackslash{}text\{ vases\}

\textbackslash{}]

Each of those 6 vases has 3 wilted flowers, so the total number of wilted flowers is

\textbackslash{}[

6\textbackslash{}text\{ vases\}\textbackslash{}times 3\textbackslash{}text\{ wilted flowers per vase\}=18\textbackslash{}text\{ wilted flowers\}

\textbackslash{}]

Finally, subtract the wilted flowers from the total to get the fresh flowers:

\textbackslash{}[

162\textbackslash{}text\{ total\} - 18\textbackslash{}text\{ wilted\}=144\textbackslash{}text\{ fresh flowers\}

\textbackslash{}]

\textbackslash{}[

\textbackslash{}boxed\{144\}

\textbackslash{}]
\end{tcolorbox}

\begin{tcolorbox}[breakable, title={OMI / Llama-3.3-70B-Instruct},
  title after break={OMI / Llama-3.3-70B-Instruct (continued)}, lines before break=16,
  before upper={\raggedright\setlength{\parindent}{0pt}},
  before lower={\raggedright\setlength{\parindent}{0pt}\setlength{\parskip}{0pt}}]
\sigkey{sigLlaA}{\# \# NNP CD}{\# \# Step 1}{\# \# NNP CD :, \# NNP CD, \# NNP CD :, . \# \# NNP CD, . \# \#, . \# \# NNP, NNP CD : VB, \# NNP CD : VB, NNP CD : VB DT}
\tcblower
\sighl{sigLlaA}{\#\# Step 1: Calculate the} total number of flowers in all the vases.

Ava has 18 vases with 9 flowers in each vase. The total number of flowers is 18 * 9 = 162.

\sighl{sigLlaA}{\#\# Step 2: Determine the} number of vases with wilted flowers.

1/3 of the vases have 3 wilted flowers each. Since there are 18 vases, 1/3 of them is (1/3) * 18 = 6 vases.

\sighl{sigLlaA}{\#\# Step 3: Calculate the} total number of wilted flowers.

Each of the 6 vases has 3 wilted flowers. So, the total number of wilted flowers is 6 * 3 = 18.

\sighl{sigLlaA}{\#\# Step 4: Calculate the} number of fresh flowers.

The total number of flowers is 162, and the number of wilted flowers is 18. Therefore, the number of fresh flowers is 162 - 18 = 144.

The final answer is: \$\textbackslash{}boxed\{144\}\$
\end{tcolorbox}

\begin{tcolorbox}[breakable, title={OMI / Qwen-3-8B},
  title after break={OMI / Qwen-3-8B (continued)}, lines before break=16,
  before upper={\raggedright\setlength{\parindent}{0pt}},
  before lower={\raggedright\setlength{\parindent}{0pt}\setlength{\parskip}{0pt}}]
\sigkey{sigQwnA}{VB PRP VB}{Let me start}{. VB PRP VB, . VB PRP, , VB PRP VB, , VB PRP, NNP , VB PRP VB, NNP , VB PRP, VB PRP VB DT, . NNP , VB PRP, VB PRP VB IN}
\tcblower
\sighl{sigQwnA}{Okay, let me try} to figure out this problem. So Ava has 18 vases, and each vase has 9 flowers. The question is asking how many flowers are still fresh, given that 1/3 of them have 3 wilted flowers each. \sighl{sigQwnA}{Hmm, let me break this} down step by step. \sighl{sigQwnA}{First, let me find} out the total number of flowers Ava has. If there are 18 vases and each vase has 9 flowers, then the total number of flowers is 18 multiplied by 9. \sighl{sigQwnA}{Let me calculate that}. 18 times 9... Well, 10 times 9 is 90, and 8 times 9 is 72, so adding those together, 90 + 72 is 162. So there are 162 flowers in total. Now, the problem says that 1/3 of them have 3 wilted flowers each. Wait, does that mean 1/3 of the vases or 1/3 of the flowers? Hmm, the wording is a bit ambiguous. \sighl{sigQwnA}{Let me read} it again: ``1/3 of them have 3 wilted flowers each.'' The ``them'' here probably refers to the vases, since the previous sentence mentions vases. So it's likely that 1/3 of the vases have 3 wilted flowers each. \sighl{sigQwnA}{Let me check that} assumption. If it's 1/3 of the vases, then 1/3 of 18 vases is 6 vases. Each of those 6 vases has 3 wilted flowers. So, the number of wilted flowers would be 6 vases multiplied by 3 wilted flowers per vase, which is 18 wilted flowers. Therefore, the number of fresh flowers would be total flowers minus wilted flowers, which is 162 - 18 = 144. But \sighl{sigQwnA}{let me make} sure that's correct.

Alternatively, if ``1/3 of them'' refers to 1/3 of the flowers, then it would be different. \sighl{sigQwnA}{Let me consider that} possibility as well. If 1/3 of the flowers are wilted, then 1/3 of 162 is 54 wilted flowers. Then the fresh flowers would be 162 - 54 = 108. But which interpretation is correct? Looking back at the problem: ``Ava has 18 vases with nine flowers in each vase. If 1/3 of them have 3 wilted flowers each...'' The ``them'' is likely referring to the vases because the previous sentence is about vases. So ``them'' would be the vases. Therefore, 1/3 of the vases have 3 wilted flowers each. So that would be 6 vases with 3 wilted flowers each.\sigcut{} Therefore, total wilted flowers are 6*3=18. So total fresh flowers are 162 - 18 = 144. That seems to make sense. But \sighl{sigQwnA}{let me check} again. If it was 1/3 of the flowers, then the problem might have said ``1/3 of the flowers'' instead of ``1/3 of them.'' Since ``them'' is used, which refers back to the vases. So I think the first interpretation is correct. So the answer would be 144. Wait, but \sighl{sigQwnA}{let me think} again. Maybe the problem is that each of those vases has 3 wilted flowers, meaning that in each of those vases, 3 flowers are wilted, so the rest are fresh. So for each of those 6 vases, there are 9 - 3 = 6 fresh flowers. Therefore, total fresh flowers would be (number of vases without wilted flowers) * 9 + (number of vases with wilted flowers) * (9 - 3). \sighl{sigQwnA}{Let me calculate that}.

Number of vases without wilted flowers is 18 - 6 = 12 vases. Each of those has 9 fresh flowers, so 12 * 9 = 108. Then the 6 vases with wilted flowers have 6 fresh flowers each, so 6 * 6 = 36. Then total fresh flowers are 108 + 36 = 144. So that's the same as before. So that's consistent. Alternatively, if I had done 162 total flowers minus 18 wilted flowers, that's also 144. So that seems to check out. But let me just make sure that I didn't misinterpret the problem. The problem says ``1/3 of them have 3 wilted flowers each.'' So ``them'' is the vases. So 1/3 of the vases, which is 6 vases, each have 3 wilted flowers. So that's 6*3=18 wilted flowers. So total fresh is 162-18=144. So that seems correct. Alternatively, if someone thought that ``1/3 of them'' refers to 1/3 of the flowers, then it would be 54 wilted flowers. But that would mean that 54 flowers are wilted, which is 54/162 = 1/3. But the problem says ``have 3 wilted flowers each.'' If it was 1/3 of the flowers, then it would be 54 flowers, but how does that relate to 3 wilted flowers each? Maybe that's a different interpretation. \sighl{sigQwnA}{Let me think}. If ``1/3 of them have 3 wilted flowers each,'' then maybe ``them'' refers to the flowers? But that would be odd. Because if ``them'' refers to flowers, then 1/3 of the flowers (which is 54) have 3 wilted flowers each. But that would mean that each of those 54 flowers has 3 wilted flowers? That doesn't make sense. Because a flower can't have 3 wilted flowers. So that interpretation is invalid. Therefore, the correct interpretation must be that ``them'' refers to the vases. Therefore, 1/3 of the vases have 3 wilted flowers each. Therefore, the answer is 144.

But \sighl{sigQwnA}{let me check} again. If ``1/3 of them have 3 wilted flowers each,'' then ``them'' could be ambiguous. But given the context, since the previous sentence is about vases, it's more likely referring to the vases. Also, if it was referring to flowers, the sentence would probably say ``1/3 of the flowers have 3 wilted flowers each,'' which would be confusing. So the answer is 144. Alternatively, maybe the problem is that in each of the vases that have wilted flowers, 3 flowers are wilted. So, for each of those vases, 3 are wilted, so 9 - 3 = 6 are fresh. So total fresh flowers would be (18 vases - 6 vases with wilted flowers) * 9 + 6 vases * 6 flowers. Which is 12*9 + 6*6 = 108 + 36 = 144. Same answer. So that seems to confirm. Alternatively, if someone thought that 1/3 of the flowers are wilted, and each of those wilted flowers is in a vase with 3 wilted flowers... Wait, that seems more complicated. But the problem says ``1/3 of them have 3 wilted flowers each.'' So maybe ``have 3 wilted flowers each'' is referring to the vases. So each of those vases has 3 wilted flowers. So that's 3 per vase. So that seems to be the correct interpretation. Therefore, I think the answer is 144. But \sighl{sigQwnA}{let me check} once more. \sighl{sigQwnA}{Let me rephrase the} problem. Total vases: 18. Each has 9 flowers. So total flowers: 18*9=162. 1/3 of them (vases) have 3 wilted flowers each. So 1/3 of 18 is 6 vases. Each of these 6 vases has 3 wilted flowers. So total wilted flowers: 6*3=18. Therefore, fresh flowers: 162-18=144. Yes, that seems right.

Alternatively, if someone thought that 1/3 of the flowers are wilted, but that would be 54 flowers. But the problem says ``have 3 wilted flowers each.'' So if you have 1/3 of the flowers, which is 54, but how does that relate to 3 wilted flowers each? Maybe that's a different way of saying that each of those wilted flowers is in a group of 3? But that seems more complex and not how the problem is phrased. So I think the answer is 144. Therefore, I think the correct answer is 144.

**Final Answer**

The number of fresh flowers is \textbackslash{}boxed\{144\}.

\textless{}/think\textgreater{}

To determine how many flowers are still fresh, let's break the problem down into logical steps.

-{}-{}-

\#\#\# **Step 1: Total Number of Flowers**

Ava has 18 vases, and each vase contains 9 flowers. So, the total number of flowers is:

\$\$

18 \textbackslash{}times 9 = 162

\$\$

-{}-{}-

\#\#\# **Step 2: Understanding the Wilted Flowers**

The problem states
\end{tcolorbox}

\begin{tcolorbox}[breakable, title={OMI / Gemma-3-27B-it},
  title after break={OMI / Gemma-3-27B-it (continued)}, lines before break=16,
  before upper={\raggedright\setlength{\parindent}{0pt}},
  before lower={\raggedright\setlength{\parindent}{0pt}\setlength{\parskip}{0pt}}]
\sigkey{sigGemA}{. RB \$}{. Then \$}{\$ . RB \$}
\sigkey{sigGemC}{\$ VB DT}{\$ be the}{\$ VB DT NN, \$ VB DT NN IN}
\sigkey{sigGemD}{. VB \$}{. Let \$}{}
\tcblower
Let \$N\sighl{sigGemC}{\$ be the} total number of vases Ava has, which is 18. \sighl{sigGemD}{Let \$}F\sighl{sigGemC}{\$ be the number of} flowers in each vase, which is 9. The total number of flowers Ava has is \$N \textbackslash{}times F = 18 \textbackslash{}times 9 = 162\$. \sighl{sigGemD}{Let \$}W\sighl{sigGemC}{\$ be the fraction of} vases that have wilted flowers, which is \$\textbackslash{}frac\{1\}\{3\}\$. The number of vases with wilted flowers is \$W \textbackslash{}times N = \textbackslash{}frac\{1\}\{3\} \textbackslash{}times 18 = 6\$. Each of these vases has 3 wilted flowers. The total number of wilted flowers is \$6 \textbackslash{}times 3 = 18\$. The number of fresh flowers is the total number of flowers minus the number of wilted flowers.

Number of fresh flowers = Total number of flowers - Number of wilted flowers

Number of fresh flowers = \$162 - 18 = 144\$. Alternatively, we can calculate the number of fresh flowers in each vase with wilted flowers. Since each vase has 9 flowers and 3 are wilted, the number of fresh flowers in each vase with wilted flowers is \$9 - 3 = 6\$. The number of vases with wilted flowers is 6, so the total number of fresh flowers in these vases is \$6 \textbackslash{}times 6 = 36\$. The number of vases without wilted flowers is \$18 - 6 = 12\$. Each of these vases has 9 fresh flowers, so the total number of fresh flowers in these vases is \$12 \textbackslash{}times 9 = 108\$. The total number of fresh flowers is \$36 + 108 = 144\$.

Final Answer: The final answer is \$\textbackslash{}boxed\{144\}\$
\end{tcolorbox}

\Needspace{26\baselineskip}
\begin{tcolorbox}[title={s1 Example (prompt 15)}]
Joanie takes a $\$6,\!000$ loan to pay for her car.  The annual interest rate on the loan is $12\%$.  She makes no payments for 4 years, but has to pay back all the money she owes at the end of 4 years. How much more money will she owe if the interest compounds quarterly than if the interest compounds annually?  Express your answer as a dollar value to the nearest cent.
\end{tcolorbox}

\begin{tcolorbox}[breakable, title={s1 / GPT-OSS-120B},
  title after break={s1 / GPT-OSS-120B (continued)}, lines before break=16,
  before upper={\raggedright\setlength{\parindent}{0pt}},
  before lower={\raggedright\setlength{\parindent}{0pt}\setlength{\parskip}{0pt}}]
\sigkey{sigGptB}{. RB JJ NN}{. So -a cos}{) . RB JJ, . RB JJ, NN . RB JJ, . RB NN}
\sigkey{sigGptC}{. VB NNP VB}{. Let ' s}{. VB NNP}
\sigkey{sigGptA}{. NNP NN}{? Wait compute}{. NNP JJ, . NNP NN IN}
\tcblower
We have a loan principal P=6000. \sighl{sigGptA}{Annual nominal} interest rate r = 12\% = 0.12. We need to compute amount owed after 4 years with compounding annually vs quarterly. Annual compounding: amount = P * (1 + r)\textasciicircum{}\{4\}. Quarterly compounding: rate per quarter = r/4 = 0.12/4 = 0.03. Number of quarters = 4 years *4 = 16. Amount = P * (1 + 0.03)\textasciicircum{}\{16\}. Compute both and find difference. Compute annual: (1.12)\textasciicircum{}4. Let's compute.

1.12\textasciicircum{}2 = 1.2544.

1.12\textasciicircum{}4 = (1.12\textasciicircum{}2)\textasciicircum{}2 = 1.2544\textasciicircum{}2.

Compute 1.2544\textasciicircum{}2:

1.2544 * 1.2544.

1.2544 * 1 = 1.2544

0.2544 *0.2544? No let's multiply precisely. Or use approximate: 1.2544\textasciicircum{}2 = ? Compute using calculator mental:

1.2544 * 1.2544.

Take 12544 *12544 scaled by 10\textasciicircum{}8 maybe. 12544 *12544: 12544*12544:

12544*12000 = 150,528,000

12544*544 = ?

12544*500 = 6,272,000

12544*40 = 501,760

12544*4 = 50,176

Sum = 6,823,936

Add to 150,528,000 = 157,351,936. Thus 12544*12544 = 157,351,936. Now scaling: Since 12544 = 1.2544*10\textasciicircum{}4? Actually 1.2544 = 12544/10000. \sighl{sigGptB}{So product} = (12544\textasciicircum{}2)/(10000\textasciicircum{}2) = 157,351,936 / 100,000,000 = 1.57351936. So 1.12\textasciicircum{}4 = 1.57351936. \sighl{sigGptB}{Thus amount} annual = 6000 * 1.57351936 = 9,441.11616. \sighl{sigGptB}{Now quarterly}: (1.03)\textasciicircum{}16. Compute (1.03)\textasciicircum{}16. We can compute using powers.

1.03\textasciicircum{}2 = 1.0609.

1.03\textasciicircum{}4 = (1.0609)\textasciicircum{}2.

Compute 1.0609\textasciicircum{}2:

1.0609 * 1.0609.

10609*10609 = ?

Let's compute: 10609*10609.

10609*10000 = 106,090,000

10609*600 = 6,365,400

10609*9 = 95,481

Sum = 112,550,881. Now scaling: because 1.0609 = 10609/10000. Square = 112,550,881/100,000,000 = 1.12550881. So 1.03\textasciicircum{}4 = 1.12550881. Now 1.03\textasciicircum{}8 = (1.03\textasciicircum{}4)\textasciicircum{}2 = (1.12550881)\textasciicircum{}2. \sighl{sigGptA}{Compute square}: 1.12550881\textasciicircum{}2. \sighl{sigGptC}{Let's} compute: 1.12550881 * 1.12550881. Convert to integer: 112550881 * 112550881 scaled by 10\textasciicircum{}16? Since 1.12550881 = 112550881 / 100,000,000. \sighl{sigGptA}{Square numerator}: 112,550,881\textasciicircum{}2. That's big. \sighl{sigGptB}{Maybe use log} approximations? Or compute stepwise using slight approximations. \sighl{sigGptB}{Alternatively compute} 1.03\textasciicircum{}8 via incremental multiplication: 1.03\textasciicircum{}5 = 1.03\textasciicircum{}4 *1.03 = 1.12550881*1.03 = 1.1592740\sigcut{}743.

1.03\textasciicircum{}6 = *1.03 = 1.1940522965.

1.03\textasciicircum{}7 = *1.03 = 1.2298738654.

1.03\textasciicircum{}8 = *1.03 = 1.2667700814.

That seems easier: multiply sequentially. Let's continue: Actually we need up to 16. We have:

n   value

1: 1.03

2: 1.0609

3: 1.092727

4: 1.12550881

5: 1.1592740743

6: 1.1940522965

7: 1.2298738654

8: 1.2667700814

9: *1.03 = 1.3047731838

10:*1.03 = 1.3439163793

11:*1.03 = 1.3842338707

12:*1.03 = 1.4257608868

13:*1.03 = 1.4685337124

14:*1.03 = 1.5125897238

15:*1.03 = 1.5579674155

16:*1.03 = 1.6047064380? Let's compute precisely. Let's do with more accurate rounding. Start from 1.12550881 we had accurate to 8 decimal; progress. We should compute with better precision perhaps using calculator style but okay. Let's recalc using high precision maybe with Python mental? \sighl{sigGptA}{Could approximate}. We can compute (1.03)\textasciicircum{}16 directly using formula: (1+0.03)\textasciicircum{}16 \ensuremath{\approx} e\textasciicircum{}\{16 ln 1.03\}. ln 1.03 \ensuremath{\approx} 0.0295588. Multiply by 16 = 0.47294. e\textasciicircum{}\{0.47294\} \ensuremath{\approx} 1.6045. So around 1.604. Our sequential multiplication gave 1.6047, close. Let's compute more precisely using incremental with more decimals. I'll keep more digits:

1.03\textasciicircum{}1 = 1.03

\textasciicircum{}2 = 1.0609

\textasciicircum{}3 = 1.0609*1.03 = 1.092727

\textasciicircum{}4 = 1.092727*1.03 = 1.12550881 (matches)

\textasciicircum{}5 = 1.12550881*1.03 = 1.1592740743

\textasciicircum{}6 = *1.03 = 1.1940522965 (since 1.1592740743*0.03=0.0347782222 add =\textgreater{}1.1940522965)

Let's compute precisely: 1.1592740743*0.03 = 0.034778222229, add = 1.194052296529. We'll keep 1.1940522965.

\textasciicircum{}7 = *1.03 =\textgreater{} add 3\%: 1.1940522965*0.03=0.0358215689, sum=1.2298738654 (like before). Keep maybe 1.2298738654.

\textasciicircum{}8 = *1.03 =\textgreater{} 1.2298738654*0.03 = 0.03689621596, sum=1.26677008136.

\textasciicircum{}9 = *1.03 =\textgreater{} add 0.03800310244 =\textgreater{} 1.30477318380.

\textasciicircum{}10 = *1.03 =\textgreater{} add 0.0391431955 =\textgreater{} 1.34391637930.

\textasciicircum{}11 = *1.03 =\textgreater{} add 0.0403174914 =\textgreater{} 1.38423387070.

\textasciicircum{}12 = *1.03 =\textgreater{} add 0.0415270161 =\textgreater{} 1.42576088680.

\textasciicircum{}13 = *1.03 =\textgreater{} add 0.0427728266 =\textgreater{} 1.46853371340.

\textasciicircum{}14 = *1.03 =\textgreater{} add 0.044055, compute: 1.4685337134*0.03=0.044055? Actually 1.4685337134*0.03 = 0.044055, precisely 0.044055, wait compute: 1.4685337134 * 0.03 = 0.044055, but need more digits: 1.4685337134 *0.03 = 0.044055... multiply 1.4685*3 = 4.4055 then divide by 100 =\textgreater{}0.044055. \sighl{sigGptB}{So sum =} 1.5125887134? Wait add to 1.4685 = 1.5125887134.

But earlier we had 1.5125897238 slightly diff due to rounding. Let's compute more accurately: 1.4685337134 *0.03 = 0.044055, hmm 1.4685337134 *3 = 4.4056011402 divide 100 =0.044056011402. Add -\textgreater{} 1.512589724802. So value =1.5125897248.

\textasciicircum{}15 = *1.03 =\textgreater{} add 0.045377691744 =\textgreater{} sum = 1.5579674165 (actually 1.5125897248 *0.03 = 0.045377691744; plus =\textgreater{} 1.557967416544).

\textasciicircum{}16 = *1.03 =\textgreater{} add 0.0467390224963 =\textgreater{} sum = 1.604706438... compute: 1.557967416544*0.03 = 0.04673902249632; add =\textgreater{} 1.60470643904032.

Thus (1.03)\textasciicircum{}16 \ensuremath{\approx} 1.60470644.

Now amount quarterly = 6000 * 1.60470644 = 9
\end{tcolorbox}

\begin{tcolorbox}[breakable, title={s1 / Llama-3.3-70B-Instruct},
  title after break={s1 / Llama-3.3-70B-Instruct (continued)}, lines before break=16,
  before upper={\raggedright\setlength{\parindent}{0pt}},
  before lower={\raggedright\setlength{\parindent}{0pt}\setlength{\parskip}{0pt}}]
\sigkey{sigLlaA}{\# \# NNP CD}{\# \# Step 1}{\# NNP CD, \# NNP CD :, \# \# NNP CD :, \# \# NNP, . \# \# NNP CD, . \# \# NNP, . \# \#, \# NNP CD : VB, NNP CD : VB}
\tcblower
\sighl{sigLlaA}{\#\# Step 1: Calculate} the amount owed if interest compounds annually

The formula for compound interest is \$A = P(1 + r)\textasciicircum{}n\$, where \$A\$ is the amount of money accumulated after \$n\$ years, including interest, \$P\$ is the principal amount (initial amount of money), \$r\$ is the annual interest rate (in decimal), and \$n\$ is the number of years the money is invested or borrowed for. For annual compounding, \$P = \textbackslash{}\$6,000\$, \$r = 0.12\$, and \$n = 4\$. So, \$A\_\{annual\} = 6000(1 + 0.12)\textasciicircum{}4\$.

\sighl{sigLlaA}{\#\# Step 2: Calculate} the amount owed if interest compounds quarterly

For quarterly compounding, the formula adjusts to \$A = P(1 + \textbackslash{}frac\{r\}\{m\})\textasciicircum{}\{mt\}\$, where \$m\$ is the number of times that interest is compounded per year and \$t\$ is the time the money is invested or borrowed for, in years. Here, \$P = \textbackslash{}\$6,000\$, \$r = 0.12\$, \$m = 4\$ (since the interest is compounded quarterly), and \$t = 4\$. So, \$A\_\{quarterly\} = 6000(1 + \textbackslash{}frac\{0.12\}\{4\})\textasciicircum{}\{4*4\}\$.

\sighl{sigLlaA}{\#\# Step 3:} Compute \$A\_\{annual\}\$

\$A\_\{annual\} = 6000(1 + 0.12)\textasciicircum{}4 = 6000(1.12)\textasciicircum{}4 = 6000 * 1.573520 = 9441.12\$.

\sighl{sigLlaA}{\#\# Step 4:} Compute \$A\_\{quarterly\}\$

\$A\_\{quarterly\} = 6000(1 + \textbackslash{}frac\{0.12\}\{4\})\textasciicircum{}\{4*4\} = 6000(1 + 0.03)\textasciicircum{}\{16\} = 6000(1.03)\textasciicircum{}\{16\} = 6000 * 1.604706 = 9628.24\$.

\sighl{sigLlaA}{\#\# Step 5: Find} the difference in the amounts owed

The difference in the amounts owed is \$A\_\{quarterly\} - A\_\{annual\} = 9628.24 - 9441.12\$.

\sighl{sigLlaA}{\#\# Step 6: Calculate} the final difference

\$9628.24 - 9441.12 = 187.12\$.

The final answer is: \$\textbackslash{}boxed\{187.12\}\$
\end{tcolorbox}

\begin{tcolorbox}[breakable, title={s1 / Qwen-3-8B},
  title after break={s1 / Qwen-3-8B (continued)}, lines before break=16,
  before upper={\raggedright\setlength{\parindent}{0pt}},
  before lower={\raggedright\setlength{\parindent}{0pt}\setlength{\parskip}{0pt}}]
\sigkey{sigQwnA}{. VB PRP VB}{. Let me try}{. VB PRP, VB PRP VB, , VB PRP VB, VB PRP VB IN, NNP , VB PRP, NNP , VB PRP VB, PRP VB IN, NN . VB PRP, . NNP , VB PRP}
\tcblower
Okay, so Joanie took out a \$6,000 loan with a 12\% annual interest rate. She doesn't make any payments for 4 years and has to pay it all back at the end. The question is asking how much more she'll owe if the interest compounds quarterly compared to if it compounds annually. I need to find the difference between the two amounts. \sighl{sigQwnA}{Alright, let me recall} the formula for compound interest.

The formula for compound interest is A = P(1 + r/n)\textasciicircum{}(nt), where:

- A is the amount of money accumulated after n years, including interest.

- P is the principal amount (the initial amount of money).

- r is the annual interest rate (decimal).

- n is the number of times that interest is compounded per year.

- t is the time the money is invested or borrowed for, in years.

So, first, I need to calculate the amount she owes if the interest is compounded annually and then if it's compounded quarterly. Then subtract the two to find the \sighl{sigQwnA}{difference.}

\sighl{sigQwnA}{Let me start with} annual compounding. For annual compounding, n = 1. So plugging into the formula:

A\_annual = 6000*(1 + 0.12/1)\textasciicircum{}(1*4)

\sighl{sigQwnA}{Let me compute} that. First, 0.12 divided by 1 is 0.12. Then 1 + 0.12 is 1.12. Then raise that to the power of 4. \sighl{sigQwnA}{Let me calculate} 1.12\textasciicircum{}4. Hmm, 1.12\textasciicircum{}2 is 1.2544. Then squaring that result for the fourth power. So 1.2544\textasciicircum{}2. \sighl{sigQwnA}{Let me compute} that. \sighl{sigQwnA}{Let me do} 1.2544 * 1.2544. First, multiply 1.25 * 1.25 = 1.5625. But since it's 1.2544, maybe I should do more precise calculation. Alternatively, maybe I can use logarithms or exponent rules, but maybe it's easier to just multiply step by \sighl{sigQwnA}{step.} \sighl{sigQwnA}{Let me write} it out: 1.2544 * 1.2544: Break it down as (1 + 0.25 + 0.0044)\textasciicircum{}2? Wait, maybe that's not helpful. Alternatively, use the standard multiplication method. Multiply 1.2544 by 1.2544: First, multiply 1.2544 by 1.2544: \sighl{sigQwnA}{Let me think of} it as:

1.2544

x1.2544

-{}-{}-{}-{}-{}-{}-{}-{}-{}-{}-{}-

Well, this might take some time. Alternatively, I can remember that 1.12\textasciicircum{}4 is a standard \sighl{sigQwnA}{value. Let me check if} I remember it or if I can \sighl{sigQwnA}{calculate i}\sigcut{}\sighl{sigQwnA}{t step by} step. Alternatively, use the formula for annual compounding:

After 1 year: 6000 * 1.12 = 6720

After 2 years: 6720 * 1.12. \sighl{sigQwnA}{Let me compute} that. 6720 * 1.12. 6720 * 1 = 6720, 6720 * 0.12 = 806.4. So total is 6720 + 806.4 = 7526.4

After 3 years: 7526.4 * 1.12. \sighl{sigQwnA}{Let me compute} that. 7526.4 * 1.12. \sighl{sigQwnA}{Let me break} it down:

First, 7526.4 * 1 = 7526.4

7526.4 * 0.12 = ?

Well, 7526.4 * 0.1 = 752.64

7526.4 * 0.02 = 150.528

So total is 752.64 + 150.528 = 903.168

Therefore, total after 3 years is 7526.4 + 903.168 = 8429.568

Then after 4 years: 8429.568 * 1.12. \sighl{sigQwnA}{Let me compute} that. Again, 8429.568 * 1.12. Break into 8429.568 * 1 + 8429.568 * 0.12.

First part is 8429.568

Second part: 8429.568 * 0.12. \sighl{sigQwnA}{Let me calculate}:

8429.568 * 0.1 = 842.9568

8429.568 * 0.02 = 168.59136

Adding those gives 842.9568 + 168.59136 = 1011.54816

Therefore, total amount after 4 years is 8429.568 + 1011.54816 = 9441.11616

So approximately \$9,441.12 when compounded annually. Wait, but \sighl{sigQwnA}{let me check if} that's correct. Alternatively, maybe I can use the formula A = P*(1 + r)\textasciicircum{}t. Since it's annual compounding, n=1, so yeah, that's the same as A = 6000*(1.12)\textasciicircum{}4. Alternatively, maybe I can use a calculator for 1.12\textasciicircum{}4. Wait, but since I don't have a calculator here, maybe I can compute it more accurately. Wait, but I think my step-by-step calculation is correct. \sighl{sigQwnA}{Let me check} again: First year: 6000 * 1.12 = 6720. Correct. Second year: 6720 * 1.12. \sighl{sigQwnA}{Let me compute} 6720 * 1.12: Well, 6720 * 1.12 = 6720 + 6720 * 0.12 = 6720 + 806.4 = 7526.4. Correct. Third year: 7526.4 * 1.12. \sighl{sigQwnA}{Let me compute} 7526.4 * 1.12:

First, 7526.4 * 1.1 = 7526.4 + 752.64 = 8279.04

Then 7526.4 * 0.02 = 150.528

So total is 8279.04 + 150.528 = 8429.568. Correct. Fourth year: 8429.568 * 1.12. \sighl{sigQwnA}{Let me compute}:

Again, 8429.568 * 1.1 = 8429.568 + 842.9568 = 9272.5248

Then 8429.568 * 0.02 = 168.59136

So total is 9272.5248 + 168.59136 = 9441.11616. So yes, \$9,441.12 when rounded to the nearest cent. Okay, so that's the amount with annual compounding.

Now, moving on to quarterly compounding. Here, n = 4, since it's compounded quarterly. So the formula becomes:

A\_quarterly = 6000*(1 + 0.12/4)\textasciicircum{}(4*4) = 6000*(1 + 0.03)\textasciicircum{}16

Wait, 0.12 divided by 4 is 0.03. So 1 + 0.03 = 1.03. Then raised to the power of 16, since 4 years * 4 quarters per year = 16 periods. So I need to compute 1.03\textasciicircum{}16. Hmm, that might be a bit more involved. \sighl{sigQwnA}{Let me see if} I can calculate that step by step or find a way to approximate it. Alternatively, maybe I can use logarithms or remember some exponent values, but since I need to be precise, \sighl{sigQwnA}{let me try} to \sighl{sigQwnA}{compute it step by} step. Alternatively, maybe I can use the formula for compound interest step by \sighl{sigQwnA}{step. Let me try} that. Alternatively, I can use the rule that (1 + r/n)\textasciicircum{}(nt). So 1.03\textasciicircum{}16. \sighl{sigQwnA}{Let me compute} that. \sighl{sigQwnA}{Let me note that} 1.03\textasciicircum{}16 can be calculated as (1.03\textasciicircum{}4)\textasciicircum{}4. Maybe that's easier? \sighl{sigQwnA}{Let me compute} 1.03\textasciicircum{}4 first. Compute 1.03\textasciicircum{}2 = 1.0609. Then 1.0609\textasciicircum{}2 = ? \sighl{sigQwnA}{Let me calculate} 1.
\end{tcolorbox}

\begin{tcolorbox}[breakable, title={s1 / Gemma-3-27B-it},
  title after break={s1 / Gemma-3-27B-it (continued)}, lines before break=16,
  before upper={\raggedright\setlength{\parindent}{0pt}},
  before lower={\raggedright\setlength{\parindent}{0pt}\setlength{\parskip}{0pt}}]
\sigkey{sigGemA}{. RB \$}{. Then \$}{\$ . RB \$}
\sigkey{sigGemD}{. VB \$}{. Let \$}{\$ . VB \$, VB \$ VB, . VB \$ VB, \$ VB DT NN, \$ . VB \$ VB, \$ VB DT NN IN}
\tcblower
\sighl{sigGemD}{Let \$P\$ be} the principal amount of the loan, which is \$P = \textbackslash{}\$6000\sighl{sigGemD}{\$.} \sighl{sigGemD}{Let \$r\$ be} the annual interest rate, which is \$r = 12\textbackslash{}\% = 0.12\sighl{sigGemD}{\$.} \sighl{sigGemD}{Let \$t\$ be the number of} years, which is \$t = 4\$.

If the interest compounds annually, the amount owed at the end of 4 years is

\$\$ A\_1 = P(1+r)\textasciicircum{}t = 6000(1+0.12)\textasciicircum{}4 = 6000(1.12)\textasciicircum{}4 = 6000(1.57351936) = 9441.11616 \textbackslash{}approx 9441.12. \$\$

If the interest compounds quarterly, the interest rate per quarter is \$r\_q = \textbackslash{}frac\{r\}\{4\} = \textbackslash{}frac\{0.12\}\{4\} = 0.03\$. The number of quarters in 4 years is \$n = 4t = 4(4) = 16\$.

The amount owed at the end of 4 years is

\$\$ A\_2 = P\textbackslash{}left(1+\textbackslash{}frac\{r\}\{4\}\textbackslash{}right)\textasciicircum{}\{4t\} = 6000(1+0.03)\textasciicircum{}\{16\} = 6000(1.03)\textasciicircum{}\{16\} = 6000(1.6047064378) \textbackslash{}approx 9628.2386268 \textbackslash{}approx 9628.24. \$\$

The difference between the amount owed with quarterly compounding and the amount owed with annual compounding is

\$\$ A\_2 - A\_1 = 9628.24 - 9441.12 = 187.12. \$\$

The difference in the amount owed is

\$\$ A\_2 - A\_1 = 6000(1.03)\textasciicircum{}\{16\} - 6000(1.12)\textasciicircum{}4 = 6000((1.03)\textasciicircum{}\{16\} - (1.12)\textasciicircum{}4) \$\$

\$\$ = 6000(1.6047064378 - 1.57351936) = 6000(0.0311870778) \textbackslash{}approx 187.12246668 \textbackslash{}approx 187.12. \$\$

Final Answer: The final answer is \$\textbackslash{}boxed\{187.12\}\$
\end{tcolorbox}

\paragraph{Interpretation}

The lists depend on the candidate pool and signature size, and tagging errors can complicate their interpretation: sentence-initial ``Compute $b$'' receives \texttt{. NNP NN}. They describe recurring forms rather than establish attribution on their own.

\endgroup

\section{Representation Choice and Robustness}
\label{app:ngram}
\begin{table}[tp]
\centering
\caption{\textbf{Removing markup re-tags a signature rather than erasing it.} For each candidate in the controlled pool, we rank patterns against the mean of the other three using raw text and text with Markdown and math markup removed. The last column counts patterns common to the two top-ten lists.}
\label{tab:qual-markup}
\setlength{\tabcolsep}{4pt}
\begin{tabular}{@{}llllr@{}}
\toprule
Candidate & Markup & Signature & Example & Shared \\
\midrule
\multicolumn{5}{@{}l}{OMI probe set} \\
GPT-OSS-120B & kept & . NNP JJ & ``. Need third'' & 6/10 \\
 &  & RB NN : & ``Thus answer :'' &  \\
 & removed & RB NN : & ``Thus answer :'' &  \\
 &  & . NNP JJ & ``. Need third'' &  \\
\addlinespace
Llama-3.3-70B-Instruct & kept & \# \# NNP CD & ``\# \# Step 1'' & 0/10 \\
 &  & NNP CD : VB & ``Step 1 : Define'' &  \\
 & removed & NN CD : VB & ``Step 1 : Define'' &  \\
\addlinespace
Qwen-3-8B & kept & VB PRP VB & ``Let me start'' & 9/10 \\
 & removed & , VB PRP VB & ``, let me plug'' &  \\
\addlinespace
Gemma-3-27B-it & kept & . RB \$ & ``. Then \$'' & 0/10 \\
 &  & \$ \$ \$ & ``\$ \$ \$'' &  \\
 & removed & : DT JJ NN VBZ & ``: The final answer is'' &  \\
 &  & . JJ NN : DT & ``. Final Answer : The'' &  \\
\midrule
\multicolumn{5}{@{}l}{s1 probe set} \\
GPT-OSS-120B & kept & . RB JJ NN & ``. So -a cos'' & 5/10 \\
 &  & . VB NNP VB & ``. Let ' s'' &  \\
 & removed & . NNP NN & ``? Wait compute'' &  \\
 &  & ) . RB JJ & ``) . So -a'' &  \\
\addlinespace
Llama-3.3-70B-Instruct & kept & \# \# NNP CD & ``\# \# Step 1'' & 0/10 \\
 &  & NNP CD : VB & ``Step 1 : Understand'' &  \\
 & removed & CD : VB DT & ``1 : Understand the'' &  \\
 &  & . NN CD : & ``. Step 2 :'' &  \\
\addlinespace
Qwen-3-8B & kept & . VB PRP VB & ``. Let me try'' & 9/10 \\
 &  & NNP , VB PRP & ``First , let me'' &  \\
 & removed & . VB PRP VB & ``. Let me try'' &  \\
 &  & NNP , VB PRP & ``First , let me'' &  \\
\addlinespace
Gemma-3-27B-it & kept & . RB \$ & ``. Then \$'' & 0/10 \\
 &  & . VB \$ & ``. Let \$'' &  \\
 & removed & VB NNP VB DT & ``Let S be a'' &  \\
 &  & CD ( CD CD & ``2 ( 1200 10\textasciicircum{}'' &  \\
\bottomrule
\end{tabular}
\end{table}

\subsection{Markup and Window Controls}

\paragraph{Markup Effects} Table~\ref{tab:qual-markup} shows that removing Markdown and math markup changes tag sequences without necessarily removing the underlying writing habits. Five or six of GPT-OSS-120B's top ten patterns remain unchanged, as do nine of Qwen-3-8B's. Although no exact top-ten pattern overlaps for Llama-3.3-70B-Instruct or Gemma-3-27B-it, the stripped responses retain step headings and definition-related phrases such as ``Step 1: Define,'' ``The final answer is,'' and ``Let \(S\) be a.'' The factorial analysis below quantifies how these changes affect \(\Delta G_T\).

\paragraph{Text Length and Preprocessing} The factorial analysis varies window length from \(500\) to \(2{,}000\) characters, position at the start, a random location, or the end, and preprocessing with raw or canonicalized text. Candidate-pool centering shifts \(G_T\) by the same constant for \(S\) and \(S^{(0)}\), which cancels in \(\Delta G_T\). We therefore report \(D\) and \(C\) together. Figure~\ref{fig:factorial-dgt} summarizes the trends, while Table~\ref{tab:factorial-cells} reports all \(24\) conditions under the default \(\lambda=0.15\).

\begin{figure*}[t]
\centering
\begin{subfigure}[t]{0.475\textwidth}
  \centering
  \includegraphics[width=\linewidth]{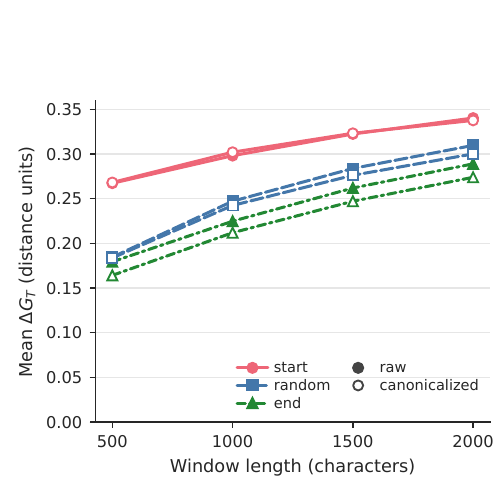}
  \caption{Profile \(D\) and Centered \(C\)}
\end{subfigure}\hfill
\begin{subfigure}[t]{0.475\textwidth}
  \centering
  \includegraphics[width=\linewidth]{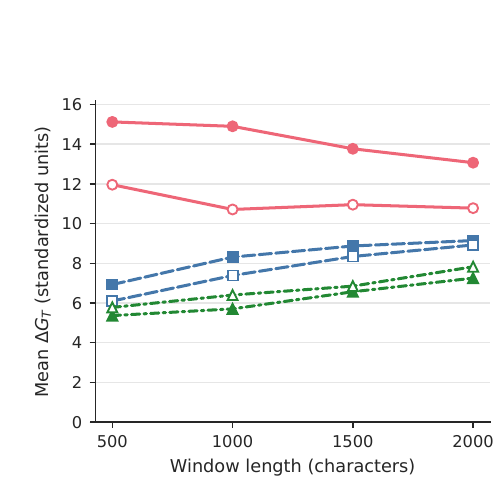}
  \caption{SCOUT}
\end{subfigure}
\caption{\textbf{Teacher-directed margin changes across text windows and preprocessing.} Each point reports mean \(\Delta G_T\) over nineteen controlled students under \(\lambda=0.15\). Filled and open markers denote raw and canonicalized text, respectively. Random-window curves average ten seeds. All nineteen students have positive \(\Delta G_T\) in every condition. Panel \textbf{(a)} combines \(D\) and \(C\) because centering cancels in \(\Delta G_T\), while panel \textbf{(b)} applies candidate-specific scaling. The legend in \textbf{(a)} applies to both panels; their vertical scales are not comparable.}
\label{fig:factorial-dgt}
\end{figure*}

\begin{table*}[t]
  \centering
  \caption{\textbf{Trajectory-margin changes across text windows and preprocessing.} Each entry reports mean \(\Delta G_T\) over nineteen controlled students. All nineteen have positive \(\Delta G_T\) in every condition under \(D\), \(C\), and SCOUT. \(D\) and \(C\) are identical for \(\Delta G_T\), whereas SCOUT additionally applies candidate-specific scaling, so magnitudes should be compared only within each readout. Random-window means are averaged across ten seeds, and the \(19/19\) result holds under the minimum count across seeds.}
  \label{tab:factorial-cells}
  \small
  \setlength{\tabcolsep}{4pt}
  \begin{tabular}{@{}llrrrrrrrr@{}}
    \toprule
    & & \multicolumn{4}{c}{Profile \(D\) and Centered \(C\)}
      & \multicolumn{4}{c}{SCOUT} \\
    \cmidrule(lr){3-6}\cmidrule(lr){7-10}
    Preprocessing & Window
      & 500 & 1000 & 1500 & 2000
      & 500 & 1000 & 1500 & 2000 \\
    \midrule
    Raw & start
      & 0.2674 & 0.2982 & 0.3225 & 0.3404
      & 15.1212 & 14.8995 & 13.7677 & 13.0675 \\
    & random
      & 0.1846 & 0.2472 & 0.2839 & 0.3097
      & 6.9287 & 8.3123 & 8.8714 & 9.1461 \\
    & end
      & 0.1795 & 0.2249 & 0.2620 & 0.2891
      & 5.3654 & 5.7042 & 6.5654 & 7.2521 \\
    \midrule
    Canonicalized & start
      & 0.2683 & 0.3021 & 0.3233 & 0.3377
      & 11.9578 & 10.7116 & 10.9529 & 10.7779 \\
    & random
      & 0.1838 & 0.2425 & 0.2762 & 0.3002
      & 6.0996 & 7.3841 & 8.3427 & 8.9223 \\
    & end
      & 0.1643 & 0.2120 & 0.2472 & 0.2740
      & 5.7769 & 6.4008 & 6.8487 & 7.8247 \\
    \bottomrule
  \end{tabular}
\end{table*}

For \(D\) and \(C\), mean \(\Delta G_T\) increases with window length in all six preprocessing--position combinations. Under SCOUT, it increases for random and end windows, whereas both start-window series peak at \(500\) characters. Nevertheless, every student has positive \(\Delta G_T\) in every condition. Relative to the first \(2{,}000\) raw characters, the first \(500\) canonicalized characters retain \(78.8\%\) of the gain under \(D\) and \(C\), and \(91.5\%\) under SCOUT. Thus, teacher-associated evidence occurs throughout the response but is especially concentrated near its beginning.

\subsection{Inheritance of Teacher Signatures}

\paragraph{Signature Inheritance} Figure~\ref{fig:qual-heatmap} presents heatmaps of teacher-signature inheritance. Panels \textbf{(a)} and \textbf{(b)} repeat the 1.5--4B math and conversation results from Figure~\ref{fig:controlled-signatures}, while panel \textbf{(c)} adds the 7--8B math cohort. Each cell measures how strongly held-out student responses reproduce a candidate's top-ten signature after candidate-contrast filtering with \(\lambda=0.15\). The block-diagonal structure shows preferential reproduction of teacher-associated patterns.

\begin{figure*}[p]
  \centering
  \begin{subfigure}[t]{0.96\textwidth}
    \centering
    \includegraphics[width=\linewidth]{./figures/pos_sig_heatmap_cohort_lambda015_a_budgetstyle}
    \caption{Math, 1.5--4B cohort}
  \end{subfigure}

  \vspace{0.4em}

  \begin{subfigure}[t]{0.96\textwidth}
    \centering
    \includegraphics[width=\linewidth]{./figures/pos_sig_heatmap_cohort_lambda015_b_budgetstyle}
    \caption{Conversation, 1.5--4B cohort}
  \end{subfigure}

  \vspace{0.4em}

  \begin{subfigure}[t]{0.96\textwidth}
    \centering
    \includegraphics[width=\linewidth]{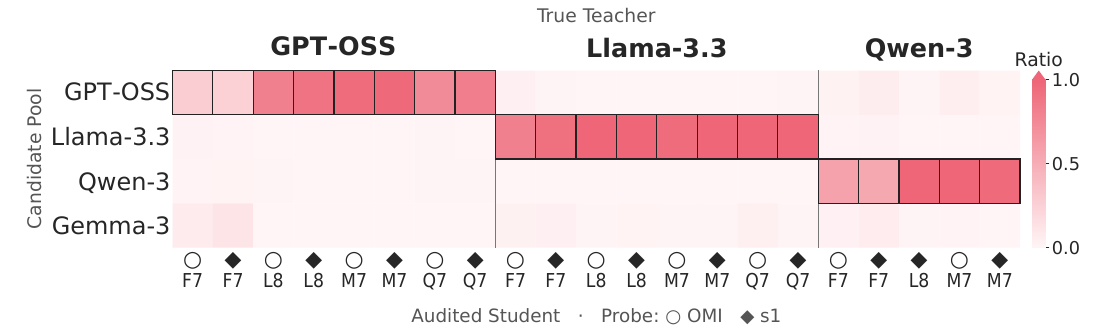}
    \caption{Math, 7--8B cohort}
  \end{subfigure}

  \caption{\textbf{Controlled students reproduce teacher-associated syntactic patterns.} Each cell shows the share of a candidate's top-ten PoS \(n\)-grams observed in held-out student responses, normalized by the corresponding share in the candidate's own responses. Signatures are selected after candidate-contrast filtering with \(\lambda=0.15\). Rows are candidates, columns are students grouped by true teacher, and black outlines mark each student's true teacher. Student labels G4, L3, Q1.5, Q3, F7, L8, M7, and Q7 denote Gemma-3-4B-PT, Llama-3.2-3B-Instruct, Qwen-2.5-1.5B, Qwen-2.5-3B, Falcon3-7B-Base, Llama-3.1-8B, Mistral-7B-v0.3, and Qwen-2.5-7B, respectively. Circles and diamonds distinguish the two probes used in each cohort.}
  \label{fig:qual-heatmap}
\end{figure*}

\paragraph{Training-Scaffold Removal} Table~\ref{tab:qual-students} lists representative student-side patterns after removing training-format markers, specifically GPT-OSS channel labels and Qwen \texttt{<think>} tags. Teacher-associated forms recur across both probes and model scales, indicating that the inheritance pattern is not solely driven by copied scaffolding.

\begin{table}[t]
\centering
\caption{\textbf{Student signatures.} Controlled students grouped by source, against the students of the other two sources on the same base models, with format scaffolding removed.}
\label{tab:qual-students}
\setlength{\tabcolsep}{4pt}
\begin{tabular}{@{}lllll@{}}
\toprule
Cohort & Source & Probe & Signature & Example \\
\midrule
1.5--4B & GPT-OSS-120B & OMI & . NNP NN & ``. Compute b'' \\
 &  &  & ) . RB JJ & ``) . Similarly median'' \\
 &  & s1 & . RB JJ & ``. Similarly \textbackslash{}'' \\
 &  &  & . NNP NN & ``. Compute product'' \\
\addlinespace
 & Llama-3.3-70B-Instruct & OMI & NNP CD : VB DT & ``Step 2 : Recall the'' \\
 &  &  & . \# \# NNP CD & ``. \# \# Step 2'' \\
 &  & s1 & . \# \# NNP CD & ``. \# \# Step 2'' \\
 &  &  & \# NNP CD : VB & ``\# Step 1 : Understand'' \\
\addlinespace
 & Qwen-3-8B & OMI & . VB PRP VB & ``. Let me try'' \\
 &  &  & NNP , VB PRP & ``First , let me'' \\
 &  & s1 & . VB PRP VB & ``? Let me write'' \\
 &  &  & NNP , VB PRP & ``Hmm , let me'' \\
\midrule
7--8B & GPT-OSS-120B & OMI & ) . RB JJ & ``) . Thus w\textasciicircum{}-6'' \\
 &  &  & . NNP NN & ``. Done computation'' \\
 &  & s1 & ) . RB JJ & ``\} . Thus upper'' \\
 &  &  & . NNP NN & ``? Simplify orig'' \\
\addlinespace
 & Llama-3.3-70B-Instruct & OMI & . \# \# NNP & ``. \# \# Step'' \\
 &  &  & \# NNP CD : VB & ``\# Step 3 : Apply'' \\
 &  & s1 & . \# \# NNP CD & ``. \# \# Step 2'' \\
 &  &  & \# NNP CD : VB & ``\# Step 1 : Consider'' \\
\addlinespace
 & Qwen-3-8B & OMI & . VB PRP VB & ``. Let me recall'' \\
 &  &  & NNP , VB PRP & ``First , let me'' \\
 &  & s1 & . VB PRP VB & ``. Let me come'' \\
 &  &  & NNP , VB PRP & ``Okay , let me'' \\
\bottomrule
\end{tabular}
\end{table}

\paragraph{Signature-Size Sensitivity} Figure~\ref{fig:signature-sensitivity} tests whether the true teacher remains top-ranked as the signature size and candidate pool vary. The ranking is stable around the default \(k=10\), but broader lists eventually include patterns shared across models.

\begin{figure}[t]
  \centering
  \includegraphics[width=0.5\linewidth]{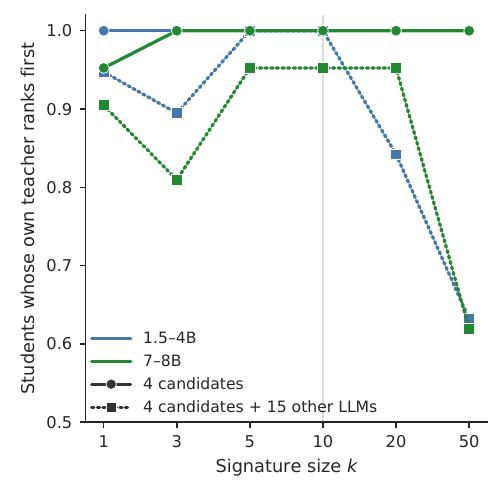}
  \caption{\textbf{Sensitivity to signature size and candidate-pool composition.} The fraction of controlled students whose true source ranks first as the number of signature PoS $n$-grams varies. Solid lines with circles use the four-candidate pool, whereas dotted lines with squares use the nineteen-candidate pool. The vertical line marks the default choice, $k=10$.}
  \label{fig:signature-sensitivity}
\end{figure}

\subsection{Candidate-Contrast Sensitivity}
\label{app:filter-sensitivity}

The filtering fraction \(\lambda\) controls how many low-contrast PoS \(n\)-grams are removed before profile distances are computed. This representation-stage filter is separate from the candidate-pool calibration. We examine sensitivity to \(\lambda\) across controlled cohorts and public descendants after subsequent training. The analysis characterizes a stable region rather than selecting an optimal value.

\paragraph{Controlled Cohorts}

\begin{table*}[t]
  \centering
  \caption{\textbf{Sensitivity to \(\lambda\) in the controlled cohorts.} \(\mathrm{Det.}\) counts students correctly attributed on both probes with the teacher present. Each ordered pair \((\mathrm{FP}(S),\mathrm{FP}(T))\) gives the numbers of false acceptances after teacher removal under student-backbone and teacher holdout, respectively. Bold headers mark the only values yielding no false acceptances under either protocol in all three cohorts.}
  \label{tab:lambda-controlled}
  \setlength{\tabcolsep}{4.5pt}
  \begin{tabular}{@{}llccccccc@{}}
    \toprule
    Cohort & Measure
      & \(\lambda=0\)
      & \(0.05\)
      & \(\mathbf{0.10}\)
      & \(\mathbf{0.15}\)
      & \(0.25\)
      & \(0.50\)
      & \(0.75\) \\
    \midrule
    Math 1.5--4B (19)
      & \(\mathrm{Det.}\)
      & 19/19 & 15/19 & 18/19 & 18/19 & 18/19 & 17/19 & 17/19 \\
      & \((\mathrm{FP}(S),\mathrm{FP}(T))\)
      & (0, 0) & (0, 0) & (0, 0) & (0, 0) & (13, 19) & (0, 0) & (0, 0) \\
    \addlinespace
    Conversation (19)
      & \(\mathrm{Det.}\)
      & 8/19 & 12/19 & 15/19 & 15/19 & 15/19 & 15/19 & 16/19 \\
      & \((\mathrm{FP}(S),\mathrm{FP}(T))\)
      & (2, 2) & (2, 0) & (0, 0) & (0, 0) & (0, 4) & (2, 1) & (2, 0) \\
    \midrule
    \multicolumn{9}{@{}l}{\textit{Additional scale cohort}} \\
    Math 7--8B (21)
      & \(\mathrm{Det.}\)
      & 19/21 & 15/21 & 19/21 & 18/21 & 19/21 & 18/21 & 18/21 \\
      & \((\mathrm{FP}(S),\mathrm{FP}(T))\)
      & (0, 0) & (0, 0) & (0, 0) & (0, 0) & (13, 21) & (0, 0) & (0, 0) \\
    \bottomrule
  \end{tabular}
\end{table*}

Table~\ref{tab:lambda-controlled} shows consistent behavior at \(\lambda=0.10\) and \(0.15\). Both yield no false acceptances under either threshold-transfer protocol in any cohort while retaining high detection. Performance is less consistent outside these values, particularly in the conversation cohort and at \(\lambda=0.25\). The same pattern in the additional 7--8B cohort indicates that it is not confined to smaller students. All primary experiments use the fixed value \(\lambda=0.15\), one point in this stable region.

\paragraph{After Subsequent Training}

\begin{table*}[t]
  \centering
  \caption{\textbf{Sensitivity to \(\lambda\) after subsequent training.} Each entry counts DeepSeek-R1 assignments among the twelve descendants in Section~\ref{sec:exp-descendants}. Math and conversation probes use their corresponding ten-candidate pools. Profile distance \(D\) uses the same filtered representations without candidate-pool calibration. Bold headers mark the values highlighted by the controlled cohorts.}
  \label{tab:lambda-retrospective}
  \setlength{\tabcolsep}{5pt}
  \begin{tabular}{@{}llcccccc@{}}
    \toprule
    Method & Probe
      & \(\lambda=0\)
      & \(0.05\)
      & \(\mathbf{0.10}\)
      & \(\mathbf{0.15}\)
      & \(0.25\)
      & \(0.50\) \\
    \midrule
    SCOUT
      & Math / OMI
      & 12/12 & 12/12 & 12/12 & 12/12 & 12/12 & 12/12 \\
      & Math / s1
      & 12/12 & 12/12 & 12/12 & 12/12 & 12/12 & 12/12 \\
      & Conversation / OASST1
      & 1/12 & 5/12 & 12/12 & 12/12 & 1/12 & 1/12 \\
      & Conversation / Dolly
      & 3/12 & 12/12 & 12/12 & 12/12 & 12/12 & 1/12 \\
    \midrule
    Profile distance \(D\)
      & Math / OMI
      & 11/12 & 12/12 & 12/12 & 12/12 & 12/12 & 12/12 \\
      & Math / s1
      & 12/12 & 12/12 & 12/12 & 12/12 & 12/12 & 12/12 \\
      & Conversation / OASST1
      & 1/12 & 1/12 & 2/12 & 2/12 & 2/12 & 1/12 \\
      & Conversation / Dolly
      & 1/12 & 1/12 & 1/12 & 1/12 & 1/12 & 1/12 \\
    \bottomrule
  \end{tabular}
\end{table*}

Table~\ref{tab:lambda-retrospective} shows a consistent pattern after subsequent training. SCOUT assigns all twelve descendants to DeepSeek-R1 across the evaluated grid on both math probes. On the conversation probes, OASST1 and Dolly simultaneously reach \(12/12\) at \(\lambda=0.10\) and \(0.15\). In every SCOUT cell with \(12/12\) descendants, both distilled parents are also assigned to DeepSeek-R1 and both pre-distillation bases are correctly rejected as non-R1. Profile distance \(D\) does not recover the conversation descendants, indicating that contrast filtering alone does not explain SCOUT's performance. Together, the controlled and retrospective results identify \(0.10\)--\(0.15\) as a stable region across model scales, probe domains, and training stages.

\subsection{Decoding and Response-Length Controls}
\label{app:decoding}

This subsection tests whether the measured signal can be explained by differences in teacher-side decoding, candidate-side decoding, or total response length.

\paragraph{Common Teacher Decoding} The controlled training responses were generated using each teacher's default decoding configuration. Qwen-3-8B uses \(T=0.6\), top-\(p=0.95\), and top-\(k=20\); Llama-3.3-70B uses \(T=0.6\) and top-\(p=0.9\); GPT-OSS-120B provides no sampling defaults and was served at \(T=1.0\). Because these differences may affect PoS patterns, we regenerate training responses from all three teachers using \(T=0.7\), top-\(p=0.95\), top-\(k=20\), and repetition penalty \(1.0\), then train the corresponding students with the original recipe.

For each candidate \(T_k\), let \(\Delta_k\) denote the reduction in PoS distance from the base to the student. We measure source-directed formation as
\[
I_T
=
\Delta_T-\frac{1}{K-1}\sum_{k\neq T}\Delta_k,
\]
where positive values indicate greater movement toward the true source than toward the remaining candidates.

\begin{table}[t]
  \centering
  \caption{\textbf{Source-signature formation under common teacher decoding.} Entries report the source-directed increment \(I_T\) on the held-out probe. Dashes denote combinations outside the retained nineteen-student cohort.}
  \label{tab:common-decoding}
  \small
  \begin{tabular}{@{}llccc@{}}
    \toprule
    Backbone & Training set
      & GPT-OSS & Qwen-3 & Llama-3.3 \\
    \midrule
    Gemma-3-4B-PT
      & OMI & 0.087 & 0.169 & --- \\
      & s1  & 0.146 & 0.157 & 0.191 \\
    \addlinespace
    Llama-3.2-3B-Instruct
      & OMI & --- & --- & 0.192 \\
      & s1  & --- & --- & 0.228 \\
    \addlinespace
    Qwen2.5-1.5B
      & OMI & 0.108 & 0.179 & 0.139 \\
      & s1  & 0.169 & 0.081 & 0.176 \\
    \addlinespace
    Qwen2.5-3B
      & OMI & 0.108 & 0.203 & 0.168 \\
      & s1  & 0.175 & 0.234 & 0.191 \\
    \midrule
    Mean \(I_T\)
      & & 0.132 & 0.170 & 0.184 \\
    True source ranks first
      & & 6/6 & 6/6 & 7/7 \\
    \bottomrule
  \end{tabular}
\end{table}

Table~\ref{tab:common-decoding} shows positive \(I_T\) for every student in the nineteen-student cohort. The largest candidate-wise distance reduction identifies the true source in all \(19/19\) cases. The mean \(I_T\) is \(0.163\), compared with \(0.167\) for the same students under the original teacher configurations.

For exact inference, we permute teacher labels within five three-teacher blocks and one two-teacher block. These blocks cover \(17\) students; the other two Llama-3.2-3B-Instruct students have no within-block alternative. All \(17\) students have positive \(I_T\) and move most toward the true source, with a mean \(I_T\) of \(0.158\). The observed assignment uniquely maximizes the mean among all \((3!)^5 2!=15{,}552\) assignments, giving an exact one-sided \(p=6.4\times10^{-5}\). Equalizing teacher decoding therefore does not remove source-signature formation.

\paragraph{Temperature and Candidate-Side Decoding} We next measure the magnitude of the decoding effect. GPT-OSS training responses are generated at three temperatures while all other settings remain fixed. On the candidate side, we regenerate Gemma-3-27B-it, Llama-3.3-70B-Instruct, GPT-OSS-120B, and QwQ-32B-Preview under both their model defaults and the common configuration above. The other six candidates, including the true source DeepSeek-R1, remain fixed. We denote the true-source margin computed from uncalibrated profile distance \(D\) by \(G_T^{(D)}\).

\begin{table}[t]
  \centering
  \caption{\textbf{Effects of teacher- and candidate-side decoding.} The upper block reports \(G_T^{(D)}\) for students trained on GPT-OSS responses at three temperatures. The middle block compares its mean across teachers under common decoding. In the lower block, \emph{Released} uses the original candidate bank, while \emph{Default} and \emph{Common} replace four open-weight competitors with regenerated responses. Entries count correct source identifications in thirteen public suspect--probe cells using profile readouts with \(\lambda=0.15\). The s1.1-32B cell on s1 is excluded because its fine-tuning data contain the probe prompts. \emph{Flips} counts verdict changes between default and common candidate decoding.}
  \label{tab:decoding-interventions}
  \small
  \begin{tabular}{@{}lcccc@{}}
    \toprule
    & \multicolumn{3}{c}{Temperature} & \\
    \cmidrule(lr){2-4}
    Backbone
      & \(T=0.3\) & \(T=0.7\) & \(T=1.0\) & Range \\
    \midrule
    Gemma-3-4B-PT
      & 0.0839 & 0.0865 & 0.0937 & 0.0098 \\
    Qwen2.5-1.5B
      & 0.0602 & 0.0777 & 0.0697 & 0.0176 \\
    Qwen2.5-3B
      & 0.0833 & 0.0988 & 0.0845 & 0.0155 \\
    \midrule
    Mean range
      & & & & 0.0143 \\
    \midrule
    & GPT-OSS & Qwen-3 & Llama-3.3 & Range \\
    \midrule
    Mean \(G_T^{(D)}\)
      & 0.0647 & 0.1400 & 0.2578 & 0.1931 \\
    \midrule
    Readout
      & Released & Default & Common & Flips \\
    \midrule
    Profile distance \(D\)
      & 13/13 & 13/13 & 13/13 & 0/13 \\
    \quad \(+\) Centering \(C\)
      & 13/13 & 13/13 & 13/13 & 0/13 \\
    \qquad \(+\) Scaling (SCOUT)
      & 13/13 & 13/13 & 13/13 & 0/13 \\
    \bottomrule
  \end{tabular}
\end{table}

The mean temperature-induced range in \(G_T^{(D)}\) is \(0.0143\), compared with a \(0.1931\) range between teachers. Changing candidate-side decoding alters none of the \(39\) verdicts across the three profile readouts. The tested training-side temperature variation is therefore substantially smaller than the source variation, and changing the four regenerable competitors does not alter attribution.

\paragraph{Response-Length Controls} At audit time, SCOUT retains only the first \(2{,}000\) characters of every student and candidate response and \(L_2\)-normalizes the resulting count vector. Table~\ref{tab:length-controls} measures the remaining length variation and tests whether simple length statistics recover the source. For each prompt, a length-only readout scores candidate \(T_k\) by the absolute difference between the student's response statistic and the median statistic over that candidate's \(200\) responses; the released evaluator then aggregates these scores across prompts.

\begin{table}[t]
  \centering
  \caption{\textbf{Response-length controls.} The upper block reports candidate-response statistics on OMI before and after the \(2{,}000\)-character cap; entries are medians except \emph{At cap}, which gives the percentage of responses reaching the cap. \emph{Sentence chars} reports median sentence length within the retained window. The lower block compares three length-only readouts with SCOUT on the same nineteen controlled students. SCOUT uses \(\lambda=0.15\), and lower-block entries are percentages.}
  \label{tab:length-controls}
  \small
  \begin{tabular}{@{}lrrrr@{}}
    \toprule
    Candidate
      & Raw chars & Retained chars & At cap (\%) & Sentence chars \\
    \midrule
    GPT-OSS-120B
      & 2,877 & 2,000 & 63.5 & 54.7 \\
    Qwen-3-8B
      & 6,036 & 2,000 & 100.0 & 62.5 \\
    Llama-3.3-70B-Instruct
      & 1,851 & 1,851 & 44.5 & 117.6 \\
    Gemma-3-27B-it
      & 2,013 & 1,993 & 50.0 & 80.2 \\
    \midrule
    Max/min ratio
      & 3.26 & 1.08 & --- & 2.15 \\
    \bottomrule
  \end{tabular}

  \vspace{4pt}

  \begin{tabular}{@{}lrrr@{}}
    \toprule
    Readout
      & GPT-OSS & Qwen-3 & Llama-3.3 \\
    \midrule
    Character count
      & 100.0 & 0.0 & 14.3 \\
    Token count
      & 66.7 & 83.3 & 0.0 \\
    Line count
      & 50.0 & 50.0 & 42.9 \\
    SCOUT (Ours)
      & 100.0 & 100.0 & 100.0 \\
    \bottomrule
  \end{tabular}
\end{table}

The character cap reduces the ratio between the longest and shortest candidate medians from \(3.26\) to \(1.08\). Each length-only readout succeeds for some sources but fails for others, whereas SCOUT identifies every source group. Therefore, total length does not explain the controlled result, although sentence structure within the retained window may still contribute to the PoS representation.

Together, these controls show that neither the tested decoding variation nor total retained length alone accounts for the attribution results.

\section{Properties and Limits of the SCOUT Readout}
\label{app:theory}
This appendix establishes four properties of SCOUT. It identifies what candidate-pool calibration removes, when attribution remains stable under prompt subsampling, when candidate-only calibration cannot remove pair-specific affinity, and what current outputs cannot reveal about training history. The calibration identities condition on each supplied candidate pool, while the subsampling analysis fixes the candidate responses, profiles, null statistics, and scales. Randomness in the latter enters only through subsets drawn from a fixed prompt bank. These results characterize the readout rather than the causal effects of distillation or post-training. Then, we connect the final limit to source-switching experiments that examine which source remains readable after later SFT.

\subsection{Candidate-Relative Calibration}

\begin{proposition}[Exact finite-pool calibration]
\label{prop:centering}
For any evaluated model \(A\), we have
\begin{equation}
C(A,T_k)
=
D(A,T_k)
-
\frac{1}{K-1}\sum_{j\neq k}D(T_j,T_k).
\label{eq:centered-identity}
\end{equation}
Moreover, assume \(\sigma(T_k)>0\) and suppose every distance to \(T_k\) used in its calibrated score receives the same candidate-specific additive shift \(b_k\):
\[
\delta_i'(M,T_k)=\delta_i(M,T_k)+b_k
\]
for every \(M\in\{A\}\cup(\mathcal{T}\setminus\{T_k\})\) and prompt \(i\). Then,
\begin{equation}
C'(A,T_k)=C(A,T_k),
\qquad
Z'(A,T_k)=Z(A,T_k).
\end{equation}
\end{proposition}

\begin{proof}
Exchanging the order of summation gives
\begin{align*}
\frac{1}{N}\sum_{i=1}^{N}\mu_i(T_k)
&=
\frac{1}{N(K-1)}
\sum_{i=1}^{N}\sum_{j\neq k}\delta_i(T_j,T_k) \\
&=
\frac{1}{K-1}\sum_{j\neq k}D(T_j,T_k).
\end{align*}
Substitution into the definition of \(C\) yields \eqref{eq:centered-identity}. Under the additive shift, \(\mu_i'(T_k)=\mu_i(T_k)+b_k\), so every centered distance is unchanged. The standard deviation is also unchanged because all distances entering it receive the same shift. Hence both \(C\) and \(Z\) remain unchanged.
\end{proof}

Thus, centering removes an offset shared by the evaluated model and all pool members relative to \(T_k\). It prevents \(T_k\) from being favored solely because all models are uniformly close to it. The null is computed without the evaluated model. Therefore, a low \(Z(A,T_k)\) indicates closeness relative to the supplied pool but does not by itself establish distillation.

\paragraph{Effect of pool composition}
\eqref{eq:centered-identity} also shows how the supplied pool affects centering. For \(K\geq3\), holding the representation, profiles, and distances fixed, removing a candidate \(T_r\) with \(r\neq k\) changes the centered score by
\begin{equation}
C^{(-r)}(A,T_k)-C(A,T_k)
=
\frac{1}{N(K-2)}
\sum_{i=1}^{N}
\left[
\delta_i(T_r,T_k)-\mu_i(T_k)
\right].
\label{eq:centered-removal}
\end{equation}
Therefore, including a pool member that is unusually close to \(T_k\) makes \(T_k\)'s centered score less favorable. This identity isolates the centering term because removing a candidate also changes \(\sigma(T_k)\). In our implementation, changing the pool may also alter the PoS \(n\)-gram vocabulary and candidate-contrast mask \(w_\lambda\), thereby changing the distances themselves.

\subsection{Finite-Prompt Stability}

For a fixed audited model \(S\), define the standardized per-prompt score
\begin{equation}
z_{ik}
=
\frac{\delta_i(S,T_k)-\mu_i(T_k)}{\sigma(T_k)}.
\label{eq:per-prompt-z}
\end{equation}
For a subset \(I\) of \(B\) audit prompts, let
\begin{equation}
Z_I(S,T_k)
=
\frac{1}{B}\sum_{i\in I}z_{ik}.
\label{eq:subset-z}
\end{equation}

\begin{proposition}[Stability under prompt subsampling]
\label{prop:finite-sample}
Assume \(\sigma(T_k)>0\) for every candidate. Let \(T_{k^\star}\) be the true source and suppose it uniquely minimizes the full-bank score. For each \(k\neq k^\star\), define
\begin{equation*}
g_{ik}=z_{ik}-z_{ik^\star},
\qquad
\gamma_k
=
\frac{1}{N}\sum_{i=1}^{N}g_{ik}
=
Z(S,T_k)-Z(S,T_{k^\star})
>0,
\end{equation*}
with range width
\[
L_k=\max_i g_{ik}-\min_i g_{ik}>0.
\]
If \(I\) is sampled uniformly without replacement from the \(N\) audit prompts, with \(1\leq B\leq N\), then
\begin{equation}
\Pr\!\left(
\arg\min_k Z_I(S,T_k)\neq k^\star
\right)
\leq
\sum_{k\neq k^\star}
\exp\left(
-\frac{2B\gamma_k^2}{L_k^2}
\right).
\label{eq:finite-prompt-bound}
\end{equation}
Consequently, with \(\gamma=\min_{k\neq k^\star}\gamma_k\) and \(L=\max_{k\neq k^\star}L_k\),
\begin{equation}
B
\geq
\frac{L^2}{2\gamma^2}
\log\frac{K-1}{\varepsilon}
\label{eq:prompt-sufficiency}
\end{equation}
is sufficient for subsampling error at most \(\varepsilon\).
\end{proposition}

\begin{proof}
A competitor \(T_k\) can defeat \(T_{k^\star}\) only if
\begin{equation*}
Z_I(S,T_k)-Z_I(S,T_{k^\star})
=
\frac{1}{B}\sum_{i\in I}g_{ik}
\leq 0.
\end{equation*}
Applying Hoeffding's inequality for sampling without replacement \citep{hoeffding1963probability} to each competitor gives
\begin{equation*}
\Pr\left(
\frac{1}{B}\sum_{i\in I}g_{ik}\leq0
\right)
\leq
\exp\left(
-\frac{2B\gamma_k^2}{L_k^2}
\right).
\end{equation*}
A union bound over the \(K-1\) competitors yields \eqref{eq:finite-prompt-bound}. Bounding every summand using \(\gamma\) and \(L\) and solving for \(B\) gives \eqref{eq:prompt-sufficiency}.
\end{proof}

This result is a stability guarantee rather than an unconditional attribution guarantee. It assumes that the true source wins on the full prompt bank and bounds the probability that prompt subsampling changes that ranking. The required budget is governed by the hardest competitor rather than average separation across the pool. Requiring all \(J\) audited models to retain the correct attribution adds a union bound over models and replaces \(K-1\) in \eqref{eq:prompt-sufficiency} with \(J(K-1)\), with \(\gamma\) and \(L\) taken over all model and competitor pairs.

The bound applies directly to mean-score attribution, for which attribution is determined by the audited-model average in \eqref{eq:subset-z}. It does not directly cover the sorted-score voting component included in the original-evaluator rows of Tables~\ref{tab:budget-math} and~\ref{tab:budget-conversation}. In both cases, candidate responses, profiles, contrast masks, null statistics, and scales remain fixed while only the audited-model responses are subsampled.

\paragraph{Prompt-budget resampling}
We evaluate the twelve descendants from Section~\ref{sec:exp-descendants}. For each budget \(B<200\), we draw \(2{,}000\) uniformly random \(B\)-prompt subsets without replacement from the \(200\)-prompt bank and apply each subset to all descendants. The full bank at \(B=200\) is evaluated once. For score-based methods, the original evaluator counts each descendant as correct when either mean-score attribution or sorted-score voting assigns it to R1, and a subset succeeds only when all twelve descendants are correct. We also report mean-score attribution alone because it corresponds directly to Proposition~\ref{prop:finite-sample}. Each LLM judge assigns every descendant using its unique most frequent prompt-level label, with ties yielding no verdict. Tables~\ref{tab:budget-math} and~\ref{tab:budget-conversation} report the resulting full curves and the first evaluated budget \(n^\star\) at which the success rate reaches \(0.95\).

\begin{table*}[t]
\centering
\caption{\textbf{Prompt-budget curves on math probes.} Each value for \(B<200\) is the fraction of \(2{,}000\) prompt subsets on which all twelve descendants are assigned to R1. Under the original evaluator, each descendant is correct when either mean-score attribution or sorted-score voting assigns it to R1, and a subset succeeds only when all twelve descendants are correct. Mean-score attribution uses the lowest average candidate score alone. At \(B=200\), the complete prompt bank is evaluated once. \(n^\star\) is the first evaluated budget reaching \(0.95\). Dashes indicate that this level is not reached.}
\label{tab:budget-math}
\setlength{\tabcolsep}{2.5pt}
\resizebox{\textwidth}{!}{%
\begin{tabular}{ll*{15}{c}}
\toprule
Method & Aggregation
& 5 & 10 & 20 & 50 & 100 & 110 & 120 & 130
& 140 & 150 & 160 & 175 & 190 & 200 & \(n^\star\) \\
\midrule

\multicolumn{17}{l}{\textbf{OMI}} \\
\rowcolor{black!7}
\multicolumn{17}{l}{\textit{Reference-based}} \\
\rowcolor{black!7}
DistillDetect & Original evaluator
& 0.9885 & 1.0000 & 1.0000 & 1.0000 & 1.0000 & 1.0000 & 1.0000 & 1.0000
& 1.0000 & 1.0000 & 1.0000 & 1.0000 & 1.0000 & 1.0000 & 5 \\
\rowcolor{black!7}
DistillDetect & Mean-score attribution
& 0.9585 & 0.9905 & 0.9995 & 1.0000 & 1.0000 & 1.0000 & 1.0000 & 1.0000
& 1.0000 & 1.0000 & 1.0000 & 1.0000 & 1.0000 & 1.0000 & 5 \\
\midrule
\multicolumn{17}{l}{\textit{Output-only}} \\
PoS Templates & Original evaluator
& 0.0115 & 0.0095 & 0.0025 & 0.0000 & 0.0000 & 0.0000 & 0.0000 & 0.0000
& 0.0000 & 0.0000 & 0.0000 & 0.0000 & 0.0000 & 0.0000 & --- \\
PoS Templates & Mean-score attribution
& 0.0060 & 0.0060 & 0.0020 & 0.0000 & 0.0000 & 0.0000 & 0.0000 & 0.0000
& 0.0000 & 0.0000 & 0.0000 & 0.0000 & 0.0000 & 0.0000 & --- \\
LLM judge (Mistral) & Most frequent label
& 0.0045 & 0.0280 & 0.0860 & 0.2835 & 0.5495 & 0.5630 & 0.6085 & 0.6310
& 0.6640 & 0.6960 & 0.7060 & 0.7685 & 0.8875 & 1.0000 & 200 \\
LLM judge (Solar) & Most frequent label
& 0.0315 & 0.1155 & 0.2525 & 0.6075 & 0.8900 & 0.9230 & 0.9485 & 0.9630
& 0.9825 & 0.9895 & 1.0000 & 1.0000 & 1.0000 & 1.0000 & 130 \\
\cmidrule{1-17}
SCOUT (Ours) & Original evaluator
& 0.5400 & 0.7025 & 0.8665 & 0.9830 & 1.0000 & 1.0000 & 1.0000 & 1.0000
& 1.0000 & 1.0000 & 1.0000 & 1.0000 & 1.0000 & 1.0000 & 50 \\
SCOUT (Ours) & Mean-score attribution
& 0.3975 & 0.5995 & 0.8160 & 0.9800 & 1.0000 & 1.0000 & 1.0000 & 1.0000
& 1.0000 & 1.0000 & 1.0000 & 1.0000 & 1.0000 & 1.0000 & 50 \\
\quad \( - \) Scaling (\(C\)) & Original evaluator
& 0.5180 & 0.6850 & 0.8525 & 0.9825 & 1.0000 & 1.0000 & 1.0000 & 1.0000
& 1.0000 & 1.0000 & 1.0000 & 1.0000 & 1.0000 & 1.0000 & 50 \\
\quad \( - \) Scaling (\(C\)) & Mean-score attribution
& 0.3920 & 0.5915 & 0.8040 & 0.9775 & 0.9995 & 1.0000 & 1.0000 & 1.0000
& 1.0000 & 1.0000 & 1.0000 & 1.0000 & 1.0000 & 1.0000 & 50 \\
\qquad \( - \) Centering (\(D\)) & Original evaluator
& 0.4360 & 0.4500 & 0.5730 & 0.6975 & 0.7650 & 0.7660 & 0.7940 & 0.8300
& 0.8290 & 0.8405 & 0.8795 & 0.9305 & 0.9930 & 1.0000 & 190 \\
\qquad \( - \) Centering (\(D\)) & Mean-score attribution
& 0.2965 & 0.3245 & 0.4175 & 0.5485 & 0.6785 & 0.6800 & 0.7105 & 0.7630
& 0.7760 & 0.7930 & 0.8475 & 0.9150 & 0.9930 & 1.0000 & 190 \\

\midrule
\multicolumn{17}{l}{\textbf{s1}} \\
\rowcolor{black!7}
\multicolumn{17}{l}{\textit{Reference-based}} \\
\rowcolor{black!7}
DistillDetect & Original evaluator
& 0.7270 & 0.7990 & 0.8520 & 0.9415 & 0.9950 & 0.9980 & 0.9985 & 0.9995
& 1.0000 & 1.0000 & 1.0000 & 1.0000 & 1.0000 & 1.0000 & 100 \\
\rowcolor{black!7}
DistillDetect & Mean-score attribution
& 0.6525 & 0.6995 & 0.7560 & 0.8820 & 0.9825 & 0.9915 & 0.9970 & 0.9990
& 0.9995 & 1.0000 & 1.0000 & 1.0000 & 1.0000 & 1.0000 & 100 \\
\midrule
\multicolumn{17}{l}{\textit{Output-only}} \\
PoS Templates & Original evaluator
& 0.0320 & 0.0360 & 0.0325 & 0.0095 & 0.0000 & 0.0000 & 0.0000 & 0.0000
& 0.0000 & 0.0000 & 0.0000 & 0.0000 & 0.0000 & 0.0000 & --- \\
PoS Templates & Mean-score attribution
& 0.0185 & 0.0215 & 0.0185 & 0.0060 & 0.0000 & 0.0000 & 0.0000 & 0.0000
& 0.0000 & 0.0000 & 0.0000 & 0.0000 & 0.0000 & 0.0000 & --- \\
LLM judge (Mistral) & Most frequent label
& 0.0010 & 0.0050 & 0.0280 & 0.1235 & 0.2840 & 0.3275 & 0.3865 & 0.3770
& 0.3930 & 0.4175 & 0.4460 & 0.4285 & 0.4095 & 0.0000 & --- \\
LLM judge (Solar) & Most frequent label
& 0.0195 & 0.0505 & 0.1470 & 0.4705 & 0.8035 & 0.8275 & 0.8850 & 0.9025
& 0.9360 & 0.9620 & 0.9840 & 0.9975 & 1.0000 & 1.0000 & 150 \\
\cmidrule{1-17}
SCOUT (Ours) & Original evaluator
& 0.7745 & 0.9340 & 0.9955 & 1.0000 & 1.0000 & 1.0000 & 1.0000 & 1.0000
& 1.0000 & 1.0000 & 1.0000 & 1.0000 & 1.0000 & 1.0000 & 20 \\
SCOUT (Ours) & Mean-score attribution
& 0.6680 & 0.8835 & 0.9780 & 0.9995 & 1.0000 & 1.0000 & 1.0000 & 1.0000
& 1.0000 & 1.0000 & 1.0000 & 1.0000 & 1.0000 & 1.0000 & 20 \\
\quad \( - \) Scaling (\(C\)) & Original evaluator
& 0.7635 & 0.9290 & 0.9925 & 1.0000 & 1.0000 & 1.0000 & 1.0000 & 1.0000
& 1.0000 & 1.0000 & 1.0000 & 1.0000 & 1.0000 & 1.0000 & 20 \\
\quad \( - \) Scaling (\(C\)) & Mean-score attribution
& 0.6630 & 0.8810 & 0.9760 & 0.9995 & 1.0000 & 1.0000 & 1.0000 & 1.0000
& 1.0000 & 1.0000 & 1.0000 & 1.0000 & 1.0000 & 1.0000 & 20 \\
\qquad \( - \) Centering (\(D\)) & Original evaluator
& 0.9155 & 0.9935 & 1.0000 & 1.0000 & 1.0000 & 1.0000 & 1.0000 & 1.0000
& 1.0000 & 1.0000 & 1.0000 & 1.0000 & 1.0000 & 1.0000 & 10 \\
\qquad \( - \) Centering (\(D\)) & Mean-score attribution
& 0.8760 & 0.9875 & 0.9990 & 1.0000 & 1.0000 & 1.0000 & 1.0000 & 1.0000
& 1.0000 & 1.0000 & 1.0000 & 1.0000 & 1.0000 & 1.0000 & 10 \\
\bottomrule
\end{tabular}%
}
\end{table*}

\begin{table*}[t]
\centering
\caption{\textbf{Prompt-budget curves on conversation probes.} Entries follow Table~\ref{tab:budget-math}. Each value for \(B<200\) is the fraction of \(2{,}000\) prompt subsets on which all twelve descendants are assigned to R1. At \(B=200\), the complete prompt bank is evaluated once. Dashes indicate that the \(0.95\) criterion is not reached.}
\label{tab:budget-conversation}
\setlength{\tabcolsep}{2.5pt}
\resizebox{\textwidth}{!}{%
\begin{tabular}{ll*{15}{c}}
\toprule
Method & Aggregation
& 5 & 10 & 20 & 50 & 100 & 110 & 120 & 130
& 140 & 150 & 160 & 175 & 190 & 200 & \(n^\star\) \\
\midrule

\multicolumn{17}{l}{\textbf{OASST1}} \\
\rowcolor{black!7}
\multicolumn{17}{l}{\textit{Reference-based}} \\
\rowcolor{black!7}
DistillDetect & Original evaluator
& 0.0740 & 0.0275 & 0.0060 & 0.0000 & 0.0000 & 0.0000 & 0.0000 & 0.0000
& 0.0000 & 0.0000 & 0.0000 & 0.0000 & 0.0000 & 0.0000 & --- \\
\rowcolor{black!7}
DistillDetect & Mean-score attribution
& 0.0540 & 0.0220 & 0.0035 & 0.0000 & 0.0000 & 0.0000 & 0.0000 & 0.0000
& 0.0000 & 0.0000 & 0.0000 & 0.0000 & 0.0000 & 0.0000 & --- \\
\midrule
\multicolumn{17}{l}{\textit{Output-only}} \\
PoS Templates & Original evaluator
& 0.0000 & 0.0000 & 0.0000 & 0.0000 & 0.0000 & 0.0000 & 0.0000 & 0.0000
& 0.0000 & 0.0000 & 0.0000 & 0.0000 & 0.0000 & 0.0000 & --- \\
PoS Templates & Mean-score attribution
& 0.0000 & 0.0000 & 0.0000 & 0.0000 & 0.0000 & 0.0000 & 0.0000 & 0.0000
& 0.0000 & 0.0000 & 0.0000 & 0.0000 & 0.0000 & 0.0000 & --- \\
LLM judge (Mistral) & Most frequent label
& 0.0000 & 0.0000 & 0.0000 & 0.0000 & 0.0000 & 0.0000 & 0.0000 & 0.0000
& 0.0000 & 0.0000 & 0.0000 & 0.0000 & 0.0000 & 0.0000 & --- \\
LLM judge (Solar) & Most frequent label
& 0.0025 & 0.0030 & 0.0005 & 0.0000 & 0.0000 & 0.0000 & 0.0000 & 0.0000
& 0.0000 & 0.0000 & 0.0000 & 0.0000 & 0.0000 & 0.0000 & --- \\
\cmidrule{1-17}
SCOUT (Ours) & Original evaluator
& 0.6995 & 0.9225 & 0.9955 & 1.0000 & 1.0000 & 1.0000 & 1.0000 & 1.0000
& 1.0000 & 1.0000 & 1.0000 & 1.0000 & 1.0000 & 1.0000 & 20 \\
SCOUT (Ours) & Mean-score attribution
& 0.6265 & 0.8940 & 0.9930 & 1.0000 & 1.0000 & 1.0000 & 1.0000 & 1.0000
& 1.0000 & 1.0000 & 1.0000 & 1.0000 & 1.0000 & 1.0000 & 20 \\
\quad \( - \) Scaling (\(C\)) & Original evaluator
& 0.6940 & 0.9190 & 0.9935 & 1.0000 & 1.0000 & 1.0000 & 1.0000 & 1.0000
& 1.0000 & 1.0000 & 1.0000 & 1.0000 & 1.0000 & 1.0000 & 20 \\
\quad \( - \) Scaling (\(C\)) & Mean-score attribution
& 0.6105 & 0.8825 & 0.9880 & 1.0000 & 1.0000 & 1.0000 & 1.0000 & 1.0000
& 1.0000 & 1.0000 & 1.0000 & 1.0000 & 1.0000 & 1.0000 & 20 \\
\qquad \( - \) Centering (\(D\)) & Original evaluator
& 0.0065 & 0.0010 & 0.0000 & 0.0000 & 0.0000 & 0.0000 & 0.0000 & 0.0000
& 0.0000 & 0.0000 & 0.0000 & 0.0000 & 0.0000 & 0.0000 & --- \\
\qquad \( - \) Centering (\(D\)) & Mean-score attribution
& 0.0035 & 0.0010 & 0.0000 & 0.0000 & 0.0000 & 0.0000 & 0.0000 & 0.0000
& 0.0000 & 0.0000 & 0.0000 & 0.0000 & 0.0000 & 0.0000 & --- \\

\midrule
\multicolumn{17}{l}{\textbf{Dolly}} \\
\rowcolor{black!7}
\multicolumn{17}{l}{\textit{Reference-based}} \\
\rowcolor{black!7}
DistillDetect & Original evaluator
& 0.1750 & 0.1120 & 0.0510 & 0.0045 & 0.0000 & 0.0000 & 0.0000 & 0.0000
& 0.0000 & 0.0000 & 0.0000 & 0.0000 & 0.0000 & 0.0000 & --- \\
\rowcolor{black!7}
DistillDetect & Mean-score attribution
& 0.1415 & 0.1005 & 0.0420 & 0.0035 & 0.0000 & 0.0000 & 0.0000 & 0.0000
& 0.0000 & 0.0000 & 0.0000 & 0.0000 & 0.0000 & 0.0000 & --- \\
\midrule
\multicolumn{17}{l}{\textit{Output-only}} \\
PoS Templates & Original evaluator
& 0.0000 & 0.0000 & 0.0000 & 0.0000 & 0.0000 & 0.0000 & 0.0000 & 0.0000
& 0.0000 & 0.0000 & 0.0000 & 0.0000 & 0.0000 & 0.0000 & --- \\
PoS Templates & Mean-score attribution
& 0.0000 & 0.0000 & 0.0000 & 0.0000 & 0.0000 & 0.0000 & 0.0000 & 0.0000
& 0.0000 & 0.0000 & 0.0000 & 0.0000 & 0.0000 & 0.0000 & --- \\
LLM judge (Mistral) & Most frequent label
& 0.0020 & 0.0000 & 0.0000 & 0.0000 & 0.0000 & 0.0000 & 0.0000 & 0.0000
& 0.0000 & 0.0000 & 0.0000 & 0.0000 & 0.0000 & 0.0000 & --- \\
LLM judge (Solar) & Most frequent label
& 0.0025 & 0.0055 & 0.0045 & 0.0015 & 0.0000 & 0.0000 & 0.0000 & 0.0000
& 0.0000 & 0.0000 & 0.0000 & 0.0000 & 0.0000 & 0.0000 & --- \\
\cmidrule{1-17}
SCOUT (Ours) & Original evaluator
& 0.8755 & 0.9805 & 1.0000 & 1.0000 & 1.0000 & 1.0000 & 1.0000 & 1.0000
& 1.0000 & 1.0000 & 1.0000 & 1.0000 & 1.0000 & 1.0000 & 10 \\
SCOUT (Ours) & Mean-score attribution
& 0.8055 & 0.9495 & 0.9950 & 1.0000 & 1.0000 & 1.0000 & 1.0000 & 1.0000
& 1.0000 & 1.0000 & 1.0000 & 1.0000 & 1.0000 & 1.0000 & 20 \\
\quad \( - \) Scaling (\(C\)) & Original evaluator
& 0.8780 & 0.9815 & 0.9995 & 1.0000 & 1.0000 & 1.0000 & 1.0000 & 1.0000
& 1.0000 & 1.0000 & 1.0000 & 1.0000 & 1.0000 & 1.0000 & 10 \\
\quad \( - \) Scaling (\(C\)) & Mean-score attribution
& 0.8035 & 0.9490 & 0.9950 & 1.0000 & 1.0000 & 1.0000 & 1.0000 & 1.0000
& 1.0000 & 1.0000 & 1.0000 & 1.0000 & 1.0000 & 1.0000 & 20 \\
\qquad \( - \) Centering (\(D\)) & Original evaluator
& 0.0010 & 0.0000 & 0.0000 & 0.0000 & 0.0000 & 0.0000 & 0.0000 & 0.0000
& 0.0000 & 0.0000 & 0.0000 & 0.0000 & 0.0000 & 0.0000 & --- \\
\qquad \( - \) Centering (\(D\)) & Mean-score attribution
& 0.0000 & 0.0000 & 0.0000 & 0.0000 & 0.0000 & 0.0000 & 0.0000 & 0.0000
& 0.0000 & 0.0000 & 0.0000 & 0.0000 & 0.0000 & 0.0000 & --- \\
\bottomrule
\end{tabular}%
}
\end{table*}

Across the four probes, the two score aggregations yield the same \(n^\star\) except on Dolly, where SCOUT and its centered variant move from \(10\) prompts under the original evaluator to \(20\) under mean-score attribution. On OMI, SCOUT reaches the criterion at \(50\) prompts under both aggregations, earlier than the uncalibrated profile and output-only baselines. Raw profile distance reaches the criterion earlier on s1 but never reaches it on either conversation probe. On OASST1 and Dolly, centering separates SCOUT from every baseline, while candidate-specific scaling produces nearly identical prompt-budget curves. These results show that the conversation performance does not depend on the original evaluator combining two score aggregations.

\paragraph{Stability of attribution and abstention} Let \(\mathbf{z}\) and \(\mathbf{z}'\) be two candidate-score vectors satisfying
\(
\lVert \mathbf{z}'-\mathbf{z}\rVert_\infty\leq\eta,
\)
and let \(H(\mathbf{z})\) be the gap between the smallest and second-smallest scores. Since both order statistics can change by at most \(\eta\),
\(
\left|H(\mathbf{z}')-H(\mathbf{z})\right|\leq2\eta.
\)
The selected candidate is unchanged when \(H(\mathbf{z})>2\eta\). For a fixed threshold \(\tau\), the accept--abstain decision is unchanged whenever
\(
H(\mathbf{z})\geq\tau+2\eta
\)
or
\(
H(\mathbf{z})<\tau-2\eta.
\)
Therefore, the decision margin in \eqref{eq:margin-G} also measures robustness to score perturbations.

\subsection{Limits of Candidate-Only Calibration}

Candidate-pool calibration applies the same candidate-specific correction to every audited model. Such a correction fully removes an interaction when the remaining score differs from the source-related component only by a term shared across candidates. The following result characterizes when this is possible.

\begin{proposition}[Candidate-only correction]
\label{prop:candidate-only}
Suppose an audited-model--candidate distance admits the decomposition
\begin{equation}
d(S,k)=t(S,k)+a(S,k)+b(k),
\end{equation}
where \(t(S,k)\) is the source-related component, \(b(k)\) is a candidate-wide component, and \(a(S,k)\) is an audited-model--candidate interaction. Fix a positive candidate-dependent scale \(s(k)\). There exist a candidate-only correction \(c(k)\) and a function \(h(S)\) such that
\begin{equation}
\frac{d(S,k)-c(k)}{s(k)}
=
\frac{t(S,k)}{s(k)}+h(S)
\end{equation}
for every \(S\) and \(k\) if and only if
\begin{equation}
\frac{a(S,k)}{s(k)}=u(S)+v(k)
\end{equation}
for some functions \(u\) and \(v\).
\end{proposition}

\begin{proof}
If such \(c(k)\) and \(h(S)\) exist, substituting the decomposition of \(d(S,k)\) gives
\begin{equation*}
\frac{a(S,k)}{s(k)}
=
h(S)+\frac{c(k)-b(k)}{s(k)},
\end{equation*}
which has the form \(u(S)+v(k)\). Conversely, if \(a(S,k)/s(k)=u(S)+v(k)\), choose
\begin{equation*}
c(k)=b(k)+s(k)v(k).
\end{equation*}
The required relation then holds with \(h(S)=u(S)\).
\end{proof}

Because \(h(S)\) is shared across candidates, it does not affect their ranking for a fixed audited model. For SCOUT, the candidate-only correction and scale are
\[
c(k)=\frac{1}{K-1}\sum_{j\neq k}D(T_j,T_k),
\qquad
s(k)=\sigma(T_k).
\]
Hence, it can remove candidate-wide proximity but cannot generally remove affinity specific to one audited-model--candidate pair. This limitation applies to candidate-only calibration generally rather than SCOUT alone. Shared model lineage may induce such an interaction, although the proposition does not imply that every lineage effect is nonseparable or unresolvable with additional information.

\subsection{Limits of Output-Only History Recovery}

\begin{proposition}[Limits of output-only history recovery]
\label{prop:history}
Let \(h_1\) and \(h_2\) be two training histories, and let \(P_{h_1}\) and \(P_{h_2}\) denote the distributions they induce over the complete current response observation \(\mathcal{Y}\) under a fixed prompting and decoding protocol. With equal prior probability on the histories, any output-only auditor distinguishing \(h_1\) from \(h_2\) has balanced accuracy at most
\begin{equation}
\operatorname{BA}
\leq
\frac{1+\operatorname{TV}(P_{h_1},P_{h_2})}{2}.
\label{eq:history-tv-bound}
\end{equation}
\end{proposition}

\begin{proof}
Let \(A\) be the set of observations for which the auditor selects \(h_1\). Its balanced accuracy is
\begin{align*}
\operatorname{BA}
&=
\frac{1}{2}P_{h_1}(A)
+
\frac{1}{2}\bigl(1-P_{h_2}(A)\bigr) \\
&=
\frac{1}{2}
+
\frac{1}{2}
\left(
P_{h_1}(A)-P_{h_2}(A)
\right) \\
&\leq
\frac{1+\operatorname{TV}(P_{h_1},P_{h_2})}{2},
\end{align*}
by the definition of total variation distance.
\end{proof}

If the two histories induce identical current-output distributions, no output-only auditor can exceed chance-level balanced accuracy. The proposition is conditional. It bounds recoverability when two histories induce similar current-output distributions but does not assert that post-training necessarily makes those distributions similar. The source-switching results below show that the evaluated output-based readouts can lose the earlier signature. They do not estimate total variation or establish failure for every possible auditor. Consequently, abstention indicates insufficient evidence in current outputs rather than proof that a candidate was never used earlier in training.

\subsection{Source Switching Under Later SFT}

Section~\ref{sec:findings} examines later training that does not introduce a new candidate source. In this section, we consider a second SFT round using responses from another candidate. Let \(A\) and \(B\) denote the first- and second-round sources, respectively.

\paragraph{Source Redirection}
We start from eighteen controlled students spanning three base models, three first-round sources, and two training sets. Continuing each student on either alternative source yields \(36\) switch arms, while continuing on \(A\) yields \(18\) same-source controls. Each second round uses \(800\) responses, and evaluation uses the probe set excluded from training.

Figure~\ref{fig:switch-attribution} shows rapid redirection toward \(B\). SCOUT selects \(B\) in \(16/36\) arms after \(5\%\) of the second round, \(29/36\) after \(10\%\), and \(33/36\) after \(15\%\). It selects \(B\) in every switch arm from \(25\%\) onward, while none of the same-source controls switches.

\paragraph{Earlier-Source Recovery}
To test whether \(A\) remains recoverable after the readout switches to \(B\), we use a separate \(18\)-arm replication crossing three base models with all six ordered source pairs. Each second round uses \(1{,}000\) responses. We remove \(B\), recompute the representation and calibration statistics over the remaining three candidates, and test whether \(A\) ranks first. Random selection would recover \(A\) in \(6/18\) arms. Figure~\ref{fig:switch-recovery} shows that SCOUT recovers \(A\) in \(18/18\) arms at \(1\%\) and \(2\%\) of the second round, \(15/18\) at \(5\%\), and \(6/18\) at the endpoint. Thus, after sufficient second-round SFT, SCOUT no longer recovers the earlier source reliably even after \(B\) is removed from the candidate pool.

\paragraph{Capability Tradeoff}
For each benchmark, we define source strength using the median accuracy before the second round among students initially trained on that source. We call \(B\) stronger when its median is at least that of \(A\), and weaker otherwise. Figure~\ref{fig:switch-capability} shows that when \(B\) is stronger, median accuracy rises by \(0.143\) on both GSM8K and MATH-500. When \(B\) is weaker, it falls by \(0.105\) and \(0.136\), respectively.

The capability change depends on the relative strength of the two sources. Switching to a stronger \(B\) preserves or improves median benchmark accuracy, whereas switching to a weaker \(B\) reduces it. However, without a matched \(B\)-only baseline, these experiments do not isolate whether benefits acquired from \(A\) remain after switching. Section~\ref{sec:findings} instead considers subsequent training without source replacement and finds that the original signature remains detectable.

\begin{figure*}[t]
  \centering
  \begin{subfigure}[t]{0.32\textwidth}
    \centering
    \includegraphics[width=\linewidth]{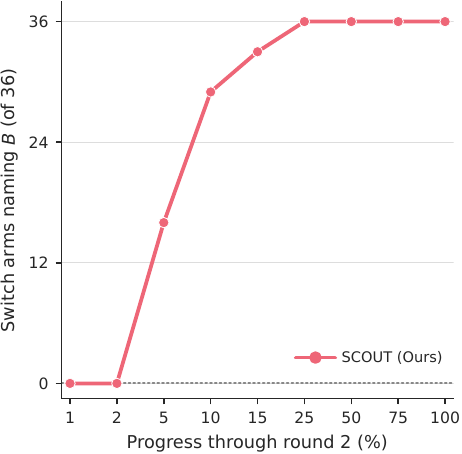}
    \caption{Attribution to the new source}
    \label{fig:switch-attribution}
  \end{subfigure}\hfill
  \begin{subfigure}[t]{0.32\textwidth}
    \centering
    \includegraphics[width=\linewidth]{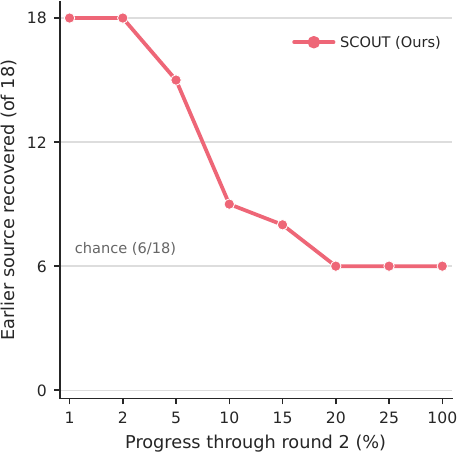}
    \caption{Recovery of the earlier source}
    \label{fig:switch-recovery}
  \end{subfigure}\hfill
  \begin{subfigure}[t]{0.32\textwidth}
    \centering
    \includegraphics[width=\linewidth]{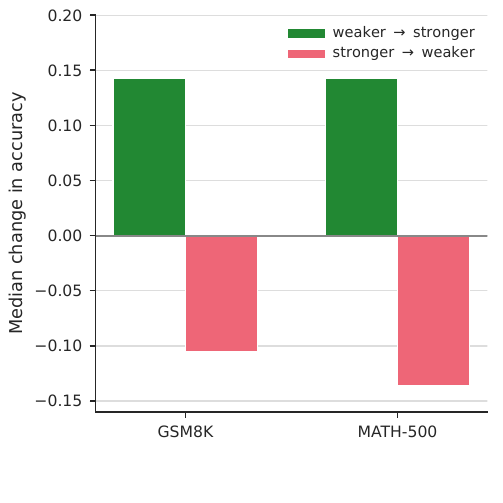}
    \caption{Capability change}
    \label{fig:switch-capability}
  \end{subfigure}
  \caption{\textbf{Source switching under a second SFT round.} \textbf{(a)} SCOUT attribution to \(B\) across \(36\) switch arms, with no switches among the \(18\) same-source controls. \textbf{(b)} Recovery of \(A\) across \(18\) arms after removing \(B\). The dashed line marks chance. \textbf{(c)} Median endpoint accuracy change when moving to a stronger or weaker source.}
  \label{fig:switch}
\end{figure*}

\paragraph{Redirection Across Readouts}

\begin{table*}[t]
  \centering
  \small
  \caption{\textbf{Attribution to the second-round source during source switching.} Entries count the \(18\) switch arms attributed to \(B\) after each percentage of the second SFT round. Higher counts indicate stronger redirection, not better attribution. DistillDetect uses the pre-distillation base as its reference.}
  \label{tab:switch-readouts}
  \begin{tabular}{lrrrrrrrr}
    \toprule
    Method
      & \(1\%\) & \(2\%\) & \(5\%\) & \(10\%\)
      & \(15\%\) & \(20\%\) & \(25\%\) & \(100\%\) \\
    \midrule
    \rowcolor{black!7}
    \multicolumn{9}{l}{\textit{Reference-based}} \\
    \rowcolor{black!7}
    DistillDetect
      & 0 & 0 & 2 & 11 & 12 & 14 & 14 & 14 \\
    \midrule
    \multicolumn{9}{l}{\textit{Output-only}} \\
    PoS Templates
      & 0 & 0 & 6 & 14 & 16 & 18 & 18 & 18 \\
    SCOUT (Ours)
      & 0 & 0 & 7 & 15 & 17 & 17 & 17 & 18 \\
    \quad \( - \) Scaling (\(C\))
      & 0 & 0 & 7 & 15 & 17 & 17 & 17 & 18 \\
    \qquad \( - \) Centering (\(D\))
      & 0 & 0 & 6 & 15 & 17 & 17 & 17 & 18 \\
    \bottomrule
  \end{tabular}
\end{table*}

Table~\ref{tab:switch-readouts} compares readouts on the \(18\)-arm replication. Every method selects \(B\) in a majority of arms by \(10\%\). All output-only readouts select \(B\) in \(18/18\) arms at the endpoint, showing that redirection is not specific to SCOUT or profile aggregation.

These results do not establish that the earlier history is erased or inaccessible to every possible auditor. They show that later SFT can redirect the signature expressed in current outputs, supporting the output-only limit formalized in Proposition~\ref{prop:history}.

\section{Additional Attribution Analyses}
\label{app:tables}
This appendix reports two additional analyses. First, it supplements the controlled evaluation with threshold-sweep and significance analyses. Second, it tests candidate-pool sensitivity using a separate panel of seven public R1-derived suspects.

\subsection{Controlled Detection Ablation}

Table~\ref{tab:detection-roc} separates score discrimination from the fold-fitted threshold transfer evaluated in Table~\ref{tab:controlled-detection}. We sweep the rejection threshold after evaluating each student with its source present and removed. At every threshold, detection requires correct attribution on both probes, while a false positive requires acceptance after source removal under the same two-probe decision rule.

\begin{table*}[t]
  \centering
  \caption{\textbf{Controlled detection ablation.} AUC is computed from student-level detection and false-positive rates. Confidence intervals are obtained by bootstrapping students, with all evaluation cells belonging to each sampled student retained together. At zero false positives on the corresponding aggregation axis, \emph{Cells} reports the maximum number of correctly accepted source-present probe cells, whereas \emph{Students} reports the maximum number of students correctly accepted on both probes. The two columns may use different thresholds. Higher is better.}
  \label{tab:detection-roc}
  \small
  \setlength{\tabcolsep}{5pt}
  \begin{tabular}{lccc}
    \toprule
    Method & AUC [95\% CI] & Cells & Students \\
    \midrule

    \multicolumn{4}{l}{\textbf{Math, 1.5--4B students}} \\
    \rowcolor{black!7}
    DistillDetect
      & 0.8587 [0.7895, 0.9501] & 26/38 & 13/19 \\
    PoS Templates
      & 0.7230 [0.5346, 0.8726] & 15/38 & 11/19 \\
    \midrule
    SCOUT (Ours)
      & 1.0000 [1.0000, 1.0000] & 37/38 & 19/19 \\
    \quad \( - \) Scaling (\(C\))
      & 1.0000 [1.0000, 1.0000] & 33/38 & 19/19 \\
    \qquad \( - \) Centering (\(D\))
      & 1.0000 [1.0000, 1.0000] & 35/38 & 19/19 \\

    \addlinespace
    \midrule
    \multicolumn{4}{l}{\textbf{Conversation, 1.5--4B students}} \\
    \rowcolor{black!7}
    DistillDetect
      & 0.8255 [0.7729, 0.9197] & 24/38 & 12/19 \\
    PoS Templates
      & 0.8172 [0.6149, 0.9723] & 27/38 & 12/19 \\
    \midrule
    SCOUT (Ours)
      & 0.9695 [0.9169, 1.0000] & 34/38 & 17/19 \\
    \quad \( - \) Scaling (\(C\))
      & 0.9501 [0.8615, 0.9945] & 25/38 & 14/19 \\
    \qquad \( - \) Centering (\(D\))
      & 0.9030 [0.7701, 0.9917] & 25/38 & 14/19 \\

    \addlinespace
    \midrule
    \multicolumn{4}{l}{\textbf{Math, 7--8B students}} \\
    \rowcolor{black!7}
    DistillDetect
      & 0.9478 [0.8752, 1.0000] & 23/42 & 11/21 \\
    PoS Templates
      & 0.6485 [0.4671, 0.8277] & 13/42 & 11/21 \\
    \midrule
    SCOUT (Ours)
      & 0.9932 [0.9728, 1.0000] & 38/42 & 20/21 \\
    \quad \( - \) Scaling (\(C\))
      & 0.9478 [0.8821, 1.0000] & 36/42 & 19/21 \\
    \qquad \( - \) Centering (\(D\))
      & 0.9410 [0.8413, 1.0000] & 38/42 & 19/21 \\
    \bottomrule
  \end{tabular}
\end{table*}

In the 1.5--4B math cohort, all three profile variants attain an AUC of \(1.000\) and detect all nineteen students at zero false positives, while SCOUT gives the highest cell-level count. In the conversation cohort, centering raises AUC from \(0.903\) to \(0.950\), and scaling raises it to \(0.970\) while increasing student detection from \(14/19\) to \(17/19\). In the 7--8B math cohort, SCOUT attains an AUC of \(0.993\) and detects \(20/21\) students, exceeding the uncalibrated and centered profiles. These threshold sweeps diagnose score separation, whereas Table~\ref{tab:controlled-detection} remains the operational evaluation of threshold transfer.

\begin{table*}[t]
  \centering
  \caption{\textbf{Adjusted significance values underlying controlled detection.} The source-present column reports the range of \(p^\star=\min\{1,(K-1)p\}\), where smaller values provide stronger evidence for correct acceptance. The source-removed columns report the minimum under backbone and teacher holdout. A source-removed value below \(0.05\) identifies a significant gap in one probe cell but does not itself constitute a false acceptance, which also requires the fold-fitted threshold and \(p^\star<0.05\) on both probes with the same selected candidate. LLM judges are omitted because they do not produce score gaps or fitted thresholds.}
  \label{tab:detection-pstar}
  \small
  \resizebox{\textwidth}{!}{%
  \begin{tabular}{lccc}
    \toprule
    Method
      & \shortstack{Source present\\\(p^\star\) range}
      & \shortstack{Source removed\\backbone-holdout minimum \(p^\star\)}
      & \shortstack{Source removed\\teacher-holdout minimum \(p^\star\)} \\
    \midrule

    \multicolumn{4}{l}{\textbf{Math, 1.5--4B students}} \\
    \rowcolor{black!7}
    DistillDetect
      & \(8.35{\times}10^{-164}\)--\(1.00\)
      & \(4.18{\times}10^{-60}\)
      & \(1.32{\times}10^{-78}\) \\
    PoS Templates
      & \(2.19{\times}10^{-34}\)--\(1.00\)
      & \(8.22{\times}10^{-29}\)
      & \(1.43{\times}10^{-34}\) \\
    SCOUT (Ours)
      & \(3.85{\times}10^{-92}\)--\(1.00\)
      & \(0.402\)
      & \(1.02{\times}10^{-6}\) \\

    \midrule
    \multicolumn{4}{l}{\textbf{Conversation, 1.5--4B students}} \\
    \rowcolor{black!7}
    DistillDetect
      & \(6.44{\times}10^{-130}\)--\(1.00\)
      & \(1.11{\times}10^{-82}\)
      & \(7.19{\times}10^{-77}\) \\
    PoS Templates
      & \(5.27{\times}10^{-31}\)--\(1.00\)
      & \(0.708\)
      & \(3.80{\times}10^{-16}\) \\
    SCOUT (Ours)
      & \(3.53{\times}10^{-124}\)--\(1.00\)
      & \(0.487\)
      & \(1.00\) \\

    \midrule
    \multicolumn{4}{l}{\textbf{Math, 7--8B students}} \\
    \rowcolor{black!7}
    DistillDetect
      & \(7.50{\times}10^{-155}\)--\(1.00\)
      & \(1.11{\times}10^{-62}\)
      & \(4.93{\times}10^{-53}\) \\
    PoS Templates
      & \(2.15{\times}10^{-34}\)--\(1.00\)
      & \(4.68{\times}10^{-33}\)
      & \(9.56{\times}10^{-29}\) \\
    SCOUT (Ours)
      & \(5.62{\times}10^{-105}\)--\(1.00\)
      & \(7.58{\times}10^{-16}\)
      & \(7.58{\times}10^{-16}\) \\
    \bottomrule
  \end{tabular}%
  }
\end{table*}

Table~\ref{tab:detection-pstar} shows that individual source-removed probe cells can have \(p^\star<0.05\) even when the complete decision rule produces no false acceptance. For SCOUT, this occurs under teacher holdout in the 1.5--4B math cohort and under both holdouts in the 7--8B cohort. However, these significant cells do not align across both probes with the same selected candidate, so no source-removed student is accepted. In the conversation cohort, the minimum source-removed \(p^\star\) for SCOUT remains above \(0.05\) under both holdouts. An upper endpoint of \(1.00\) indicates that at least one source-present probe cell has no significant top-two gap, preventing detection for the corresponding student.

\subsection{Candidate-Pool Stress Test with Source Relatives}
\label{app:source-relatives}

To test sensitivity to close alternatives, we construct a source-relative set containing DeepSeek-R1-Distill-Qwen-1.5B, 7B, 14B, and 32B and DeepSeek-R1-Distill-Llama-8B and 70B~\citep{deepseekai2025r1}, together with s1.1-32B~\citep{muennighoff2025s1simpletesttimescaling}. When auditing one of these seven suspects, the expanded candidate pool combines the original ten candidates with the other six source-relative models, excluding the audited suspect. Each trial is therefore a sixteen-way attribution decision that includes DeepSeek-R1 and six alternative R1-derived models. This paired comparison holds the suspect and audit prompts fixed while changing the candidate pool. Table~\ref{tab:source-relatives} reports how often each method selects DeepSeek-R1. We use OMI because its audit prompts do not overlap the suspects' training data.

\begin{table}[t]
  \centering
  \caption{\textbf{Candidate-pool stress test with source relatives.} The original pool contains ten candidates. For each suspect, the expanded pool adds the other six models from the source-relative set described above, giving sixteen candidates. Each entry counts DeepSeek-R1 selections among seven suspects. DistillDetect uses the corresponding pre-distillation checkpoint. Higher is better.}
  \label{tab:source-relatives}
  \small
  \begin{tabular}{lcc}
    \toprule
    Method & Original Pool & Expanded Pool \\
    \midrule
    \rowcolor{black!7}
    \multicolumn{3}{l}{\textit{Reference-based}} \\
    \rowcolor{black!7}
    DistillDetect, ref. \(=S^{(0)}\)
      & 7/7 & 6/7 \\
    \midrule
    \multicolumn{3}{l}{\textit{Output-only baselines}} \\
    PoS Templates
      & 5/7 & 1/7 \\
    LLM judge (Mistral)
      & 6/7 & 1/7 \\
    LLM judge (Solar)
      & 7/7 & 1/7 \\
    \midrule
    \multicolumn{3}{l}{\textit{SCOUT and ablations}} \\
    SCOUT (Ours)
      & 7/7 & 1/7 \\
    \quad \( - \) Scaling (\(C\))
      & 7/7 & 1/7 \\
    \qquad \( - \) Centering (\(D\))
      & 7/7 & 1/7 \\
    \bottomrule
  \end{tabular}
\end{table}

Every output-only method selects DeepSeek-R1 for only \(1/7\) suspects with the expanded pool, including several methods that attain \(7/7\) with the original pool. DistillDetect instead retains \(6/7\) selections when given the corresponding pre-distillation reference. Thus, selecting DeepSeek-R1 from current outputs is highly sensitive to close source relatives, whereas access to a historical reference largely preserves source resolution. For SCOUT, this sensitivity is consistent with the candidate-only calibration limit in Proposition~\ref{prop:candidate-only}.

\section{Additional Trajectory Analyses}
\label{app:template}
\begin{table*}[t]
\centering
\caption{\textbf{Training stages in the checkpoint trajectories.} \(S^{(0)}\) is each lineage's pre-distillation checkpoint, and later columns name the training applied after the preceding checkpoint. Source headings indicate the designated teacher used for the reported margin rather than every contributor to the training mixture. Gray rows distinguish source and lineage anchors from their subsequent checkpoints. Dashes indicate that the trajectory ends.}
\label{tab:trajectory-values}
\small
\setlength{\tabcolsep}{4pt}
\resizebox{\textwidth}{!}{%
\begin{tabular}{lccccc}
\toprule
Model or ladder
  & \(S^{(0)}\)
  & \(S^{(1)}\)
  & \(S^{(2)}\)
  & \(S^{(3)}\)
  & \(S^{(4)}\) \\
\midrule

\multicolumn{6}{l}{\textbf{(a)} Thirteen public ladders} \\

\rowcolor{black!7}
\multicolumn{6}{l}{DeepSeek-R1} \\
OLMo 3 7B
  & Base & Dolci-Think SFT & DPO & RLVR & --- \\
OLMo 3 32B
  & Base & Dolci-Think SFT & DPO & RLVR & --- \\
Skywork-OR1 32B
  & Base & R1 distillation & RL & --- & --- \\
AceReason 14B
  & Base & R1 distillation & RL & --- & --- \\
\addlinespace

\rowcolor{black!7}
\multicolumn{6}{l}{GPT-4o} \\
OLMo 2 7B
  & Base & T\"{u}lu 3 SFT & DPO & RLVR & --- \\
OLMoE 1B-7B
  & Base & T\"{u}lu 3.1 SFT & DPO & RLVR & --- \\
OLMo 2 13B
  & Base & T\"{u}lu 3 SFT & DPO & RLVR & --- \\
OLMo 2 32B
  & Base & T\"{u}lu 3 SFT & DPO & RLVR & --- \\
\addlinespace

\rowcolor{black!7}
\multicolumn{6}{l}{QwQ-32B-Preview} \\
Sky-T1 7B
  & Base & QwQ SFT & PRIME RL & QwQ+self SFT & RLOO RL \\
Sky-T1 32B
  & Base & QwQ SFT & Preference optimization & --- & --- \\
\addlinespace

\rowcolor{black!7}
\multicolumn{6}{l}{GLM-4.6} \\
DeepDive 4B
  & Base & GLM-4.6 SFT & C-GRPO & --- & --- \\
DeepDive 30B-A3B
  & Base & GLM-4.6 SFT & C-GRPO & --- & --- \\
\addlinespace

\rowcolor{black!7}
\multicolumn{6}{l}{Mixtral-8x7B-Instruct} \\
Merlinite 7B
  & Base & Mixtral SFT & RLAIF & --- & --- \\

\midrule
\multicolumn{6}{l}{\textbf{(b)} Qwen2.5-Math lineages} \\

\rowcolor{black!7}
\multicolumn{6}{l}{1.5B lineage} \\
\rowcolor{black!7}
R1-Distill-Qwen-1.5B
  & Base & R1 distillation & --- & --- & --- \\
DeepScaleR-1.5B
  & Base & R1 distillation & GRPO & --- & --- \\
ZR1-1.5B
  & Base & R1 distillation & PRIME+RLOO & --- & --- \\
Nemotron-Research-1.5B
  & Base & R1 distillation & ProRL & --- & --- \\
DRPO-1.5B
  & Base & R1 distillation & DRPO & --- & --- \\
DisCO-1.5B-logL
  & Base & R1 distillation & DisCO & --- & --- \\
\addlinespace

\rowcolor{black!7}
\multicolumn{6}{l}{7B lineage} \\
\rowcolor{black!7}
R1-Distill-Qwen-7B
  & Base & R1 distillation & --- & --- & --- \\
OREAL-7B
  & Base & R1 distillation & Outcome RL & --- & --- \\
Spiral-7B
  & Base & R1 distillation & Self-play RL & --- & --- \\
Skywork-OR1-7B
  & Base & R1 distillation & GRPO & --- & --- \\
DRPO-7B
  & Base & R1 distillation & DRPO & --- & --- \\
DisCO-7B-logL
  & Base & R1 distillation & DisCO & --- & --- \\
Light-R1-7B-DS
  & Base & R1 distillation & SFT on R1 outputs & --- & --- \\
R1-Distill-7B-ThinkPO
  & Base & R1 distillation & DPO with R1 chosen & --- & --- \\

\midrule
\multicolumn{6}{l}{\textbf{(c)} Dream-7B lineages} \\

\rowcolor{black!7}
Dream-v0-Instruct-7B
  & Base & Instruction tuning & --- & --- & --- \\
SARDI
  & Base & Instruction tuning & SFT on GPT-4o-mini outputs & --- & --- \\
R2-dLLM
  & Base & Instruction tuning & Redundancy-aware SFT & --- & --- \\
dParallel
  & Base & Instruction tuning & Self-distillation & --- & --- \\
d3LLM
  & Base & Instruction tuning & Pseudo-trajectory distillation & --- & --- \\
\bottomrule
\end{tabular}%
}
\end{table*}

This appendix extends the trajectory analysis in Section~\ref{sec:findings} by documenting the released checkpoint sequences analyzed there and reporting the corresponding s1 results. All analyses use the same thirteen-candidate pool, with candidate responses, profiles, and empirical nulls fixed across checkpoints and panels for each probe. Table~\ref{tab:trajectory-values} maps every checkpoint used in this analysis to its training stage. Panel \textbf{(a)} groups ladders by their designated source, panel \textbf{(b)} measures margins toward R1, and panel \textbf{(c)} reports separate GPT-4o and Sonnet 3.5 readings.

\begin{figure*}[t]
\centering
\begin{subfigure}[b]{0.44\textwidth}
\centering
\includegraphics[width=\linewidth]{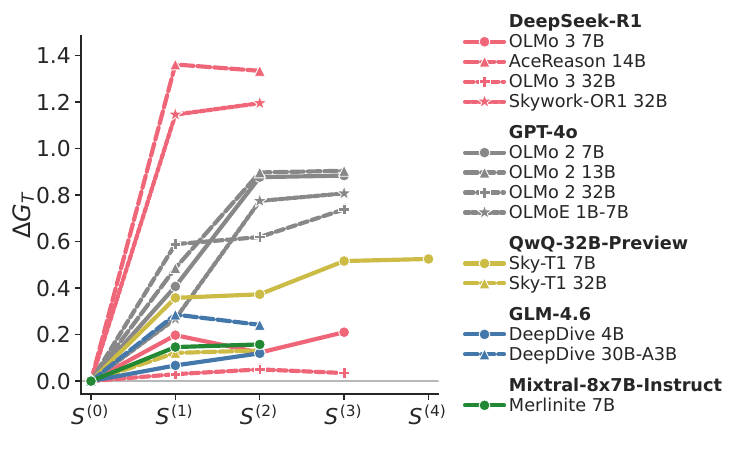}
\caption{Thirteen public ladders}
\label{fig:trajectories-s1-a}
\end{subfigure}\hfill
\begin{subfigure}[b]{0.27\textwidth}
\centering
\includegraphics[width=\linewidth]{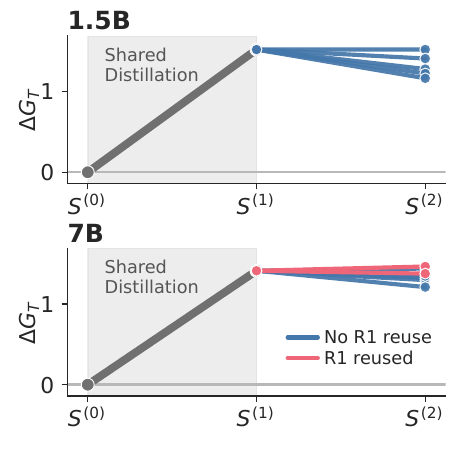}
\caption{Qwen2.5-Math lineages}
\label{fig:trajectories-s1-b}
\end{subfigure}\hfill
\begin{subfigure}[b]{0.27\textwidth}
\centering
\includegraphics[width=\linewidth]{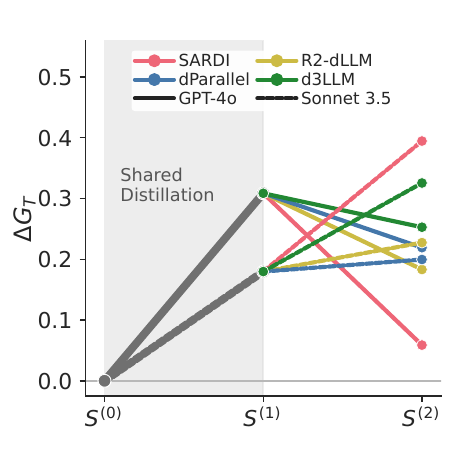}
\caption{Dream-7B lineages}
\label{fig:trajectories-s1-c}
\end{subfigure}
\caption{\textbf{Source-margin shifts from the pre-distillation state on s1.} Panels \textbf{(a)}, \textbf{(b)}, and \textbf{(c)} correspond to the OMI results in Figure~\ref{fig:trajectories} and show thirteen public ladders, two R1-distilled lineages and their descendants, and four post-training branches of Dream-7B, respectively. Positive values indicate a larger margin toward the indicated source than at \(S^{(0)}\). Every post-distillation shift is positive, reproducing the OMI pattern. Shading marks distillation in \textbf{(b)} and \textbf{(c)}. Panel scales differ.}
\label{fig:trajectories-s1}
\end{figure*}

\subsection{Thirteen Public Ladders}

The thirteen ladders comprise checkpoint sequences from OLMo 2~\citep{walsh2025}, OLMo 3~\citep{olmo2026olmo3}, OLMoE~\citep{muennighoff2024olmoeopenmixtureofexpertslanguage}, Skywork-OR1~\citep{he2025skyworkopenreasoner1}, AceReason~\citep{chen2025acereasonnemotronadvancingmathcode}, Sky-T1~\citep{sky-t1-7b,sky_t1_2025}, DeepDive~\citep{zhang2026chainingevidencerobustreinforcement}, and Merlinite~\citep{sudalairaj2024lablargescalealignmentchatbots}. Their five designated sources are DeepSeek-R1~\citep{deepseekai2025r1}, GPT-4o~\citep{openai2024gpt4ocard}, QwQ-32B-Preview~\citep{qwq-32b-preview}, GLM-4.6~\citep{team2025glm45agenticreasoningcoding}, and Mixtral-8x7B-Instruct~\citep{jiang2024mixtralexperts}. Figures~\ref{fig:trajectories-a} and~\ref{fig:trajectories-s1-a} show the thirteen ladders on OMI and s1. Across thirteen ladders evaluated on two probes, \(\Delta G_T>0\) in all \(68\) post-distillation checkpoint--probe observations. These comprise \(26\) observations immediately after distillation and \(42\) observations from \(21\) later checkpoints.

\subsection{R1 Descendant Lineages}

We anchor the two R1 lineages of Section~\ref{sec:exp-descendants} at their Qwen2.5-Math bases~\citep{yang2024qwen25mathtechnicalreportmathematical}. Their distilled parents are DeepSeek-R1-Distill-Qwen-1.5B and 7B~\citep{deepseekai2025r1}. The twelve descendants are DeepScaleR-1.5B~\citep{deepscaler2025}, ZR1-1.5B~\citep{zyphra2025ZR1}, Nemotron-Research-1.5B~\citep{liu2025prorl}, DRPO-1.5B and 7B~\citep{li2025drpo}, DisCO-1.5B-logL and 7B-logL~\citep{li2025disco}, OREAL-7B~\citep{lyu2025exploring}, Spiral-7B~\citep{liu2025spiral}, Skywork-OR1-7B~\citep{he2025skyworkopenreasoner1}, Light-R1-7B-DS~\citep{lightr1proj}, and R1-Distill-7B-ThinkPO~\citep{yang2025thinkingpreferenceoptimization}. Figures~\ref{fig:trajectories-b} and~\ref{fig:trajectories-s1-b} show positive \(\Delta G_T\) for both distilled parents and all twelve descendants on both probes. Consequently, all \(28\) model--probe observations are positive. Seven of twelve descendants on OMI and three of twelve on s1 exceed their respective distilled parents, showing that the signature persists but does not consistently strengthen during subsequent training.

\subsection{Masked Diffusion Lineage}

The Dream-7B~\citep{ye2025dream7bdiffusionlarge} analysis tests whether the contiguous PoS \(n\)-gram signal depends on left-to-right autoregressive generation. Dream-7B instead generates through iterative masked denoising. We evaluate its instruction-tuned checkpoint and four public post-training branches, SARDI~\citep{junger2026selfaugmenting}, R2-dLLM~\citep{du2026r2dllmacceleratingdiffusionlarge}, dParallel~\citep{chen2026dparallel}, and d3LLM~\citep{qian2026dllm}, relative to the pre-distillation base.

Because its instruction data include responses from GPT-4o~\citep{openai2024gpt4ocard} and Sonnet 3.5~\citep{claude35}, Figures~\ref{fig:trajectories-c} and~\ref{fig:trajectories-s1-c} report separate margins toward both models. Both margins increase at \(S^{(1)}\) after instruction tuning and remain positive at \(S^{(2)}\) for all four branches on both probes. SARDI underwent additional SFT on GPT-4o-mini outputs between \(S^{(1)}\) and \(S^{(2)}\), which may affect its GPT-4o margin at \(S^{(2)}\). Even when SARDI is excluded to account for this additional supervision, the GPT-4o margins at \(S^{(2)}\) remain positive for all three other branches on both probes.

\section{Experimental Details}
\label{app:details}
This appendix specifies the datasets, cohort construction, response generation, distillation, candidate pools, and scoring procedures underlying the retrospective, controlled, and trajectory analyses.

\paragraph{Controlled cohorts} The controlled evaluation in Section~\ref{sec:exp-controlled} compares a math cohort following DistillDetect~\citep{rawat2026referencebaseddistillationdetectionllms} with a conversation cohort constructed for this study to test whether attribution extends beyond math. Both begin as \(24\)-student grids with the same backbone-source structure, using Gemma-3-4B-PT~\citep{gemmateam2025gemma3technicalreport}, Llama-3.2-3B-Instruct~\citep{grattafiori2024llama3herdmodels}, Qwen-2.5-1.5B, and Qwen-2.5-3B~\citep{qwen2.5} as backbones. Within each grid, the source varies among GPT-OSS-120B~\citep{openai2025gptoss120bgptoss20bmodel}, Llama-3.3-70B-Instruct, and Qwen-3-8B~\citep{qwen3}. The math grid trains each student on \(1{,}000\) responses from either OMI~\citep{toshniwal2024openmathinstruct} or s1~\citep{muennighoff2025s1simpletesttimescaling} and retains the nineteen students for which distillation improves performance over the corresponding base model as assessed on GSM8K~\citep{cobbe2021gsm8k} and MATH-500~\citep{hendrycksmath2021}. The conversation grid instead trains each student on \(937\) responses from either OASST1~\citep{3666122.3668186} or Dolly~\citep{DatabricksBlog2023DollyV2}. Filtering Dolly and excluding its \(200\) audit prompts leaves \(937\) eligible prompts. The OASST1 training bank is also disjoint from its \(200\) audit prompts and is subsampled to the same size. This grid retains the nineteen students whose strict prompt-level IFEval accuracy~\citep{zhou2023instruction} on full responses exceeds that of the corresponding base model. An additional \(24\)-student math grid uses Falcon3-7B-Base~\citep{Falcon3}, Llama-3.1-8B~\citep{grattafiori2024llama3herdmodels}, Mistral-7B-v0.3~\citep{jiang2023mistral7b}, and Qwen-2.5-7B~\citep{qwen2.5}. Using the same sources and math datasets, it retains twenty-one students whose mean accuracy across GSM8K and MATH-500 exceeds that of the corresponding base model. The controlled candidate pool contains the three sources and Gemma-3-27B-it, an unexposed distractor that shares a model family with Gemma-3-4B-PT in the 1.5--4B grids. Attribution and significance use the probe dataset not used for SFT, whereas detection and abstention use both probe datasets.

\paragraph{Controlled SFT} We train each controlled student using full-parameter SFT with a completion-only causal language-modeling objective. Prompt tokens are masked to \(-100\), and each prompt--response pair forms one example. Examples are truncated rather than packed. Table~\ref{tab:train-config} reports the shared configuration.

\begin{table}[t]
  \centering
  \caption{\textbf{Controlled SFT configuration.}}
  \label{tab:train-config}
  \small
  \begin{tabular}{ll}
    \toprule
    Setting & Value \\
    \midrule
    Objective & Completion-only causal language modeling \\
    Fine-tuning method & Full-parameter fine-tuning \\
    Epochs & 3 \\
    Learning rate & \(1\times10^{-5}\) \\
    Schedule & Cosine with warmup ratio \(0.05\) \\
    Effective batch size & 16 (batch size 4, accumulation 4) \\
    Maximum sequence length & \(4{,}096\) tokens \\
    Training precision & bfloat16 \\
    \bottomrule
  \end{tabular}
\end{table}

\paragraph{Prompt formatting} Llama-3.2-3B-Instruct is the only controlled backbone that uses its tokenizer's chat template. All other controlled backbones use the plain-text format
\[
\texttt{Problem:\textbackslash n\{question\}\textbackslash n\textbackslash nSolution:\textbackslash n}.
\]
The same backbone-specific format is used for SFT and for generating controlled student responses.

\paragraph{Response generation} For locally generated corpora, we generate one response per prompt with a model-prompt-specific seed, using each model's released generation configuration when available and the serving defaults otherwise. Consequently, temperature, top-\(p\), top-\(k\), repetition penalty, and end-of-sequence tokens may differ across models. We set the maximum output length to \(4{,}096\) tokens, except that the existing Qwen-3-8B~\citep{qwen3} training and candidate response banks use a \(2{,}048\)-token limit. Every response that reaches either limit exceeds the first \(2{,}000\) characters analyzed by SCOUT, so the limit does not alter the analyzed prefix. For the utility screens used to select the controlled cohorts, we instead generate responses from the controlled students with a \(16{,}384\)-token limit and score their full responses on GSM8K, MATH-500, and IFEval.

\paragraph{Candidate pools and response provenance} The retrospective public-model evaluation in Section~\ref{sec:exp-descendants} uses domain-specific ten-candidate pools. For the math probes, the pool contains Gemma-3-27B-it~\citep{gemmateam2025gemma3technicalreport}, Llama-3.3-70B-Instruct, GPT-OSS-120B~\citep{openai2025gptoss120bgptoss20bmodel}, Claude-3.5-Sonnet~\citep{claude35}, Claude Opus 4.5 and Claude Opus 4.6~\citep{opus45,opus46}, DeepSeek-R1~\citep{deepseekai2025r1}, o1 and o3~\citep{openai2024openaio1card,openai2025openaio3card}, and QwQ-32B-Preview~\citep{qwq-32b-preview}. Candidate responses for these probes come from the response bank released with DistillDetect~\citep{rawat2026referencebaseddistillationdetectionllms}. For the conversation probes, we collect a new candidate response bank. Claude-3.5-Sonnet was no longer served through OpenRouter, so we replace it with its successor, Claude Sonnet 4~\citep{claude4}. This retains a Sonnet-family candidate and the ten-candidate pool size while providing matched responses to the OASST1 and Dolly prompts. We generate all suspect and descendant responses in our pipeline. The math and conversation pools therefore have the same size but differ in one candidate and in response provenance.

Section~\ref{sec:findings} uses a thirteen-candidate pool for all trajectory panels. This pool extends the math ten-candidate pool with GPT-4o~\citep{openai2024gpt4ocard}, Mixtral-8x7B-Instruct~\citep{jiang2024mixtralexperts}, and GLM-4.6~\citep{team2025glm45agenticreasoningcoding} so that all documented sources in the panels are included without changing the pool between trajectories. The original ten candidate responses come from the DistillDetect response bank, and we generate responses for the three added candidates. For each probe, the candidate responses, profiles, and empirical nulls remain fixed across all checkpoints and panels.

\paragraph{Public models and checkpoints} The retrospective panel in Section~\ref{sec:exp-descendants} contains the Qwen2.5-Math-1.5B and Qwen2.5-Math-7B bases~\citep{yang2024qwen25mathtechnicalreportmathematical}, their DeepSeek-R1-Distill-Qwen parents~\citep{deepseekai2025r1}, and twelve descendants. The 1.5B descendants are DeepScaleR-1.5B~\citep{deepscaler2025}, ZR1-1.5B~\citep{zyphra2025ZR1}, Nemotron-Research-1.5B~\citep{liu2025prorl}, DRPO-1.5B~\citep{li2025drpo}, and DisCO-1.5B-logL~\citep{li2025disco}. The 7B descendants are OREAL-7B~\citep{lyu2025exploring}, Spiral-7B~\citep{liu2025spiral}, Skywork-OR1-7B~\citep{he2025skyworkopenreasoner1}, DRPO-7B~\citep{li2025drpo}, DisCO-7B-logL~\citep{li2025disco}, Light-R1-7B-DS~\citep{lightr1proj}, and R1-Distill-7B-ThinkPO~\citep{yang2025thinkingpreferenceoptimization}. Light-R1-7B-DS received further SFT on R1 outputs, and R1-Distill-7B-ThinkPO used R1 outputs as chosen responses in DPO. The other ten underwent reward-based or self-play post-training without additional R1 responses.

The trajectory analysis in Section~\ref{sec:findings} uses thirteen lineages composed entirely of publicly released checkpoints. These include OLMo 2~\citep{walsh2025}, OLMo 3~\citep{olmo2026olmo3}, OLMoE~\citep{muennighoff2024olmoeopenmixtureofexpertslanguage}, Skywork-OR1~\citep{he2025skyworkopenreasoner1}, AceReason~\citep{chen2025acereasonnemotronadvancingmathcode}, Sky-T1 7B and 32B~\citep{sky-t1-7b,sky_t1_2025}, DeepDive~\citep{zhang2026chainingevidencerobustreinforcement}, and Merlinite~\citep{sudalairaj2024lablargescalealignmentchatbots}. The publicly released Dream-7B lineage comprises the base and instruction-tuned successor~\citep{ye2025dream7bdiffusionlarge} and four post-training branches, SARDI~\citep{junger2026selfaugmenting}, R2-dLLM~\citep{du2026r2dllmacceleratingdiffusionlarge}, dParallel~\citep{chen2026dparallel}, and d3LLM~\citep{qian2026dllm}. We generate their audit responses but do not train or modify the checkpoints. Appendix~\ref{app:template} lists each checkpoint, designated source, and training stage. Finally, the public-suspect evaluation in Appendix~\ref{app:tables} includes DeepSeek-R1-Distill-Qwen-1.5B, DeepSeek-R1-Distill-Qwen-7B, DeepSeek-R1-Distill-Qwen-14B, DeepSeek-R1-Distill-Qwen-32B, DeepSeek-R1-Distill-Llama-8B, DeepSeek-R1-Distill-Llama-70B~\citep{deepseekai2025r1}, and s1.1-32B~\citep{muennighoff2025s1simpletesttimescaling}. All seven are evaluated on OMI. For s1, we exclude s1.1-32B because its fine-tuning set, s1K-1.1, contains every s1 probe prompt paired with a DeepSeek-R1 response.

\paragraph{Representation and scoring} We retain the first \(2{,}000\) characters of each response and process them with NLTK 3.9.1~\citep{bird-loper-2004-nltk}, using \texttt{word\_tokenize} followed by \texttt{pos\_tag}. We then count contiguous PoS \(n\)-grams for \(n\in\{3,4,5\}\). We use no lowercasing or IDF weighting. For each probe and candidate pool, the feature vocabulary is fitted once on all candidate responses. The \(\ell_2\) norm used in cosine distance includes all PoS \(n\)-grams in the evaluated response, including those absent from the candidate vocabulary. Counts and distances are computed in double precision. Appendix~\ref{app:ngram} examines the response window and related representation choices.

\paragraph{Evaluation} For all score-based experiments, we use the evaluator released with DistillDetect~\citep{rawat2026referencebaseddistillationdetectionllms} to aggregate prompt-level scores. For controlled detection and abstention in Section~\ref{sec:exp-controlled}, we also use its threshold-fitting and significance-testing procedures. Thresholds are fitted on the non-held-out cells of each evaluation fold. Student-backbone holdout excludes all students based on the audited backbone, while teacher holdout excludes every student associated with the audited teacher. For detection, the evaluator identifies the two leading candidates and uses their per-prompt score differences to test whether the mean gap exceeds the fold-fitted threshold. It then applies a one-sided one-sample \(t\)-test or Wilcoxon signed-rank test according to a Shapiro--Wilk normality check and computes \(p^\star=\min\{1,(K-1)p\}\). A detection is reported only when \(p^\star<0.05\) on both probes and both probes select the same candidate.

\paragraph{LLM Judge Protocol} We query Mistral Medium 3.1~\citep{mistral_medium_3} and Solar Pro 4~\citep{upstage2026solarpro4} through OpenRouter with temperature \(0\) and a maximum of \(64\) output tokens. Each call contains the audit prompt, one response from every candidate, and one suspect response. Candidate and suspect responses are truncated to their first \(2{,}000\) characters. Candidate identities are replaced by letter labels, and the mapping between labels and candidates is deterministically reshuffled for every suspect and prompt. We map each parsed label back to its candidate identity before aggregation and do not retry outputs that cannot be parsed. Unparseable outputs count as \texttt{NONE} in open-set trials and are excluded in closed-set trials. The most frequent result across \(200\) prompts becomes the model-level verdict for that probe. Ties yield no verdict.

\paragraph{LLM Judge Prompts} Both judges receive the shared prompt below followed by one of two instructions. The retrospective and public-suspect evaluations use the closed-set instruction because they require a candidate choice. Controlled detection uses the open-set instruction because teacher-removed trials require abstention.

\begin{tcblisting}{
  breakable,
  listing only,
  title={Shared Prompt},
  colback=black!5!white,
  colframe=black!75!black,
  colbacktitle=black!75!black,
  coltitle=white,
  listing options={
    basicstyle=\ttfamily\footnotesize,
    breaklines=true,
    breakatwhitespace=true,
    columns=fullflexible,
    keepspaces=true,
    showstringspaces=false
  }
}
You are comparing writing styles.

Below are responses written by several CANDIDATE models to the same prompt, followed by a
response written by a SUSPECT model to that prompt.

Judge only from writing style, structure, and phrasing. Ignore correctness, factual content,
the specific reasoning path, and whether two responses reach the same conclusion.

PROMPT:
{question}

CANDIDATE RESPONSES:
{candidates}

SUSPECT RESPONSE:
{suspect}
\end{tcblisting}

\begin{tcblisting}{
  breakable,
  listing only,
  title={Closed-Set Instruction},
  colback=scout!5!white,
  colframe=scout,
  colbacktitle=scout,
  coltitle=white,
  listing options={
    basicstyle=\ttfamily\footnotesize,
    breaklines=true,
    breakatwhitespace=true,
    columns=fullflexible,
    keepspaces=true,
    showstringspaces=false
  }
}
Which candidate response is stylistically most similar to the suspect response?
You must choose exactly one candidate.

Reply with JSON only:
{"choice": "<candidate label>"}
\end{tcblisting}

\begin{tcblisting}{
  breakable,
  listing only,
  title={Open-Set Instruction},
  colback=scout!5!white,
  colframe=scout,
  colbacktitle=scout,
  coltitle=white,
  listing options={
    basicstyle=\ttfamily\footnotesize,
    breaklines=true,
    breakatwhitespace=true,
    columns=fullflexible,
    keepspaces=true,
    showstringspaces=false
  }
}
Identify the candidate response that is stylistically most similar to the suspect response.
If none of the candidate responses is sufficiently similar to support a plausible stylistic match,
answer NONE.

Reply with JSON only:
{"choice": "<candidate label>"}

or

{"choice": "NONE"}
\end{tcblisting}

\paragraph{Source switching} Appendix~\ref{app:theory} reports two independent source-switching campaigns. The \(36\)-arm campaign starts from controlled students trained with a maximum sequence length of \(4{,}096\) and uses \(800\) second-round responses per arm. The \(18\)-arm campaign starts from earlier first-round checkpoints whose effective maximum sequence length was \(1{,}024\) and uses \(1{,}000\) second-round responses per arm. Both campaigns cover Gemma-3-4B-PT, Llama-3.2-3B-Instruct, and Qwen-2.5-3B with the same three controlled sources. We report them separately because their first-round checkpoints and second-round training sets differ.
\stopcontents[appendices]

\end{document}